\pdfoutput=1
\documentclass[letterpaper, 10 pt]{article}
\usepackage{style}
\usepackage{cancel}
\usepackage{blindtext}

\usepackage{amsmath,amssymb}
\usepackage{hyperref}
\usepackage{amsthm}
\usepackage{arydshln}
\usepackage{natbib}
\usepackage{xcolor}
\usepackage{threeparttable}
\usepackage{tikz}
\usepackage{tikz-cd}
\usetikzlibrary{shapes.arrows}
\usetikzlibrary{shapes}
\usetikzlibrary{decorations}
\usetikzlibrary{backgrounds}
\usetikzlibrary{calc}
\usetikzlibrary{arrows}
\usetikzlibrary{automata,positioning}
\usetikzlibrary{decorations.pathreplacing,calligraphy}

\usepackage{pgfplots}
\pgfplotsset{compat=1.18} 
\usepackage{adjustbox}
\usepackage{bm}
\usepackage{caption}
\usepackage{subcaption}
\usepackage{graphicx}
\usepackage{float}  
\usepgfplotslibrary{fillbetween}
\usepgfplotslibrary{groupplots}
\usetikzlibrary{plotmarks}
\usetikzlibrary{patterns}
\pgfplotsset{
tick label style={font=\footnotesize},
label style={font=\footnotesize},
legend style={font=\footnotesize},
}

\allowdisplaybreaks
\newcommand{\E}{\mathbb{E}}

\usepackage{multirow}
\RequirePackage{tikz}
\usepackage{pifont}
\usepackage{xcolor,colortbl}

\title{\LARGE \bf How Accurately Can a Gaussian Approximate Stochastic Approximation Iterates?}

\author{
{\normalsize Shaan Ul Haque,\ Zedong Wang,\ Zixuan Zhang,\ Siva Theja Maguluri}
}
\date{}

\begin{document}

\maketitle

\setlength{\abovedisplayskip}{5pt}
\setlength{\belowdisplayskip}{5pt}

\begin{abstract}
  Stochastic approximation (SA) is a method for finding the root of an operator perturbed by noise. The focus of this paper is studying the distribution of SA iterates in finite time. In general, it is not possible to characterize the exact distribution, and therefore our goal is to find an approximation which can yield useful tail bounds. Inspired by the rich literature on the asymptotic normality of rescaled SA iterates, we approximate the pre-limit distributions by a sequence of Gaussians whose covariance is recursively defined. In particular, we establish explicit bounds on the Wasserstein-1 distance between the rescaled iterate at time $k$ and the aforementioned Gaussian for various choices of step-sizes. Since these covariances converge to the classical asymptotic limit, our analysis also provides a convergence rate for asymptotic normality as a by-product. As an immediate consequence of our bounds, we obtain tail bounds on the error of SA iterates at any time. Finally, we establish the sharpness of our rates by providing matching lower bounds and validate our findings through simulations. 

    We obtain the sharp rates by first studying the convergence rate of the discrete Ornstein–Uhlenbeck (O-U) process driven by general noise, whose stationary distribution is identical to the limiting Gaussian distribution of the rescaled SA iterates. We believe that this is of independent interest, given its connection to sampling literature. The analysis involves adapting Stein’s method for Gaussian approximation to handle the matrix weighted sum of i.i.d. random variables. The desired finite-time bounds for SA are obtained by characterizing the error dynamics between the rescaled SA iterate and the discrete time O-U process and combining it with the convergence rate of the latter process.
\end{abstract}

\section{Introduction}
Stochastic Approximation (SA) \cite{robbins1951} is a foundational algorithm spanning applications from machine learning to solving the Bellman equation in reinforcement learning (RL) \cite{Borkar2023, sutton2018, felisa2025}. At its core, SA provides a general framework for finding the root $x^*$ of an operator $F:\mathbb{R}^d\to\mathbb{R}^d$ using only noisy observations. Formally, starting from an arbitrary $x_0\in\mathbb{R}^d$, we run the following iteration:
\begin{align}\label{eq:SA_rec}
    x_{k+1}=x_k+\alpha_k(F(x_k)+M_k).
\end{align}
where $\{M_k\}_{k\geq 0}$ is a zero-mean noise sequence which can potentially be composed of multiplicative and additive components, and $\{\alpha_k\}_{k \geq 0}$ is the step-size schedule. 

The asymptotic almost sure (a.s.) convergence of SA was established in \cite{Borkar2023, benveniste2012, KushnerYin2003, yu2012, liu2025} under diminishing step-size sequence and various assumptions on noise. These works typically utilize the ordinary differential equation (ODE) viewpoint, which interprets the SA algorithm as the noisy discretization of the ODE, $\dot{x}=F(x)$ and then study the fluctuations between the SA and the ODE. A parallel line of work analyzes the asymptotic statistics of SA. The primary focus of these studies is obtaining asymptotic normality, i.e., the distributional convergence of the rescaled iterate $(x_k-x^*)/\sqrt{\alpha_k}$ to a Gaussian limit under some mild conditions \cite{Chung1954SA, Fabian1968, Fort2015, borkar2024odemethodasymptoticstatistics, hu2024central}. A stronger result within this context, analogous to the ODE method, is constructing the corresponding stochastic differential equation (SDE) whose stationary distribution coincides with the Gaussian limit and studying the error between rescaled iterate and this SDE \cite{benveniste2012}. 
These asymptotic normality behavior forms the basis for a vast amount of literature that focuses on improving the SA algorithm such as to achieve optimal asymptotic variance \cite{polyak1992acceleration, devraj2017zap, chen2021statistical, li2022root}.

However, in practice, one is interested in the behavior of the iterates at finite time. 
This discrepancy prompts the following question:
\begin{center}
    \textit{What can one say about the distribution of the rescaled SA iterates at finite time?}
\end{center}
It turns out that the exact distribution of the rescaled SA iterates at finite time is intractable in general. So, we show that the finite time rescaled SA iterates can be approximated by a Gaussian and quantify the approximation by presenting a bound on the Wasserstein-1 distance. We characterize the covariance of this Gaussian and show that it converges to the asymptotic covariance in the limit.

\subsection{Langevin Dynamics and Normality of Stochastic Approximation}
Theoretical analysis indicates that appropriately rescaled SA transitions from global deterministic ODE-like behavior to local stochastic SDE-like behavior as iterates approach $x^*$. Consequently, the limiting Gaussian depends only on the local geometry of $F(\cdot)$ around $x^*$. Thus, studying the convergence rate of normality requires us to consider an SDE with linear drift that depends upon the Jacobian of $F(\cdot)$ at $x^*$.

To this end, we turn to the sampling literature that has extensively studied SDEs and their discrete time analogs, especially within the context of Langevin dynamics \cite{robert1999monte}. For a continuously differentiable $\Phi:\mathbb{R}^d\to \mathbb{R}$, the Langevin dynamics is given by
\begin{align}\label{eq:langevin}
    dX_t=-\nabla \Phi(X_t)dt+\sqrt{2}dW_t,
\end{align}
where $\nabla$ is the gradient operator and $W_t$ is standard Brownian motion in $\mathbb{R}^d$. Under some mild technical conditions, $p(x)\propto e^{-\Phi(x)}$ (known as Gibbs distribution) is the stationary distribution of this SDE. To simulate this process, one typically uses the Euler-Maruyama discretization of Eq. \eqref{eq:langevin}, also termed as Unadjusted Langevin Algorithm (ULA), given by
\begin{align}\label{eq:ULA}
    X_{k+1}=X_k-\alpha_k\nabla \Phi(X_k)+\sqrt{2\alpha_k}Z_k,
\end{align}
where $Z_k\sim \mathcal{N}(0, I)$. A long line of literature focuses on studying its convergence rates for diminishing step-size schedule \cite{durmus2017nonasymptotic, fang2019multivariate, pages2023unadjusted, li2025convergence}.

Returning to SA setting, the local linearization of $F(\cdot)$ around $x^*$ leads us to the well-known O-U process, an SDE with a linear drift term
\begin{align}
    dX_t = J^{(\alpha, \xi)} X_t dt + \Sigma^{1/2}_bdW_t,
\end{align}
where $J^{(\alpha, \xi)}$ is some matrix depending on the Jacobian of the operator at $x^*$ and $\Sigma_b$ is the asymptotic covariance of $M_k$. To understand the finite-time behavior, we consider the SA inspired discretized version:
\begin{align}\label{eq:z_k}
    \hat{z}_{k+1}=\hat{z}_k+\alpha_kJ^{(\alpha, \xi)}_k\hat{z}_k+\sqrt{\alpha_k}M_k,
\end{align}
where $J^{(\alpha, \xi)}_k$ represents $\mathcal{J}$ plus a step-size dependent correction term that goes to zero in the limit. We note that the key difference between Eq. \eqref{eq:z_k} and Eq. \eqref{eq:ULA} is the noise term which is always standard Gaussian for ULA while the noise perturbing the SA \eqref{eq:SA_rec} can be arbitrary. Due to this distinction, we christen this process ``Discrete O-U with Generalized noise'' or DOUG. Furthermore, DOUG is more general than ULA since it allows for non-symmetric matrices in the drift which commonly appear in SA (recall that for ULA, $\mathcal{J}$ can only be symmetric since it Hessian of a function). Nevertheless, it is simultaneously more restrictive, as it assumes the drift term to be strictly linear.

Therefore, a crucial intermediate step for understanding the finite-time distributional behavior of SA is studying DOUG which may be of independent interest given its connection to sampling. We first study the convergence rate for the DOUG iterates to the Gaussian limit whose covariance matrix is identical to the asymptotic covariance $\Sigma$ of the rescaled SA and is given by the unique solution to the Lyapunov equation
\begin{align*}
    J^{(\alpha, \xi)})\Sigma +\Sigma (J^{(\alpha, \xi)})^T+\Sigma_b&=0.
\end{align*}

\subsection{Main Contributions}
Now, we outline our main contributions in this work:

\textbf{Quantifying Normality for SA:} We first recursively define a sequence of covariance matrix which characterize the finite-time behavior of rescaled SA more precisely than the asymptotic covariance. Furthermore, we show that this sequence converges to the asymptotic covariance in the limit. We leverage this sequence of covariance to construct a Gaussian approximation for the non-asymptotic distribution and establish a $\tilde{\mathcal{O}}(\sqrt{\alpha_k})$ (where $\tilde{\mathcal{O}}(\cdot)$ contains polynomial log factors) convergence rate under Wasserstein-1 distance. We obtain this result by decomposing the analysis into two components: an error dynamics between the rescaled iterate and DOUG and the distributional convergence of DOUG itself. We provide a matching lower bounds for our results confirming their sharpness. Finally, we empirically validate the tightness of our theoretical bounds and distributional approximation across various step-size schedules via numerical simulations.

\textbf{Finite-time analysis of the DOUG via Stein’s method:} We obtain the aforementioned results by studying the distributional convergence of DOUG \eqref{eq:z_k} to the time-varying Gaussian. We first analyze DOUG with $M_k$ that only has additive i.i.d. component and establish a $\tilde{\mathcal{O}}(\sqrt{\alpha_k})$ rate of convergence in Wasserstein-1 distance by using Stein's method for normal approximation \cite{stein1972bound}. This method was developed to establish the convergence rate in the classical central limit theorem (CLT). Unlike CLT, we have matrix weighted sum of the i.i.d. random variables, and so, we adapt the Stein's proof machinery to handle these weighted sums. To eventually handle the general noise case which consists of the multiplicative martingale difference, we construct a careful coupling between the general DOUG and DOUG with purely additive noise, and show that the error terms are higher order. Finally, we establish matching lower bounds demonstrating that this convergence rate is indeed sharp.

\textbf{Tail deviations for SA:} As an immediate corollary of the Wasserstien bound, we obtain upper bounds on the  tail behavior of  SA at any finite time. More formally, for any unit vector $\mathfrak{u}\in \mathbb{R}^d$, we show that the tail probability $P(\langle x_k-x^*, \mathfrak{u}\rangle > \sqrt{\alpha_k}a)$ is given by Gaussian tail and an additional weaker tail of the form $\tilde{\mathcal{O}}(\alpha_k^{1/4})/a$.

\textbf{Finite-time bounds for constant step-size SA:} For the constant step-size, we show that the Wasserstein-1 distance is upper bounded by an exponentially decaying transient term and an $\tilde{\mathcal{O}}(\sqrt{\alpha})$ steady-state term. In the context of stochastic gradient descent (SGD), these upper bound are structurally similar (up to log factors) to the Kolmogorov-Smirnov distance bound obtained in \cite{WeiLiLouWu2025GaussianApprox}. As pointed out in \cite{WeiLiLouWu2025GaussianApprox}, these bounds complement the asymptotic results of similar nature in \cite{Zaiwei2021, wei2025online, zedong2026icml}.

\section{Related Literature}\label{sec:lit_survey}
\paragraph{Asymptotic Normality of diminishing step-size SA:} Early results on asymptotic normality for SA include \cite{Chung1954SA} for one-dimensional SA, \cite{Sacks1958} for broad classes of stochastic approximation schemes with explicit variance formulas, and \cite{Fabian1968} for Lyapunov-style characterizations of the limit covariance. These CLT-type theorems and many variants are treated systematically in \cite{Borkar2023, benveniste2012,KushnerYin2003}. Subsequent refinements include the almost sure CLT in \cite{Pelletier1999}, and the Markovian-noise extension in \cite{Fort2015}. 

\paragraph{Asymptotic Normality for constant step-size SA:}
Prior works for constant step-size SA typically adopt an asymptotic approach. Specifically, it involves fixing a constant step-size, analyzing the induced ergodic Markov chain at stationarity, and then take a limiting on the step-size to obtain a tractable approximation of the steady-state fluctuations. For SGD, these characterizations were established in \cite{dieuleveut2020bridging, wei2025online}. Similarly, in the context of contractive SA, \cite{Zaiwei2021, zhang2024prelimit} established a Gaussian weak limits. In contrast, Theorem \ref{thm_main:main_thm} provides explicit Wasserstein error bounds with clear parameter dependence, and in particular covers the constant step-size regime.

\paragraph{Stein's method for Gaussian approximation:}
Stein's method for Gaussian approximation was first introduced in \cite{stein1972bound}, where the idea is to characterize the reference measure via a Stein identity and then bound the associated solution/operator terms. This framework was later extended to random vector case in \cite{Gallouet2018, fang2019multivariate}. A popular approach within this literature involves probabilistic couplings which were explored in \cite{arratia2018sizebias, chatterjee2010newapproachstrongembeddings, Ross2011}. While powerful, coupling-based approaches typically exploit problem-specific structure. In our diminishing step-size SA setting, such tailored probabilistic constructions are generally unavailable, which motivates a Stein's approach driven instead by the algorithmic dynamics and the associated operators.

\paragraph{Stein's method for distributional convergence of SA:} Recent years have witnessed growing interest for obtaining finite-time distributional rate of convergence for averaged SA iterates. These CLT-type results specifically characterize the rate at which the rescaled average, $\sum_{i=0}^{k-1}(x_i-x^*)/\sqrt{k}$, converges to the Gaussian limit. For SGD, these rates were established in \cite{yu2020analysisconstantstepsize} while analogous results for linear SA were established in \cite{srikant2024rates, kong2025nonasymptotic}. Furthermore, \cite{anastasiou2024wasserstein} extended this analysis to the estimation of the covariance matrices in time series procedures. Notably, our results characterize the convergence of the distribution for the rescaled SA itself and do not necessitate the additional averaging step. A contemporary work \cite{kong2026finitesamplewassersteinerrorbounds} also examines finite-time convergence rates for Wasserstein-1 distance with Markovian noise and Wasserstein-$p$ ($p>1$) distance with i.i.d. noise in a similar SA setup as considered in our paper. However, in comparison to the sharper rates in our analysis, the rate $\mathcal{O}(\alpha_k^{1/6})$ obtained in their work is sub-optimal. Furthermore, while our proof technique allows us to address constant step-size SA as a special case, their framework strictly restricts to the diminishing step-size schedule. Although both our work and \cite{kong2026finitesamplewassersteinerrorbounds} employ a similar decomposition of SA into a discrete time O-U process and the residual error, the construction of this discrete process is fundamentally different.

\paragraph{Finite-time Convergence for Unadjusted Langevin Dynamics:}
Finite-time, non-asymptotic convergence results in both total variation and Wasserstein metrics for the Euler discretization of overdamped Langevin diffusion with diminishing step-size schedules were provided in \cite{durmus2017nonasymptotic, durmus2019high}. These results were subsequently extended under various generalizations in \cite{fang2019multivariate, chen2023approximationinvariantmeasurestable, pages2023unadjusted}. In contrast to our DOUG framework, these rates are established under the assumption that the noise in the algorithm is specifically Gaussian. Whereas, our analysis is distribution free and only requires finite third moment.

\section{Problem Setup}
In this section, we will set up the notations and specify the main set of assumption in the studying the distributional convergence of iteration \eqref{eq:SA_rec}. Throughout the paper, we will denote $\ell_2$ norm as $\|x\|$ and $W$-weighted norm as $\|x\|_W:=\sqrt{x^TWx}$, where $x\in \mathbb{R}^d$ and $W\in \mathbb{R}^{d\times d}$ is a positive definite symmetric matrix. Furthermore, we denote $\|U\|_W$ as the weighted operator norm for any square matrix $U$, i.e., $\|U\|_W:=\sup_{x:\|x\|_W=1}\|Ux\|_W$. In what follows, we will denote $Z$ as the standard normal random variable in $\mathbb{R}^d$. For any two $d$-dimensional random variables $X$ and $Y$, we define the Wasserstein-1 distance between as
\begin{align}\label{eq:wass_dist}
    d_{\mathcal{W}}(X, Y):=\inf_{\gamma\in \Gamma(\nu_X, \nu_Y)}\E_{(X, Y)\sim \gamma}[\|X-Y\|],
\end{align}
where $\nu_X$ and $\nu_Y$ are the distribution for $X$ and $Y$, respectively and $\Gamma(\nu_X, \nu_Y)$ denotes the set of all couplings between $X$ and $Y$.  We will work under the following set of assumptions.

\begin{assumption}\label{assump:iterate}
    There exists a Lyapunov function $\Phi(x)$ and some constants $\gamma, L_s, l, u>0$, such that the following relations hold
    \begin{align}
        &\langle \nabla \Phi(x-x^*), F(x)\rangle\leq -2\gamma\Phi(x-x^*),\tag{Negative drift}\\
        &\Phi(y)\leq \Phi(x)+\langle \nabla \Phi(x), y-x\rangle+\frac{L_s}{2}\|x-y\|_s^2, \tag{Smoothness w.r.t to $\|\cdot\|_s$}\\
        &l\Phi(x)\leq \|x\|^2\leq u\Phi(x).\tag{Norm equivalence}
    \end{align}
\end{assumption}
\begin{remark}
    The above relation ensures the convergence and the overall stability of the SA algorithm \eqref{eq:SA_rec}. There is an extensive body of work that constructs such a function to establish finite-time mean square error (MSE) bounds for SA under various general settings \cite{moulines2011non, srikant2019, chen2021finitesampleanalysisstochasticapproximation, chen2023, haque2025stochasticapproximationunboundedmarkovian}.
    Existence of a such a smooth function immediately implies MSE bounds of the form $\E[\|x_k-x^*\|^2]\leq \mathcal{O}(\alpha_k)$ which will be extremely useful for our analysis (see Appendix \ref{sec:proof_main_thm} for the derivation).
\end{remark}

\begin{assumption}\label{assump:operator}
    The operator $F(\cdot)$ is continuously differentiable at $x^*$ with the Jacobian given by $J_F$. Furthermore, the operator admits the following local linear approximation:
    \begin{align*}
        F(x)&=J_F(x-x^*)+R(x)~~ \forall x\in \mathbb{B}_{r}(x^*)
    \end{align*}
    where $\mathbb{B}_{r}(x^*)$ is an $\ell_2$ ball in $\mathbb{R}^d$ of radius $r$ centered at $x^*$ and $\|R(x)\|\leq C_r\|x-x^*\|^{1+\delta}$ with $\delta\in (0,1]$. The eigenvalues of $J_F$ have strictly negative real parts (Hurwitz matrix). Furthermore, the operator $F(x)$ is Lipschitz, i.e., there exists a constant $L_F>0$ such that
    \begin{align*}
        \|F(x)-F(y)\|\leq L_F\|x-y\|.
    \end{align*}
\end{assumption}
\begin{remark}
    To our best knowledge, prior works on the asymptotic normality of SA typically require differentiability up to second order near the optimal point $x^*$ such as in \cite{Chung1954SA, benveniste2012, Fort2015, borkar2024odemethodasymptoticstatistics}.
    By Taylor's remainder theorem, this condition immediately implies $\delta=1$. In contrast, our assumption accommodates $\delta<1$ which allows for lower-order approximation of the operator. Moreover, in Proposition \ref{prop:global_linear}, we show that local linear approximation combined with the Lipschitz property of the operator implies global linear approximation which will be crucial to our analysis. The Hurwitz condition is standard in literature and necessary for asymptotic stability of the algorithm.
\end{remark}

\begin{assumption}\label{assump:noise}
    Let $\{v_k\}_{k\geq 0}$ be a sequence of (possibly non-i.i.d.) random variables  and let $\{w_k\}_{k\geq 0}$ be an i.i.d. sequence. Define $\mathcal{F}_k=\sigma(x_0, v_0, w_0, v_1, w_1 \dots, v_{k-1}, w_{k-1})$ as the increasing $\sigma$-field generated by the sequences. Then, the noise sequence $\{M_k\}_{k\geq 0}$ can be decomposed as the sum of two components $M_k=A(v_k, x_k)+b(w_k)$ where the noise terms satisfy the following conditions:
    \begin{enumerate}
        \item $\E[A(v_k, x_k)|\mathcal{F}_k]=0$ and $\E[b(w_k)]=0$.
        \item For some constant $A_1$, $\E[\|A(v_k, x_k)\|^2|\mathcal{F}_k]\leq A_1\|x_k-x^*\|^2$ almost surely.
        \item $\E[b(w_k)b(w_k)^T]=\Sigma_b$ and $\E[\|b(w_k)\|^\mu]=B_\mu<\infty$ for all $\mu\in [0, 3]$. Also, let $B_{max}^{(n)}=\max_{\mu\in [0, n]}B_\mu$.
    \end{enumerate}
\end{assumption} 
\begin{remark}
     These conditions hold for cases where the noise grows linearly with the iterate as captured by the multiplicative term $A(v_k, x_k)$. As a concrete example, for purely multiplicative case where $A(v_k, x_k)=A(v_k)x_k$ with $\E[A(v_k)|\mathcal{F}_k]=0$ and $\E[\|A(v_k)\|^2|\mathcal{F}_k]\leq A_1$. The bounded third moment ($B_3<\infty$) for the additive i.i.d. noise $b(w_k)$ is consistent with the prior literature for Berry-Esseen theorem and quantitative CLT-style results \cite{berry1941accuracy, rollin2018quantitative, Gallouet2018}.
\end{remark}

These assumptions are motivated by practical applications of interest such as RL and SGD. Specifically, they are readily satisfied for generative and synchronous RL settings as well as for SGD with i.i.d. noise \cite{chen2022finite, chen2021finitesampleanalysisstochasticapproximation, haque2025stochasticapproximationunboundedmarkovian,zhang2024constant}. To go beyond, one only needs to relax the noise assumption to Markovian.

\begin{assumption}\label{assump:step-size}
    Finally, we assume that the step-size sequence is of the following form: $\alpha_k=\alpha/(k+K)^\xi$, where $\alpha>0$, $K\geq 2$, and $\xi\in [0,1]$. For ease of exposition, we will assume that $\alpha_k\leq 1$ for all $k\geq 0$ and any choice of $\xi$.
\end{assumption}

\section{Normality of Stochastic Approximation}
Now, we are ready to state our main results. We first provide non-asymptotic convergence rate to a time-varying Gaussian for the rescaled SA. Recall that the SA iteration is given by:
\begin{align}\label{eq:SA_rec2}
    x_{k+1}=x_k+\alpha_k(F(x_k)+M_k).
\end{align}
Denote $y_k=(x_k-x^*)/\sqrt{\alpha_k}$ and define the following sequence of matrices
\begin{equation}\label{eq:J_k}
    \begin{split}
        J^{(\alpha, \xi)}_k&=J_F+\frac{\alpha^{-1}\xi}{2(k+K)^{1-\xi}}I\\
        &\stackrel{k\uparrow \infty}{\to} J^{(\alpha, \xi)}: = J_F+\frac{\alpha^{-1}\mathbbm{1}_{\xi=1}}{2}I.
    \end{split}
\end{equation}
where $\mathbbm{1}_{\cdot}$ is the indicator function. It is well-known that the asymptotic covariance is sensitive to the choice of step-size \cite{Chung1954SA, Fabian1968, Borkar2008, benveniste2012}. More specifically, it is given by the solution to the Lyapunov equation
\begin{align}\label{eq:lyap_eq}
    J^{(\alpha, \xi)}\Sigma^{(\alpha, \xi)}+\Sigma^{(\alpha, \xi)}(J^{(\alpha, \xi)})^T+\Sigma_b=0.
\end{align}
Note that $J^{(\alpha, \xi)}_k\stackrel{k\uparrow \infty}{\to} J_F$ for $\xi<1$ meaning that $\Sigma^{(\alpha, \xi)}$ is independent of both $\alpha$ and $\xi$ in this regime. However, a phase transition occurs in the drift dynamics when $\xi=1$ where the rate of decay for the step-size becomes comparable to the negative drift term $\alpha_kJ_Fy_k$. To address this elegantly in our analysis and get a tight rate of convergence in all regimes, we have introduced the sequence $\{J_k^{(\alpha, \xi)}\}_{k\geq 0}$. For ease of notation, we will denote $\Sigma^{(\alpha, \xi)}=\Sigma$ for $\xi<1$ and $\Sigma^{(\alpha, 1)}=\Sigma^{(\alpha)}$ for $\xi=1$ in the following.

To understand the finite time Gaussian approximation of rescaled iterates $y_k$, we define a sequence of covariance matrices $\{\Sigma^{(\alpha, \xi)}_k\}_{k\geq 0}$ governed by the following recursion:
\begin{align}\label{eq:sigma_k_eq}
    \Sigma_{k+1}^{(\alpha, \xi)}&=(I + \alpha_k J^{(\alpha, \xi)}_k)\Sigma_k^{(\alpha, \xi)}(I + \alpha_k J^{(\alpha, \xi)}_k)^T+\alpha_k\Sigma_b.
\end{align}
where $\Sigma_0=\E[y_0y_0^T]$ and $\Sigma_k^{(\alpha, \xi)}\to \Sigma^{(\alpha, \xi)}$ (see Corollary \ref{cor:asymp_gaussian_rate}). The rationale behind the introduction of this specific sequence is the observation that $y_k$ is given by the following update rule
\begin{align*}
     y_{k+1}&=(I+\alpha_kJ^{(\alpha, \xi)}_k)y_k+\sqrt{\alpha_k}M_k+\mathfrak{R}_k.
\end{align*}
where $\mathfrak{R}_k$ is a higher order remainder term (see proof sketch in Appendix \ref{sec:proof_sketch} for more details). Recall that $M_k=A(v_k, x_k)+b(w_k)$ and the multiplicative component satisfies $\E[\|A(v_k, x_k)\|^2]\leq A_1\E[\|x_k-x^*\|^2]$. Since Assumption \ref{assump:iterate} implies $\E[\|x_k-x^*\|^2]\leq \mathcal{O}(\alpha_k)\to 0$ (see Appendix \ref{sec:proof_main_thm}), the dominant behavior of the covariance of $y_k$ is governed by the additive component $b(w_k)$. In particular, the covariance of $y_k$ evolves approximately as
\begin{align*}
    \E[y_{k+1}y_{k+1}^T]&\approx(I + \alpha_k J^{(\alpha, \xi)}_k)\E[y_ky_k^T](I + \alpha_k J^{(\alpha, \xi)}_k)^T+\alpha_k\Sigma_b.
\end{align*}
Therefore, the above discussion suggests that to neatly approximate the finite-time distribution of $y_k$ with a Gaussian distribution, one must use $\mathcal{N}(0, \Sigma_k^{(\alpha, \xi)})$ as the reference measure. The higher order terms can be bounded separately using tools from MSE bounds for SA analysis.

Finally, to study the convergence rate, define $V$ as the solution to the following Lyapunov equation
\begin{align}\label{eq:lyap_gen_<1_I}
    J_FV+VJ_F^T+I=0.
\end{align}
Let $\iota_V=1/(4\lambda^{max}_V)$ where $\lambda^{max}_V$ denote the maximum eigenvalue of $V$. Furthermore, define a sequence $\{\tau_k\}_{k\geq 0}$ as $$\tau_k = 
\begin{cases} 
1, & \text{if } 2\iota_V \neq 3\gamma, \\
(k+K)^{1-\xi}, & \text{if } 2\iota_V = 3\gamma \text{ and } \xi < 1\\
\log(k+K), & \text{if } 2\iota_V = 3\gamma \text{ and } \xi = 1.
\end{cases}$$ Now we present the following convergence bound, whose proof is given in Appendix \ref{sec:proof_main_thm}.

\begin{theorem}\label{thm_main:main_thm}
Let $\eta=\min(\iota_V/2, 3\gamma/4)$ and $\mathcal{E}_0=\E[\|y_0\|^2]$. Then, under Assumptions \ref{assump:iterate}-\ref{assump:step-size}, the sequence $\{y_k\}_{k\geq 0}$ given by rescaling the iterate $x_k$ in update equation \eqref{eq:SA_rec2} satisfies the following bounds for all $k\geq 1$.
    \begin{enumerate}
        \item When $\xi=0$, and $\alpha$ is small enough, we have{\normalfont :}
        \begin{align*}
            d_{\mathcal{W}}(y_k, (\Sigma_k^{(\alpha, 0)})^{1/2}Z)&\leq \mathcal{O}\left(\alpha^{\frac{\delta}{2}}\right)+\mathcal{O}\left(\sqrt{\alpha}\log\left(\frac{1}{\alpha}\right)\right)+\mathcal{O}\left(\mathcal{E}_0\tau_ke^{-\eta\alpha k}\right).
        \end{align*}
        \item When $\xi\in (0, 1)$ and $K$ is large enough, we have{\normalfont :}
        \begin{align*}
            d_{\mathcal{W}}(y_k, (\Sigma_k^{(\alpha, \xi)})^{1/2}Z)&\leq 
            \mathcal{O}\left((k+K)^{-\frac{\xi\delta}{2}}\right)+\mathcal{O}\left(\frac{\log(k+K)}{(k+K)^{\frac{\xi}{2}}}\right)+\mathcal{O}\left(\mathcal{E}_0\tau_ke^{-\frac{\eta\alpha}{1-\xi}\left((k+K)^{1-\xi}\right)}\right).
        \end{align*}
        \item When $\xi=1$, $\iota_V\alpha>2$, $3\delta\gamma\alpha>2$, and $K$ is large enough, we have {\normalfont :}
        \begin{align*}
            d_{\mathcal{W}}(y_{k}, (\Sigma_k^{(\alpha, 1)})^{1/2}Z)\leq \mathcal{O}\left((k+K)^{-\frac{\delta}{2}}\right)+\mathcal{O}\left(\frac{\log(k+K)}{\sqrt{k+K}}\right)+\mathcal{O}\left(\frac{\mathcal{E}_0\tau_k}{(k+K)^{\eta\alpha}}\right).
        \end{align*}
    \end{enumerate}
\end{theorem}

For simplicity, we stated our theorem in $\mathcal{O}(\cdot)$ notation. A detailed version with explicit upper bounds including multiple higher order terms is presented in Appendix \ref{sec:proof_main_thm}, along with the proof.

The overall convergence rate is governed by the first two terms. In particular, the rate of convergence depends on the smoothness of the operator, i.e., the parameter $\delta$. If
 $\delta<1$, then the error dynamics between $y_k$ and the DOUG process (the first term) is dominant whereas for $\delta=1$, the convergence rate of the DOUG process itself (the second term) becomes the bottleneck. The third term is the transient term that depends on the initial error and decays rapidly. The decay rate is highly sensitive to the choice of step-size: exponential for constant step-size, stretched exponential (i.e., $\exp\left(-c(k+K)^{1-\xi}\right)$) for $\xi\in (0, 1)$, and polynomial for $\xi=1$. This decomposition of error into a dominant term and a transient term depending on the initial error is similar to mean square error bounds established in 
\cite{chen2021finitesampleanalysisstochasticapproximation, chen2023, haque2025stochasticapproximationunboundedmarkovian}. 

We highlight that these bounds further validate the dual ODE and SDE like behavior of SA.
When the iterates $x_k$ are far off from $x^*$, 
the stochastic fluctuations due to SDE are negligible and the ODE effect dominates driving the SA iteration towards $x^*$ (manifested by the negative drift rate $\gamma$). Once in the vicinity of $x^*$, the local linear or the SDE behavior becomes dominant manifested by $\iota_V$. As a result of the two competing transient effects, the ultimate rate of decay in higher order terms is governed by weaker of the two negative drift rates, $\eta=\min(\iota_V/2, 3\gamma/4)$.

Note that while the transient term provides an upper bound, the rate $\eta$ which is determined by $\gamma$ and $\iota_V$ is not sharp in general. The rate $\gamma$ comes from the choice of the Lyapunov function in Assumption \ref{assump:iterate}, which in general need not be tight. One can get the tightest possible ODE transient rate for a given problem by picking the optimal Lypaunov function. Similarly, the SDE transient rate $\iota_V$ is also not tight. While, we use the Lyapunov equation Eq. \eqref{eq:lyap_gen_<1_I} for the analysis here, one may use a more general Lyapunov equation
\begin{align*}
    J_FV+VJ_F^T+S=0.
\end{align*}
for any positive definite matrix $S$, and different choices of $S$ can give different rates. However, it is not clear if this family of equations provides the tightest possible drift rates. Nevertheless, for the purpose of finite-time analysis with a general Hurwitz matrix driving the SA recursion, using the solution of the Lyapunov equation is sufficient for obtaining non-asymptotic convergence rate. 

\subsection{Rate of convergence for asymptotic Normality}
By characterizing the convergence of $\Sigma^{(\alpha, \xi)}_k$ to $\Sigma^{(\alpha, \xi)}$ and using triangle inequality, we establish the overall convergence rate to the limiting Gaussian in the following corollary (proof in Appendix \ref{sec:pf_asymp_gaussian_rate}).
\begin{corollary}\label{cor:asymp_gaussian_rate}
   Let $\eta=\min(\iota_V/2, 3\gamma/4)$ and $\mathcal{E}_0=\E[\|y_0\|^2]$. Then, under Assumptions \ref{assump:iterate}-\ref{assump:step-size}, the sequence $\{y_k\}_{k\geq 0}$ given by rescaling the iterate $x_k$ in update equation \eqref{eq:SA_rec2} satisfies the following bounds for all $k\geq 1$.
    \begin{enumerate}
        \item When $\xi=0$, and $\alpha$ is small enough, we have{\normalfont :}
        \begin{align*}
            d_{\mathcal{W}}(y_k, \Sigma^{1/2}Z)&\leq \mathcal{O}\left(\alpha^{\frac{\delta}{2}}\right)+\mathcal{O}\left(\sqrt{\alpha}\log\left(\frac{1}{\alpha}\right)\right)+\mathcal{O}\left(\mathcal{E}_0\tau_ke^{-\eta\alpha k}\right)+\mathcal{O}(\alpha).
        \end{align*}
        \item When $\xi\in (0, 1)$ and $K$ is large enough, we have{\normalfont :}
        \begin{align*}
            d_{\mathcal{W}}(y_k, \Sigma^{1/2}Z)&\leq 
            \mathcal{O}\left(\frac{1}{(k+K)^{\frac{\xi\delta}{2}}}\right)+\mathcal{O}\left(\frac{\log(k+K)}{(k+K)^{\frac{\xi}{2}}}\right)+\mathcal{O}\left(\mathcal{E}_0\tau_ke^{-\frac{\eta\alpha}{1-\xi}\left((k+K)^{1-\xi}\right)}\right)\\
            &~+\mathcal{O}\left(\max\left(\frac{1}{(k+K)^\xi}, \frac{1}{(k+K)^{1-\xi}}\right)\right).
        \end{align*}
        \item When $\xi=1$, $\iota_V\alpha>2$, $3\delta\gamma\alpha>2$, and $K$ is large enough, we have {\normalfont :}
        \begin{align*}
            d_{\mathcal{W}}(y_{k}, (\Sigma^{(\alpha)})^{1/2}Z)\leq \mathcal{O}\left(\frac{1}{(k+K)^{\frac{\delta}{2}}}\right)+\mathcal{O}\left(\frac{\log(k+K)}{\sqrt{k+K}}\right)+\mathcal{O}\left(\frac{\mathcal{E}_0\tau_k}{(k+K)^{\eta\alpha}}\right)+\mathcal{O}\left(\frac{1}{k+K}\right).
        \end{align*}
    \end{enumerate}
\end{corollary}
Compared to Theorem \ref{thm_main:main_thm}, the upper bound now includes an additional error component given by the last term. For $\xi=0$ and $\xi=1$, this additional term is $\mathcal{O}(\alpha_k)$, making it strictly higher order. As a result, for these two choices of $\xi$, the overall rate of convergence remains fundamentally the same as in Theorem \ref{thm_main:main_thm}. Interestingly, we observe a phase transition for the intermediate regime $\xi\in(0,1)$. To elucidate more on this, let us fix $\delta=1$. For $\xi\in (0, 2/3)$, the additional error term does not affect the final rate of convergence. However, for $\xi\geq 2/3$, this term begins to dominate the convergence rate to asymptotic Gaussian limit. Notably, this transition at $\xi=2/3$ can be independently inferred by appropriately scaling the MSE bounds in \cite[Eq. (4.3)]{haque2025tightfinitetimebounds}. While this phenomenon is also noted in \cite{kong2026finitesamplewassersteinerrorbounds}, their sub-optimal convergence rate results in a different threshold point. 

Corollary \ref{cor:asymp_gaussian_rate} suggests a separation of timescales: the rescaled iterates $y_k$ converge to the time-varying Gaussian faster than the limiting distribution. Essentially, the ``Gaussianization'' of $y_k$ occurs rapidly, after which this intermediate Gaussian slowly approaches the asymptote. Simulations confirm this behavior, particularly for moderate $k$ and $\xi \geq 2/3$. As Figure \ref{fig:comparison} demonstrates, the empirical histogram of $y_k$ aligns almost perfectly with the time-varying Gaussian, whereas the asymptotic Gaussian provides a poorer finite time fit. Remarkably, while our theoretical bounds do not cleanly resolve the $\xi < 2/3$ regime (where convergence to the time-varying Gaussian is the bottleneck), empirical observations confirm that for sufficiently large $K$, the time-varying Gaussian remains a superior finite-time approximation. Further simulation results for $\xi<2/3$ and more details on the setup are provided in Appendix \ref{sec:simulations}.

\begin{figure*}[t!]
     \centering
     \begin{subfigure}[b]{0.25\textwidth}
         \centering
         \includegraphics[width=\linewidth]{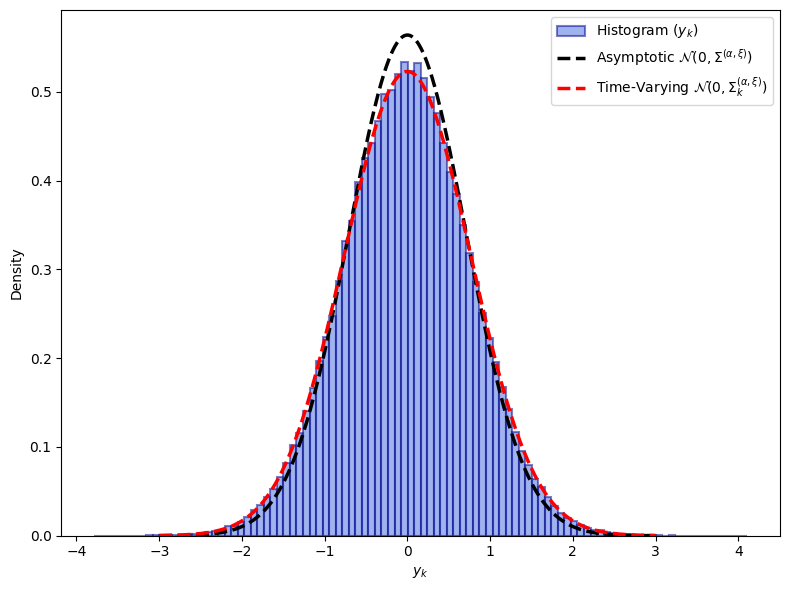}
         \caption{$\xi=0.8$, $k=100$}
         \label{fig:0.8-100}
     \end{subfigure}
     \hspace{5mm} 
     \begin{subfigure}[b]{0.25\textwidth}
         \centering
         \includegraphics[width=\linewidth]{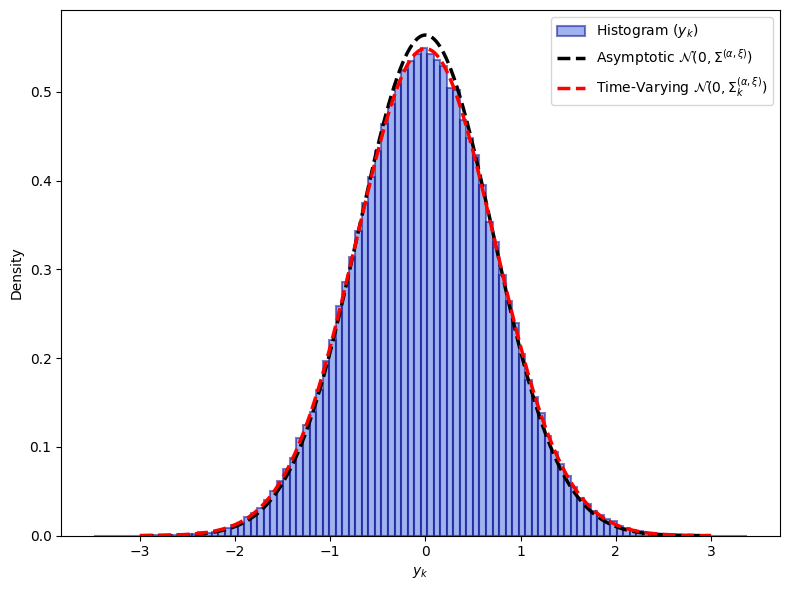}
         \caption{$\xi=0.8$, $k=1000$}
         \label{fig:0.8-1000}
     \end{subfigure}
     \hspace{5mm}
     \begin{subfigure}[b]{0.25\textwidth}
         \centering
         \includegraphics[width=\linewidth]{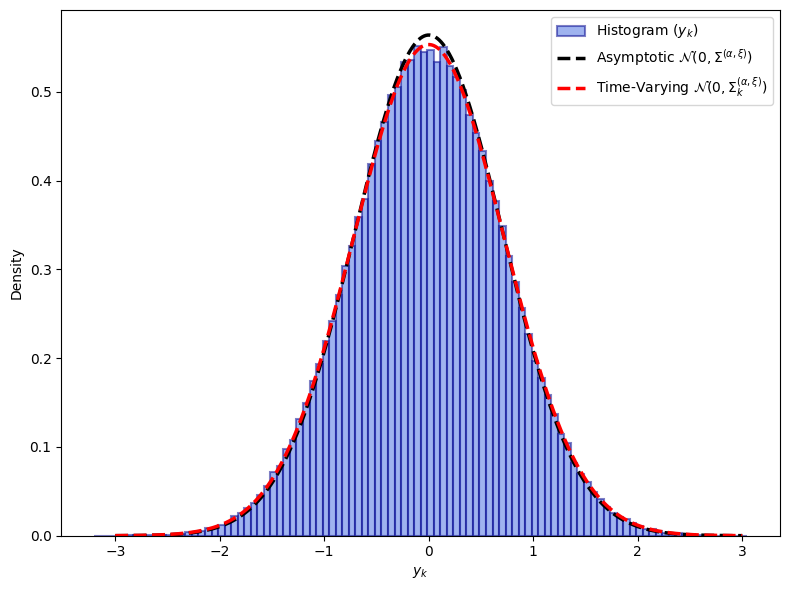}
         \caption{$\xi=0.8$, $k=5000$}
         \label{fig:0.8-5000}
     \end{subfigure}
     \\
     \vspace{2mm}
     \begin{subfigure}[b]{0.25\textwidth}
         \centering
         \includegraphics[width=\linewidth]{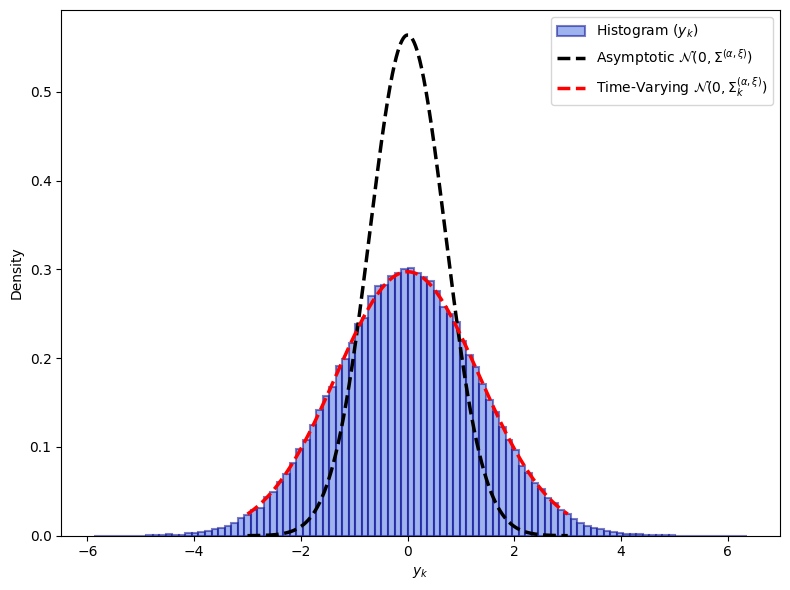}
         \caption{$\xi=0.9$, $k=100$}
         \label{fig:0.9-100}
     \end{subfigure}
     \hspace{5mm}
     \begin{subfigure}[b]{0.25\textwidth}
         \centering
         \includegraphics[width=\linewidth]{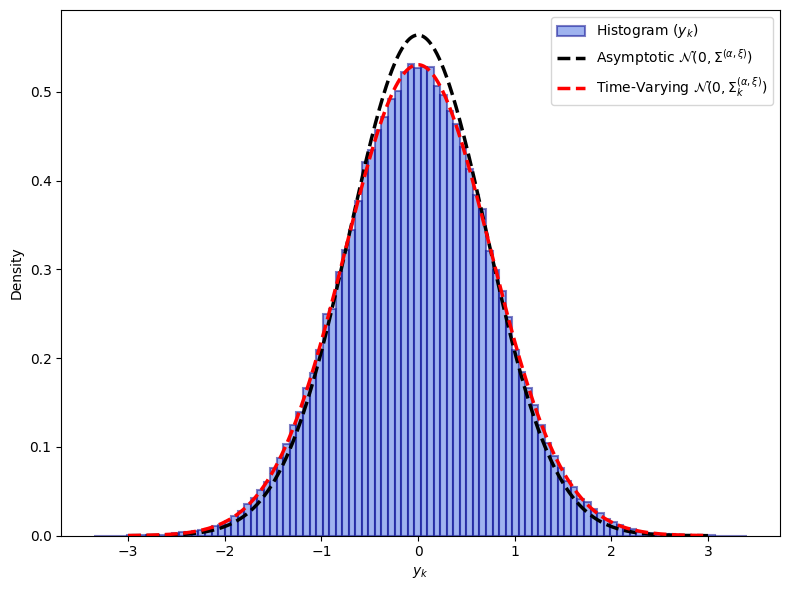}
         \caption{$\xi=0.9$, $k=1000$}
         \label{fig:0.9-1000}
     \end{subfigure}
     \hspace{5mm}
     \begin{subfigure}[b]{0.25\textwidth}
         \centering
         \includegraphics[width=\linewidth]{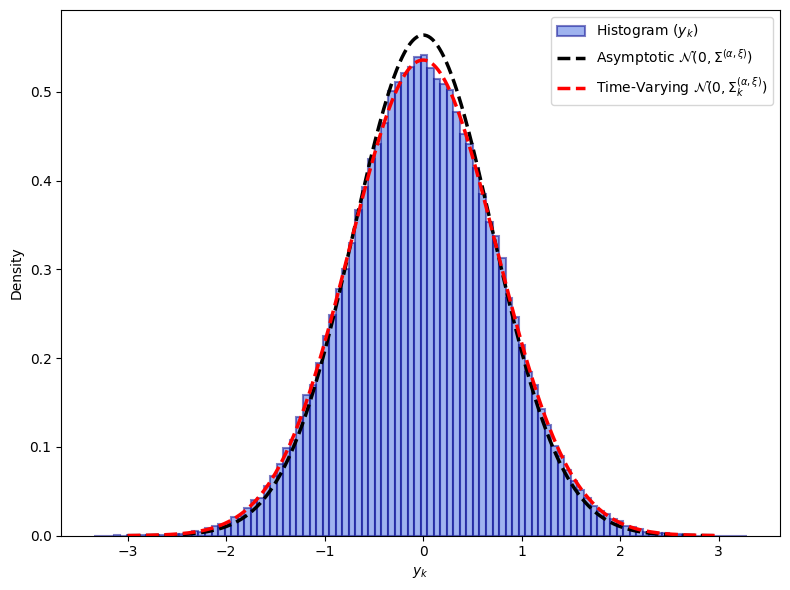}
         \caption{$\xi=0.9$, $k=5000$}
         \label{fig:0.9-5000}
     \end{subfigure}
        \caption{Comparing the Gaussian approximations for $y_k$ at various time instants.}
        \label{fig:comparison}
\end{figure*}

\subsection{Alternate Gaussian approximation}
In the preceding results, we presented two distributional approximations for the rescaled iterate $y_k$: Theorem \ref{thm_main:main_thm} for finite-time $k$ and Corollary \ref{cor:asymp_gaussian_rate} for asymptotic approximation. Beyond these, one can consider alternate Gaussian approximations by modifying the underlying covariance matrix. For instance, while the asymptotic covariance $\Sigma^{(\alpha, \xi)}$ satisfies the Lyapunov equation \eqref{eq:lyap_eq}, this property can be extended to finite-time regime. One can substitute the time-varying drift $J^{(\alpha, \xi)}_k$ and use $\hat{\Sigma}^{(\alpha, \xi)}_k$ that satisfies the following ``time-varying Lyapunov equation'':
\begin{align}\label{eq:time_varying_lyap}
    J^{(\alpha, \xi)}_k\hat{\Sigma}^{(\alpha, \xi)}_k +\hat{\Sigma}^{(\alpha, \xi)}_k(J^{(\alpha, \xi)}_k)^T+\Sigma_b&=0.
\end{align}
By definition, $\Sigma_k^{(\alpha, \xi)}$ is governed by a recursive update \eqref{eq:sigma_k_eq}. Thus, characterizing it for moderately large values of $k$ can be computationally expensive since each iteration requires $d^2$ updates. Instead, one can solve the time-varying Lyapunov equation \eqref{eq:time_varying_lyap} to efficiently approximate $\Sigma_k^{(\alpha, \xi)}$. Similar to Corollary \ref{cor:asymp_gaussian_rate}, we characterize the error between $\Sigma^{(\alpha, \xi)}_k$ and $\hat{\Sigma}^{(\alpha, \xi)}_k$, and then use triangle inequality to establish convergence to the aforementioned Gaussian. Proof of the corollary can be found in Appendix \ref{sec:sigmak_and_hatsigmak}.

\begin{corollary}\label{cor:sigmak_and_hatsigmak}
    Let $\eta=\min(\iota_V/2, 3\gamma/4)$ and $\mathcal{E}_0=\E[\|y_0\|^2]$. When $\xi\in(0,1)$ and $K$  is large enough, then under Assumptions \ref{assump:iterate}-\ref{assump:step-size}, the sequence $\{y_k\}_{k\geq 0}$ given by rescaling the iterate $x_k$ in update equation \eqref{eq:SA_rec2} satisfies the following bounds for all $k\geq 1$.
    \begin{align*}
        d_{\mathcal{W}}(y_k, (\hat{\Sigma}^{(\alpha, \xi)}_k)^{1/2}Z)&\leq 
        \mathcal{O}\left(\frac{1}{(k+K)^{\frac{\xi\delta}{2}}}\right)+\mathcal{O}\left(\frac{\log(k+K)}{(k+K)^{\frac{\xi}{2}}}\right)+\mathcal{O}\left(\mathcal{E}_0\tau_ke^{-\frac{\eta\alpha}{1-\xi}\left((k+K)^{1-\xi}\right)}\right)\\
        &~+\mathcal{O}\left(\max\left(\frac{1}{(k+K)^\xi}, \frac{1}{(k+K)^{2-2\xi}}\right)\right).
    \end{align*}
\end{corollary}
We note that the above corollary is obtained only for the case $\xi\in (0, 1)$. Recall that $J_k^{(\alpha, \xi)}=J^{(\alpha, \xi)}$ at the boundaries ($\xi\in \{0, 1\}$) and thus the corresponding bounds naturally reduce to the asymptotic results in statements 1 and 3 of Corollary \ref{cor:asymp_gaussian_rate}. Furthermore, we observe a phase transition for the convergence rate similar to Corollary \ref{cor:asymp_gaussian_rate}. Assuming $\delta=1$, we observe that for $\xi\geq 0.8$, the additional term becomes the bottleneck in the convergence rate.

The foregoing results offer various degrees of Gaussian approximation for $y_k$. Specifically, we define
\begin{align}\label{eq:slope_def}
    \mathfrak{m}(\Sigma)=-\lim_{k\to \infty} \frac{\log d_{\mathcal{W}}(y_k, (\Sigma)^{1/2}Z)}{\log (k+K)},
\end{align}
i.e., $\mathfrak{m}(\Sigma)$ is the asymptotic negative slope of the log-log plot. Consequently, the asymptotic rate of convergence is given by
\begin{align*}
    d_{\mathcal{W}}(y_k, \Sigma^{1/2}Z)\approx \mathcal{O}\left(\frac{1}{(k+K)^{\mathfrak{m}(\Sigma)}}\right).
\end{align*}
For different approximations, we observe the following relation
\begin{align*}
    \mathfrak{m}(\Sigma_k^{(\alpha, \xi)})=\xi/2;~~\mathfrak{m}(\hat{\Sigma}_k^{(\alpha, \xi)})=\min(\xi/2, 2-2\xi);~~\mathfrak{m}(\Sigma^{(\alpha, \xi)})=\min(\xi/2, 1-\xi).
\end{align*}
We plot these functions in Figure \ref{fig:function_plot}.  Clearly, we observe $\mathfrak{m}(\Sigma^{(\alpha, \xi)})\leq \mathfrak{m}(\hat{\Sigma}_k^{(\alpha, \xi)}) \leq \mathfrak{m}(\Sigma_k^{(\alpha, \xi)})$ for all $\xi$. In other words, $\mathcal{N}(0, \Sigma_k^{(\alpha, \xi)})$ provides the best approximation, then the solution to the time-varying Lyapunov equation $\mathcal{N}(0, \hat{\Sigma}_k^{(\alpha, \xi)})$, and finally the asymptotic solution $\mathcal{N}(0, \Sigma^{(\alpha, \xi)})$. Nevertheless, there exists a trade-off between the statistical accuracy and computational cost. The sequence $\Sigma_k^{(\alpha, \xi)}$ is the most compute intensive due to its recursive formulation in Eq. \eqref{eq:sigma_k_eq}. In contrast, $\hat{\Sigma}_k^{(\alpha, \xi)}$ is more tractable and provides a middle ground by solving the Lyapunov equation \eqref{eq:time_varying_lyap} directly, while the asymptotic covariance $\Sigma^{(\alpha, \xi)}$ remains the most efficient approximation of $\E[y_ky_k^T]$ for all $k$.

\begin{figure}[t!]
    \centering
    \begin{tikzpicture}
    \begin{axis}[
        axis lines = left,
        xlabel = {Step-size parameter $\xi$},
        ylabel = {$\mathfrak{m}$ (negative slope)},
        xmin = 0, xmax = 1.05, 
        ymin = 0, ymax = 0.7,
        domain = 0:1,          
        samples = 100,
        legend pos = north west,
        legend cell align = left,
        grid = major,
        width=0.6\textwidth,
        height=7cm,
        clip=false 
    ]

    \addplot [
        color=blue,
        thick,
        dashed
    ]
    {x/2};
    \addlegendentry{$\mathfrak{m}(\Sigma^{(\alpha, \xi)}_k) = \xi/2$}

    \addplot [
        color=red,
        thick,
        dotted
    ]
    {x/2 < 1-x ? x/2 : 1-x};
    \addlegendentry{$\mathfrak{m}(\Sigma^{(\alpha, \xi)}) = \min(\xi/2, 1-\xi)$}

    \addplot [
        color=black!60!green, 
        thick,
        dash dot
    ]
    {x/2 < 2-2*x ? x/2 : 2-2*x};
    \addlegendentry{$\mathfrak{m}(\hat{\Sigma}^{(\alpha, \xi)}_k) = \min(\xi/2, 2-2\xi)$}

    \addplot[
        only marks,
        mark=*, 
        mark options={fill=white, draw=black, thick},
        mark size=2.5pt
    ] coordinates {(1, 0)};

    \addplot[
        only marks,
        mark=*, 
        mark options={fill=black, draw=black, thick},
        mark size=2.5pt
    ] coordinates {(1, 0.5)};

    
    \draw[gray, thin] (2/3, 0) -- (2/3, 1/3) -- (0, 1/3);
    \addplot[only marks, mark=square*, mark size=1.5pt] coordinates {(2/3, 1/3)};
    \node[above left, font=\small, yshift=2pt] at (2/3, 1/3) {$(2/3, 1/3)$};

    \draw[gray, thin] (0.8, 0) -- (0.8, 0.4) -- (0, 0.4);
    \addplot[only marks, mark=square*, mark size=1.5pt] coordinates {(0.8, 0.4)};
    \node[below right, font=\small, xshift=4pt, yshift=-2pt] at (0.8, 0.4) {$(0.8, 0.4)$};

    \end{axis}
    \end{tikzpicture}
    \caption{$\mathfrak{m}$ (Negative slope of log-log plot) as a function of step-size parameter $\xi$.}
    \label{fig:function_plot}
\end{figure}

\subsection{Central Limit Theorem for step-size based averaging} \label{sec:clt}
Consider the empirical average estimates for a zero mean i.i.d. noise sequence $\{b_k\}_{k\geq 0}$:
\begin{align*}
    x_{k+1}&=\frac{1}{k+1}\sum_{i=0}^{k}b_i.
\end{align*}
Then, central limit theorem (CLT) states that the rescaled iterate $x_k\sqrt{k}$ converges in distribution to a Gaussian limit $\mathcal{N}(0, \Sigma_b)$. Taking a step forward, Berry-Esseen theorem and Stein's method for Gaussian approximation establishes that the finite-time rate of convergence to this Gaussian limit is $\mathcal{O}(1/\sqrt{k})$. Note that these empirical estimates $x_k$ can also be written iteratively as
\begin{align*}
    x_{k+1}&=x_k+\frac{1}{k+1}(-x_k+b_k).
\end{align*}
with $x_0=0$. Therefore, a natural generalization is to consider general step-size based averaging where the updates are given as
\begin{align}\label{eq:clt_weighted}
    \begin{split}
        x_{k+1}&=x_k+\alpha_k(-x_k+b_k)\\
        &=(1-\alpha_k)x_k+\alpha_kb_k.
    \end{split}
\end{align}
While classical asymptotic convergence results show that $x_k\to 0$ a.s. and $y_k:=x_k/\sqrt{\alpha_k}$ converges in distribution to a Gaussian limit \cite{benveniste1990adaptive, Borkar2023, nguyen2025almost}, we provide a Berry-Esseen style bound for Wasserstein-1 distance for non-asymptotic convergence for this step-size based CLT.

\begin{proposition}\label{prop:clt}
    The sequence $\{y_k\}_{k\geq 0}$ where $x_k$ is given by the update equation \eqref{eq:clt_weighted} satisfies the following bounds for all $k\geq 0$.
    \begin{enumerate}
        \item When $\xi=0$, and $\alpha$ is small enough, we have{\normalfont :}
        \begin{align*}
            d_{\mathcal{W}}\left(y_k, \frac{\Sigma_b^{1/2}Z}{2}\right)&\leq \mathcal{O}\left(\sqrt{\alpha}\log\left(\frac{1}{\alpha}\right)\right).
        \end{align*}
        \item When $\xi\in (0, 1)$ and $K$ is large enough, we have{\normalfont :}
        \begin{align*}
            d_{\mathcal{W}}\left(y_k, \frac{\Sigma_b^{1/2}Z}{2}\right)&\leq 
            \mathcal{O}\left(\frac{\log(k+K)}{(k+K)^{\frac{\xi}{2}}}\right)+\mathcal{O}\left(\frac{1}{(k+K)^{1-\xi}}\right).
        \end{align*}
        \item When $\xi=1$, $\alpha\geq 1$ and $K\geq 2\alpha-1$, we have {\normalfont :}
        \begin{align*}
            d_{\mathcal{W}}\left(y_k, \Sigma_b^{1/2}Z\right)\leq \mathcal{O}\left(\frac{\log(k+K)}{\sqrt{k+K}}\right).
        \end{align*}
    \end{enumerate}
\end{proposition}
    
Note that for the specific case when $x_0=0$ and $\alpha_k=1/(k+1)$, we obtain the finite-time distributional convergence rate for the classical CLT-setting by using statement 3 of the proposition. It is worth noting that the step-size choice in the CLT-setting, i.e., $\alpha_k=1/(k+1)$ (or $\alpha=1$) is sitting at the edge of our analysis. More specifically, we cannot directly apply Theorem \ref{thm_main:main_thm} to the CLT-setting because of the condition $\iota_V\alpha>2$ which in this instance translates to $\alpha>4$. Thus, we carry out a more delicate analysis which helps us to relax this condition and obtain rate of convergence for CLT. We leverage a key technical property that $J_F$ is $-I$ in this case and, in particular, it is a symmetric negative definite matrix. As mentioned earlier, we can avoid the $V$-weighted norm from the solution of the Lyapunov equation and directly work with $\ell_2$-norm to get sharper constants. We present these modifications in the specific context of step-size based averaging in the above proposition. Extending to any symmetric negative definite matrix is straightforward. For more details, please refer to the proof given in Appendix \ref{sec:clt_proof}.

\subsection{Lower bound on convergence rate for SA}\label{sec:lower bound}
In this section, we present a matching lower bound for the rate of convergence SA to normality, thereby confirming the sharpness of Theorem \ref{thm_main:main_thm} (up to log factors).
The proof of the proposition can be found in Appendix \ref{sec:lower_boundSA}.

\begin{figure}[t!]
    \centering
    \begin{tikzpicture}
        \begin{axis}[
            width=\linewidth, 
            height=4.5cm,
            xlabel={Step-size parameter $\xi$},
            ylabel={$\mathfrak{m}$ (negative slope)},
            ytick={0.1, 0.2, 0.3, 0.4, 0.5, 0.6},
            grid=major, 
            grid style={dashed, gray!30},
            xmin = 0, xmax = 1.05, 
            ymin = 0, ymax = 0.7,
            domain = 0:1,          
            samples = 100,
            legend pos=north west, 
            legend style={draw=black!50, fill=white, rounded corners=2pt, nodes={scale=1, transform shape}}, 
             width=0.6\textwidth,
            height=7cm,
            clip=false
        ]
        
        \addplot[color=blue, mark=*, mark size=2pt, thick, dashed] coordinates {
            (0.1, 0.079) (0.2, 0.115) (0.3, 0.154) (0.4, 0.197) (0.5, 0.243) 
            (0.6, 0.292) (0.7, 0.338) (0.8, 0.387) (0.9, 0.404) (1, 0.493)
        };

        \addlegendentry{$\hat{\mathfrak{m}}(\Sigma_k^{(\alpha, \xi)})$}

        \addplot [
            color=blue,
            thick,
            dash dot
        ]
        {x/2};
        \addlegendentry{$\mathfrak{m}(\Sigma^{(\alpha, \xi)}_k)$}
        
        \addlegendentry{$\hat{\mathfrak{m}}(\Sigma^{(\alpha, \xi)})$}
        
        \addplot[color=red, mark=square*, mark size=2pt, thick, dotted] coordinates {
            (0.1, 0.074) (0.2, 0.103) (0.3, 0.14)  (0.4, 0.18)  (0.5, 0.222) 
            (0.6, 0.263) (0.7, 0.303) (0.8, 0.244) (0.9, 0.102) (1, 0.52)
        };
    
        \addplot [
            color=red,
            thick,
            dotted
        ]
        {x/2 < 1-x ? x/2 : 1-x};
        \addlegendentry{$\mathfrak{m}(\Sigma^{(\alpha, \xi)})$}

        \addplot[
            only marks,
            mark=*, 
            mark options={fill=white, draw=black, thick},
            mark size=2.5pt
        ] coordinates {(1, 0)};

        \node [
            draw=black!50,
            fill=white,
            rounded corners=2pt,
            anchor=north east,
            align=left,
            font=\scriptsize 
        ] at (rel axis cs: 0.97, 0.97) {
            \tikz[baseline=-0.5ex]\draw[blue, dashed, thick] (0,0) -- (0.1,0); Time-varying Gaussian \\
            \tikz[baseline=-0.5ex]\draw[red, dotted, thick] (0,0) -- (0.1,0); Asymptotic Gaussian
        };
        
        \end{axis}
    \end{tikzpicture}
    \captionof{figure}{$\hat{\mathfrak{m}}$ (estimated negative slope of log-log plot) as a function of step-size parameter $\xi$}
    \label{fig:slope}
\end{figure}

\begin{proposition}\label{lem:lower_boundSA}
    Let $\xi\in [0, 1)$ and $\{b_k\}_{k\geq 0}$ be an i.i.d. noise sequence with zero-mean. Furthermore, assume that $\E[b_1b_1^T]=I$, $\E[\|b_i\|^4]<\infty$, and there exists a unit vector $\mathfrak{u}\in \mathbb{R}^d$ such that $\E[(\mathfrak{u}^Tb_1)^3]\neq 0$. Consider the following SA update:
    \begin{align*}
        x_{k+1} = (1-\alpha_k)x_k+\alpha_kb_k.
    \end{align*}
    where $\{\alpha_k\}_{k\geq 0}$ satisfies Assumption \ref{assump:step-size}. Then, the following relations hold.
    \begin{enumerate}
        \item The Wasserstein-1 distance from the time-varying Gaussian satisfies:
        \begin{align*}
            d_{\mathcal{W}}(y_{k}, (\Sigma_k^{(\alpha, \xi)})^{1/2}Z)  \geq   \Omega\left(\sqrt{\alpha_k}\right)=\Omega((k+K)^{-\xi/2}).
        \end{align*}
        \item The Wasserstein-1 distance from the asymptotic Gaussian satisfies:
        \begin{align*}
            d_{\mathcal{W}}(y_{k}, \Sigma^{1/2}Z)  \geq   \Omega\left(\max\left(\sqrt{\alpha_k}, \frac{1}{(k+K)^{1-\xi}}\right)\right)=\Omega((k+K)^{-\min(\xi/2, 1-\xi)}).
        \end{align*}
    \end{enumerate}
\end{proposition}

It is well known that the Wasserstein-1 convergence for Central Limit Theorem is lower-bounded by $\Omega(1/\sqrt{k})$ \cite{rio2009upper, bobkov2018berry}. However, this corresponds only to the special case $\xi=1$ in our framework (as discussed in Section \ref{sec:clt}). Proposition \ref{lem:lower_boundSA} complements the CLT lower bound by covering other choices of step-size schedules, thus showing the sharpness of our results for the entire regime of $\xi$. Moreover, we validate the tightness of our result by conducting numerical simulations which demonstrates that our convergence rate is achieved by well-known distributions. We ran multiple sample paths of linear SA perturbed by exponentially distributed noise (we subtracted the mean from the noise to make it zero-mean) and used linear regression on the log-log plot between Wasserstein-1 distance and the number of iterations for various choices of $\xi$ to calculate $\hat{\mathfrak{m}}$ as an estimate for $\mathfrak{m}$ defined in Eq. \eqref{eq:slope_def}. As shown in Figure \ref{fig:slope}, empirical convergence rates align perfectly with the theoretical rates obtained in Theorem \ref{thm_main:main_thm} and Corollary \ref{cor:asymp_gaussian_rate}. Please refer to Appendix \ref{sec:simulations} for more details on the simulation setup.

The proof of the proposition follows a similar methodology as used in the upper bound. Specifically, for our chosen instance, we construct the corresponding DOUG process $\{z_k\}_{k\geq 0}$ and establish that the convergence rate for Wasserstein-1 distance is lower bound by $\Omega(\sqrt{\alpha_k})$. Equipped with this result, we apply the reverse triangle inequality to obtain:
\begin{align*}
    d_{\mathcal{W}}(y_k, (\Sigma_k^{(\alpha, \xi)})^{1/2}Z)\geq d_{\mathcal{W}}(z_k, (\Sigma_k^{(\alpha, \xi)})^{1/2}Z)-d_{\mathcal{W}}(y_k, z_k).
\end{align*}
Thereafter, we show that the residual error $d_{\mathcal{W}}(y_k, z_k)$ is a higher-order term. Consequently, this term becomes asymptotically negligible, which implies that the $\Omega(\sqrt{\alpha_k})$ lower bound of the DOUG process strictly dominates and yields the final rate. To establish the convergence rate to the asymptotic Gaussian, we additionally establish a tight lower bound for $d_{\mathcal{W}}((\Sigma_k^{(\alpha, \xi)})^{1/2}Z, \Sigma^{1/2}Z)$ and combine it with the former result.

\begin{figure*}[t!]
    \centering
    \begin{tikzpicture}
        \begin{axis}[
            width=\linewidth, 
            height=4.5cm,
            xlabel={Step-size parameter $\xi$},
            ylabel={$\mathfrak{m}$ (negative slope)},
            ytick={0.1, 0.2, 0.3, 0.4, 0.5, 0.6},
            grid=major, 
            grid style={dashed, gray!30},
            xmin = 0, xmax = 0.52, 
            ymin = 0, ymax = 0.7,
            domain = 0:0.5,          
            samples = 100,
            legend pos=north west, 
            legend style={draw=black!50, fill=white, rounded corners=2pt, nodes={scale=1, transform shape}}, 
             width=0.6\textwidth,
            height=7cm,
            clip=false
        ]
        
        \addplot[color=blue, mark=*, mark size=2pt, thick, dashed] coordinates {
            (0.1, 0.098) (0.2, 0.195) (0.3, 0.283) (0.4, 0.374) (0.5, 0.462) 
        };

        \addlegendentry{$\hat{\mathfrak{m}}(\Sigma^{(\alpha, \xi)}_k) ~(\text{Laplace Distribution})$}

        \addplot[color=red, mark=square*, mark size=2pt, thick, dashed] coordinates {
            (0.1, 0.135) (0.2, 0.27) (0.3, 0.32)  (0.4, 0.386)  (0.5, 0.431) 
        };
        \addlegendentry{$\hat{\mathfrak{m}}(\Sigma_k^{(\alpha, \xi)}) ~(\text{Uniform Distribution})$}

        \addplot [
            color=black,
            thick,
            dotted
        ]
        {x};
        \addlegendentry{$\mathfrak{m}(\Sigma_k^{(\alpha, \xi)})=\xi ~(\text{Conjecture})$}

        \end{axis}
    \end{tikzpicture}
    \captionof{figure}{$\hat{\mathfrak{m}}$ (estimated negative slope of log-log plot) as a function of step-size parameter $\xi$}
    \label{fig:slope_dist}
\end{figure*}

Prior literature indicates that convergence to the asymptotic Gaussian can be accelerated if higher moments of the random variable match with the Gaussian. Notably, the author in \cite{fathi2018higherordersteinkernelsgaussian} shows that if the $n+1$ moments of the noise sequence match with those of a Gaussian, then the Zolotarev distance to the limiting Gaussian in CLT setting converges as $\mathcal{O}(k^{-n/2})$ under some additional technical conditions. We remark that the lower bound established in the preceding proposition is obtained for a noise sequence where the first two moments match but the third one does not ($n=1$). However, it is also noted in \cite{bobkov2018berry} that for symmetric distributions (where odd moments are zero, i.e., three moments match) the Wasserstein-$p$ convergence rate for CLT can be improved to $\mathcal{O}(1/k)$. Translating this insight into our framework, an interesting future direction is to investigate that if the $n+1$ moments of the noise in SA match with the asymptotic Gaussian, then whether we can improve the rate of convergence 
\begin{align*}
    d_{\mathcal{W}}(y_{k}, (\Sigma_k^{(\alpha, \xi)})^{1/2}Z)\stackrel{?}{=}\mathcal{O}(\alpha_k^{n/2}).
\end{align*} 

To empirically examine the above claim for symmetric distributions ($n=2$), we conducted numerical simulations for various choices of $\xi$ similar to Figure \ref{fig:slope} but with different noise models. Specifically, we chose two symmetric density functions: Laplace and uniform distribution. The Laplace distribution was chosen with zero mean and scale parameter as 1 while the uniform distribution supported on $[-1, 1]$. As illustrated in Figure \ref{fig:slope_dist}, results from the experiments suggest that the convergence rate is sensitive to the underlying noise distribution. In particular, it improves to $\mathcal{O}(\alpha_k)$ for Laplace and uniform distributions. We leave the theoretical analysis of this claim for future work.

\section{Applications}
Now we will present some immediate applications of our result to derive tight tail bounds and first moment bounds. We then discuss some implications of Theorem \ref{thm_main:main_thm} and Corollary \ref{cor:asymp_gaussian_rate} in the special case of constant step-size SA and SGD.

\subsection{Tight first moment bound}
Since Wasserstien-1 convergence imply convergence of first moment, Theorem \ref{thm_main:main_thm} immediately yields the following corollary. We only present the case $\xi=1$ which has the optimal rate of convergence. It is straightforward to obtain analogous bounds for other choices of step-sizes.
\begin{corollary}\label{cor:mean_err}
    Suppose $\xi=1$ and all the conditions of Theorem \ref{thm_main:main_thm} hold. Then, the sequence $\{x_k\}_{k\geq 0}$ given by update equation \eqref{eq:SA_rec2} satisfies the following first moment error bound for all $k\geq 0$
    \begin{align*}
        \E[\|x_k-x^*\|]&\leq \frac{\sqrt{\alpha}}{\sqrt{k+K}}\E[\|(\Sigma^{(\alpha)})^{1/2}Z\|]+\tilde{\mathcal{O}}\left((k+K)^{-(1+\delta)/2}\right).
    \end{align*}
\end{corollary}
The above bound is conceptually analogous to the finite-time MSE bound established in \cite{haque2025tightfinitetimebounds} as both bounds exhibit a tight dominant term that matches precisely with the asymptotic normal distribution in their specific regimes (first and second moment, respectively). Although, one can obtain error bounds in first moment via Jensen's inequality and the MSE bound in \cite{haque2025tightfinitetimebounds}, however, this approach will lead to a sub-optimal constant since $\E[\|(\Sigma^{(\alpha)})^{1/2}Z\|]\leq \sqrt{\mathrm{Tr}(\Sigma^{(\alpha)})}$. A rigorous comparison on the tightness of the higher term seems non-trivial and is left as future work. 

\subsection{Tail Bounds for the rescaled iterate}

An upper bound on the tail deviations from Gaussian distribution can be obtained via Wasserstein-1 distance \cite{austern2022efficient, fang2022wasserstein, Wang2026manuscript}. Consequently, Theorem \ref{thm_main:main_thm} can be applied to derive non-asymptotic bounds on tail deviations. Such bounds are important for quantifying the probability of rare, large excursions of the SA iterates from their equilibrium.
\begin{proposition}\label{cor:tail_bounds}
     Suppose $\xi=1$ and all the conditions of Theorem \ref{thm_main:main_thm} hold. Furthermore, let $k$ be large enough such that $d_{\mathcal{W}}(y_{k}, (\Sigma^{(\alpha, 1)}_k)^{1/2}Z)\leq 0.5$. Then, for any $a>0$ and unit vector $\mathfrak{u}\in\mathbb{R}^d$, we have
    \begin{align*} 
        \left|\mathbb{P}(\langle y_k, \mathfrak{u}\rangle > a)-\bar{\Phi}_G\left(\frac{a}{\|\mathfrak{u}^T(\Sigma_k^{(\alpha, 1)})^{\frac{1}{2}}\|}\right)\right| &\leq \tilde{\mathcal{O}}\left(\frac{1}{(k+K)^{\delta/4}}\right) \frac{1}{a}.
    \end{align*}
    where $\bar{\Phi}_G(\cdot)$ is the CCDF of the standard Gaussian.
\end{proposition}
Using Mill's ratio, the above relation implies that the tail distribution of $y_k$ contains two components:
\begin{align}\label{eq:mill's}
    \mathbb{P}(\langle y_k, \mathfrak{u}\rangle > a)  &\leq \frac{1}{\sqrt{2\pi\|\mathfrak{u}^T(\Sigma_k^{(\alpha, 1)})^{\frac{1}{2}}\|^2}}\frac{e^{-\frac{a^2}{2\|\mathfrak{u}^T(\Sigma_k^{(\alpha, 1)})^{1/2}\|^2}}}{a}+\tilde{\mathcal{O}}\left(\frac{1}{(k+K)^{\delta/4}}\right) \frac{1}{a}.
\end{align}
Specifically, we have a time-varying Gaussian tail that quantifies the normality of rescaled SA iterates at time $k$ and a decaying $\mathcal{O}(1/a)$ residual tail. While the classical result on asymptotic normality characterizes the limiting tail behavior $y_k$, it offers no insight into its pre-limit regime. Our results bridge this gap by establishing a non-asymptotic tail behavior that recovers the Gaussian limit as $k\to \infty$. We further remark that while we present the result only for $\xi=1$, analogous bounds can also be obtained for other choices of $\xi$. We defer these extensions to Appendix \ref{sec:tail_bounds} to avoid repetition. A similar tail bound can also be established by directly using Corollary \ref{cor:asymp_gaussian_rate} instead of Theorem \ref{thm_main:main_thm}. We provide a detailed discussion on this in Appendix \ref{sec:tightness_tail}.

\subsection{Constant step-size SA}
For constant step-size SA, our Wasserstein-1 bounds from Theorem \ref{thm_main:main_thm} and Corollary \ref{cor:asymp_gaussian_rate} exhibit a similar exponentially fast transient decay as established for MSE bounds. Notably, they closely resemble the Berry-Esseen-style bound obtained for SGD in \cite{WeiLiLouWu2025GaussianApprox}. More precisely, both expressions have an exponentially decaying transient term and a steady-state remainder term of the order $\tilde{\mathcal{O}}(\sqrt{\alpha})$ (ignoring log factors). Although, the rate of decay for the transient term in our bound might not be as tight as in \cite{WeiLiLouWu2025GaussianApprox}, it can be easily improved if $J_F$ is symmetric (for eg., in SGD where $J_F=-\nabla^2 g(x^*)$ is negative definite) by directly working with $\ell_2$-norm instead of the Lyapunov Eq. $\eqref{eq:lyap_gen_<1_I}$. 

The literature is quite rich on the asymptotic distributional behavior for constant step-size SA \cite{dieuleveut2020bridging, yu2020analysisconstantstepsize, Zaiwei2021, huo2024collusion, wei2025online,  zedong2026icml}. These prior works first establish the existence of the stationary distribution of the process $y_k$ as $k\to \infty$ and then work directly under the stationary regime. The authors in \cite{zedong2026icml} further show 
\begin{align}\label{eq:zedong_res}
    d_{\mathcal{W}}(y_\infty, \Sigma^{1/2}Z)\leq \mathcal{\tilde{O}}(\sqrt{\alpha}).
\end{align}
We remark that taking the limit $k\to \infty$ from Corollary \ref{cor:asymp_gaussian_rate} results in $\lim_{k\to \infty}d_{\mathcal{W}}(y_k, \Sigma^{1/2}Z)\leq \mathcal{\tilde{O}}(\sqrt{\alpha})$.
Now, to obtain a similar result as Eq. \eqref{eq:zedong_res} for the stationary process from this limit, a key step is to establish the existence of the stationary distribution and prove appropriate regularity conditions to interchange the limit with the Wasserstein distance which we leave as future work.

\subsection{Stochastic Gradient Descent}
Consider the optimization problem $\min_{x\in\mathbb{R}^d} g(x)$ where $g(x)$ is $\mu$-strongly convex and $L$-smooth with the unique minimizer $x^*$. When the exact gradients are inaccessible, a popular approach is stochastic gradient descent (SGD), which updates the estimate using noisy gradients $\nabla g_k(x_k)$:
\begin{align}\label{eq:sgd}
    x_{k+1}=x_k-\alpha_k\nabla g_k(x_k).
\end{align}
To cast SGD into the framework of Eq. \eqref{eq:SA_rec2}, we define $F(x_k)=-\nabla g(x_k)$ and $M_k:=\nabla g(x_k)-\nabla g_k(x_k)$. Here, we suppose that the noise $M_k$ admits a decomposition into a multiplicative martingale difference and an additive i.i.d. term as in Assumption \ref{assump:noise}. It is clear that now SGD transforms into Eq. \eqref{eq:SA_rec2}. Furthermore, smoothness and strong convexity ensure that the quadratic Lyapunov function $\Phi(x)=\|x\|^2/2$ naturally satisfies Assumption \ref{assump:iterate}. Finally, assuming $g(x)$ has a Lipschitz Hessian guarantees that $F(x)$ satisfies the local linear approximation (Assumption \ref{assump:operator}) with $J_F=-\nabla^2 g(x^*)$ and $\delta=1$. Thus, Theorem \ref{thm_main:main_thm} immediately characterizes the finite-time distribution of SGD.

\section{DOUG and its connection to SA}
As discussed previously, to accurately characterize the non-asymptotic distribution of the rescaled iterates $\{y_k\}_{k\geq 0}$, we analyze an auxiliary process $z_k$. We start by formally introducing DOUG and establishing its normality as follows.
\subsection{Normality of the DOUG process}
\begin{definition}
    Given a matrix $H$, a noise sequence $\{\chi_k\}$, and a step-size schedule $\{\alpha_k=\alpha/(k+K)^\xi\}_{k\geq 0}$, we define DOUG as a discrete time process given by the following update:
    \begin{align}\label{eq:discrete_ou}
        z_{k+1} &= (I + \alpha_k H_k) z_k + \sqrt{\alpha_k} \chi_k,
    \end{align}
    where $H_k=H+\xi I/(2\alpha(k+K)^{1-\xi})$.
\end{definition}

To analyze the distributional behavior of $\{z_k\}_{k\geq 0}$, we will impose standard assumptions on the drift matrix $H$ and the noise sequence $\{\chi_k\}_{k\geq 0}$. In particular, we suppose that the matrix $H$ is Hurwitz and the noise sequence is i.i.d. and has bounded third moment. Let $V_H$ denote the unique solution to the Lyapunov equation:
\begin{align}\label{eq:lyap_H}
    HV_H+V_HH^T+I=0
\end{align}
Additionally, we define $\iota_{V_H}=1/(4\lambda^{max}_{V_H})$. With a slight abuse of notation, we redefine $\Sigma_k^{(\alpha, \xi)}$ as the covariance of $z_k$. Let $\E[\chi_k\chi_k^T]=\Sigma_{\chi}$. Then, the covariance sequence satisfies the following recursive relation:
\begin{align*}
    \Sigma_{k+1}^{(\alpha, \xi)}&=(I + \alpha_k H_k)\Sigma_k^{(\alpha, \xi)}(I + \alpha_k H_k)^T+\alpha_k\Sigma_{\chi}.
\end{align*}
We now present our second main result as the following theorem. A more detailed version with all the constants characterized explicitly and the proof can be found in Appendix \ref{sec:proof_weighted_stein}.

\begin{theorem}\label{thm_main:weighted_stein}
    Suppose that the noise sequence $\{\chi_k\}_{k\geq 0}$ is zero-mean i.i.d. and has bounded third moment. Let $\mathcal{E}_0=\E[\|z_0\|^2]$. Then, the iteration $\{z_k\}_{k\geq 0}$ given by the update equation \eqref{eq:discrete_ou} satisfies the following bounds for all $k\geq 1$.
    \begin{enumerate}
       \item When $\xi=0$ and $\alpha$ is small enough, we have {\normalfont :}
       \begin{align*}
            &d_{\mathcal{W}}(z_{k}, (\Sigma_k^{(\alpha, 0)})^{1/2}Z)  \leq   \mathcal{O}\left(\sqrt{\alpha}\log\left(\frac{1}{\alpha}\right)\right)+\mathcal{O}\left(\mathcal{E}_0e^{-\iota_{V_H}\alpha k}\right).
        \end{align*}
        \item When $\xi\in (0, 1)$ and $K$ is large enough, we have {\normalfont :}
        \begin{align*}
            &d_{\mathcal{W}}(z_{k}, (\Sigma_k^{(\alpha, \xi)})^{1/2}Z)  \leq   \mathcal{O}\left(\frac{\log(k+K)}{(k+K)^{\xi/2}}\right)+\mathcal{O}\left(\mathcal{E}_0e^{-\frac{\iota_{V_H}\alpha}{1-\xi}\left((k+K)^{1-\xi}\right)}\right).
        \end{align*}
        \item When $\xi=1$, $\iota_{V_H}\alpha>1/2$, and $K$ is large enough, we have {\normalfont :}
        \begin{align*}
            d_{\mathcal{W}}(z_{k}, (\Sigma_k^{(\alpha, 1)})^{1/2}Z)  &\leq \mathcal{O}\left(\frac{\log(k+K)}{\sqrt{k+K}}\right)+\mathcal{O}\left(\frac{\mathcal{E}_0}{(k+K)^{\iota_{V_H}\alpha}}\right).
        \end{align*}
   \end{enumerate}
\end{theorem}

These bounds imply that $d_{\mathcal{W}}(z_{k}, (\Sigma_k^{(\alpha, \xi)})^{1/2}Z)\leq \tilde{\mathcal{O}}(\sqrt{\alpha_k})$. By mirroring the proof of Corollary \ref{cor:asymp_gaussian_rate}, we obtain the convergence rate for DOUG to the asymptotic Gaussian as
\begin{align}\label{eq:doug-asymp}
    d_{\mathcal{W}}(z_{k}, (\Sigma^{(\alpha, \xi)})^{1/2}Z)\leq \tilde{\mathcal{O}}\left((k+K)^{-\xi/2}\right)+\mathcal{O}\left(\max\left((k+K)^{-\xi}, \frac{\mathbbm{1}_{\xi\in (0, 1)}}{(k+K)^{1-\xi}}\right)\right)
\end{align}
Since $\Sigma_k^{(\alpha, \xi)}$ is time-varying, it inherently captures the transient behavior of DOUG. Its convergence to the asymptotic covariance is explicitly governed by the second term in Eq. \eqref{eq:doug-asymp}. Furthermore, our results suggest that irrespective of the underlying noise distribution, the statistical impact of the noise $\chi_k$ on DOUG iterates $z_k$ can be decomposed into a Gaussian component $\Sigma_{\chi}^{1/2}Z$ and non-Gaussian residual component. The influence of the non-Gaussian residual component on the distribution of $z_k$ ultimately vanishes and is precisely characterized as $\tilde{\mathcal{O}}(\sqrt{\alpha_k})$ in Theorem \ref{thm_main:weighted_stein}. Consequently, if the noise $\chi_k$ is strictly Gaussian, then the residual component is zero which gives $d_{\mathcal{W}}(z_{k}, (\Sigma_k^{(\alpha, \xi)})^{1/2}Z)=0$.

It is worth noting that for Unadjusted Langevin Algorithm (ULA), an $\mathcal{O}(\alpha_k)$ convergence rate to the asymptotic distribution has been established in prior sampling literature (for eg., \cite{pages2023unadjusted}[Theorem 2.3] and \cite{durmus2019high}[Theorem 8]). We highlight that one can easily modify our analysis to achieve these improved rates for a simplified DOUG. Specifically, suppose we modify DOUG as 
\begin{align*}
    z_{k+1} &= (I + \alpha_k H) z_k + \sqrt{\alpha_k} \chi_k,
\end{align*}
where $\chi_k$ is Gaussian and the time-varying drift $H_k$ is replaced by the time-independent constant $H$. Then, the rate of convergence to asymptotic Gaussian distribution is improved to $\mathcal{O}(\alpha_k)$. Nevertheless, compared to the ULA analysis, our method is distribution free, requiring only the bounded third moment assumption. However, our approach fundamentally leverages the linear update for $z_k$ which generally does not hold for ULA. Thus, an interesting future research direction is to investigate the application of DOUG in the analysis of ULA to improve upon the state-of-the-art sampling bounds. 

\subsection{DOUG applied to SA}
Recall that DOUG \eqref{eq:discrete_ou} is the discrete analog of the following SDE:
\begin{align*}
    dX_t = H X_t dt + \Sigma^{1/2}_{\chi}dW_t.
\end{align*}
Furthermore, it is well known from the prior literature that the joint-process asymptotic behavior for rescaled iterate $\{y_k\}_{k\geq 0}$ is governed by the SDE
\begin{align*}
    dX_t = J^{(\alpha, \xi)} X_t dt + \Sigma^{1/2}_bdW_t,
\end{align*}
where $\Sigma_b$ is the asymptotic covariance of $M_k$. Therefore, to examine the finite-time distribution of the rescaled SA iterate $\{y_k\}_{k\geq 0}$, we construct the DOUG process by setting the drift matrix $H$ as $J^{(\alpha, \xi)}$ and the noise sequence $\{\chi_k\}_{k\geq 0}$ as $\{b_k\}_{k\geq 0}$. Specifically, the SA inspired DOUG is given by
\begin{align}\label{eq:doug_SA}
    z_{k+1} &= (I + \alpha_k J_k^{(\alpha, \xi)})z_k + \sqrt{\alpha_k} b_k.
\end{align}
Recall that the full noise in SA \eqref{eq:SA_rec} is given by $M_k$ whereas the DOUG process defined above only has the additive component $b_k$. Thus, it is natural to define DOUG for the corresponding SA as
\begin{align}\label{eq:doug_SAful}
    \hat{z}_{k+1} &= (I + \alpha_k J_k^{(\alpha, \xi)})\hat{z}_k + \sqrt{\alpha_k} M_k.
\end{align}
However, there are two problems with such a process. First, the noise $M_k$ is implicitly dependent on the $\sigma$-field generated by the iterate $x_k$. Thus, analyzing the distribution of $\hat{z}_k$ directly appears to be hard. Second, we have the property that $\E[\|A(v_k, x_k)\|^2]\leq A_1\E[\|x_k-x^*\|^2]\to 0$, i.e., the multiplicative term vanishes in the limit. Consequently, the distributional behavior of $y_k$ is dictated by DOUG with purely additive noise. This motivates us to define DOUG with purely additive noise \eqref{eq:doug_SA} for the SA. To handle this discrepancy, we first obtain Wasserstein-1 convergence bound for $\{z_k\}_{k\geq 0}$ \eqref{eq:doug_SA} using Theorem \ref{thm_main:weighted_stein}. Then, we use the triangle inequality and a coupling argument to get a bound on the Wasserstein-1 distance for $\{\hat{z}_k\}_{k\geq 0}$. We defer the detailed discussion about this to the proof sketch in Section \ref{sec:proof_sketch}.

\subsection{Lower bound on convergence rate for DOUG}
Now we present a matching lower bound for the rate of convergence of DOUG, thereby confirming the sharpness of Theorem \ref{thm_main:weighted_stein} (up to log factors).
The proof of the proposition can be found in Appendix \ref{sec:lower_bound}.

\begin{proposition}\label{lem:lower_bound}
    Let $\xi\in [0, 1)$. Then, there exists an instance for DOUG \eqref{eq:discrete_ou} such that the following holds for all $k\geq 0$
    \begin{align*}
        d_{\mathcal{W}}(z_{k}, (\Sigma_k^{(\alpha, \xi)})^{1/2}Z)  \geq   \Omega\left(\sqrt{\alpha_k}\right).
    \end{align*}
\end{proposition}
The proof leverages Kantorovich-Rubinstein duality to establish a lower bound on the Wasserstein-1 distance by choosing an appropriate test function with nice regularity properties. We then apply a Lindeberg-type decomposition to express the difference as a telescoping sum and lower bound each term via a Taylor series expansion.

\section{Proof Sketch}\label{sec:proof_sketch}
In this section, we will discuss the key ideas applied in the proof of Theorem \ref{thm_main:main_thm} and Theorem \ref{thm_main:weighted_stein}. We organize the proof sketch into the following three main subsections. For ease of exposition, we will denote $A(v_k, x_k)=A_k$ and $b(w_k)=b_k$ in the following.

\subsection{Overview of Stein's Method}\label{sec:stein_method}

We begin by presenting the following well-known result that summarizes the key properties for the high dimensional Ornstein-Uhlenbeck (O-U) processes.
\begin{lemma}\label{lem:ou-process}[\cite{arnold1974stochastic}]\label{Steinop}
    Consider the O-U process:
    \begin{align}\label{eq:ou_process}
        dX_t = H X_t\,dt + \Sigma^{1/2}\,dW_t,
    \end{align}
    where $X_0 \in \mathbb{R}^d$. Let $H \in \mathbb{R}^{d \times d}$, $\Sigma^{1/2} \in \mathbb{R}^{d \times m}$, and $W_t \in \mathbb{R}^m$ is a m-dimensional standard Brownian motion. Then, the generator for the process is:
        \begin{align*}
            \mathcal{L}f(x)&:=\lim_{\epsilon\to 0}\E\left[\frac{f(X_{t+\epsilon})-f(X_t)}{\epsilon}|X_t=x\right] \\
            &= \langle H x, \nabla f(x) \rangle + \frac{1}{2} \operatorname{Tr}(\Sigma\nabla^2f(x)),
        \end{align*}
        where $f:\mathbb{R}^d\to \mathbb{R}$ is any twice differentiable function and $\operatorname{Tr}$ is the trace of the matrix. Further, suppose that $H$ is a Hurwitz matrix. Then, the stationary distribution of the O-U process is a zero mean Gaussian distribution with covariance $\Sigma_X$ given by the solution of the following Lyapunov equation:
        \begin{align*}
            H\Sigma_X + \Sigma_X H^T+ \Sigma=0.
        \end{align*}
\end{lemma}

The central idea behind Stein's method is identifying the characterizing property for Gaussian distribution. In particular, given the O-U process \eqref{eq:ou_process}, its generator $\mathcal{L}$ satisfies $\E[\mathcal{L}f(X)]=0$ if and only if $X\sim \mathcal{N}(0, \Sigma_X)$ for every twice differentiable function $f$. Thus, one expects that for any random variable $Y$ which is close to $X$ in distribution, $|\E[\mathcal{L}f(Y)]|$ should be small. To translate this notion of ``closeness'' into an error metric, we use the Kantorovich-Rubinstein duality for Wasserstein-1 distance which reformulates $d_{\mathcal{W}}(Y, X)$ as
\begin{align}\label{eq:wass_func}
    d_{\mathcal{W}}(X, Y)=\sup_{h\in Lip_1}\E[h(X)-h(Y)].
\end{align}

where $Lip_1:=\{h:|h(x)-h(y)|\leq \|x-y\|\}$. Now, for any fixed test function $h$, let $f$ be the solution to the Gaussian-Stein equation
\begin{equation}\label{eq:stein}
     h(y)-\mathbb{E}[h(X)]=
  y^TH^T\nabla f(y)+\frac{1}{2}\operatorname{Tr}(\Sigma \nabla^2f(y)).
\end{equation} 
Note that for a given $\Sigma_X$, we can construct the O-U process in Eq. \eqref{eq:ou_process} in various ways by appropriately adjusting $H$ and $\Sigma$. Notably, for the purpose of our analysis, we can work with the most straightforward choice of parameters by setting $H=-I$ and $\Sigma=2\Sigma_X$ which ensures that the stationary distribution of $X_t$ is $\mathcal{N}(0, \Sigma_X)$ due to the Lyapunov equation in Lemma \ref{lem:ou-process}.
The following lemma adapted from \cite{Gallouet2018} provides the necessary properties of the solution $f(\cdot)$ which will be essential for our analysis.

\begin{lemma}\label{lem:derivative_bound}
\label{f bound}    
Let $f$ be the solution to the Gaussian-Stein equation \eqref{eq:stein} with $H=-I$. Define 
\begin{align*}
    \sigma_{max}=\max_{\beta\in [0,1]}\|\Sigma_X^{\frac{1}{2}}\|\|\Sigma_X^{-\frac{1}{2}}\|^{2+\beta}.
\end{align*}
Then, for any $\beta\in(0,1)$ the following Hessian H\"older bound holds:

\begin{align}\label{c1}
    \bigl\|\nabla^{2}f(x)-\nabla^{2}f(y)\bigr\|
    \;&\le\; \underbrace{\left(\Tilde{C}_{1}(d)+\frac{2}{1-\beta}\right)}_{C_1(d, \beta)}L\sigma_{max}  \|x-y\|^\beta
\end{align}

with
\begin{align*}
    \Tilde{C}_{1}(d)
  =2^{3/2}\!\left(\frac{1+2d\,\Gamma\!\bigl((1+d)/2\bigr)}{d\Gamma\!\bigl(d/2\bigr)}\right).
\end{align*}
\end{lemma}
Note that the above relations were originally proved specifically for the case of $\Sigma=I$. However, they can be easily adapted to the general covariance case via change of variables. For completeness, we present this modification in Appendix \ref{appendix:derivative_bound}. 

Now we can bound the distance of the distribution of $Y$ with respect to Gaussian using the derivatives bounds for the solution to the Gaussian-Stein equation as follows
\begin{align*}
    d_{\mathcal{W}}(Y, X)&=\sup_{h\in Lip_1}\E[h(Y)-h(X)]\\
    &\leq |\E[\mathcal{L}f(Y)]|\\
    &=|\E[\operatorname{Tr}(\Sigma_X \nabla^2f(Y))-Y^T\nabla f(Y)]|,
\end{align*}
where $f$ belongs to the class of twice differentiable functions which satisfy Lemma \ref{lem:derivative_bound}.

\subsection{Wasserstein-1 convergence for DOUG}
Consider a scalar DOUG process \eqref{eq:discrete_ou} where $H<0$ and $\chi_k$ is a sequence of mean-zero additive i.i.d. noise with bounded third moment:
\begin{align}\label{eq:weighed_rec1}
    z_{k+1} = (1 + \alpha_kH_k) z_k + \sqrt{\alpha_k} \chi_k.
\end{align}
where $H_k=H+\xi/(2\alpha(k+K)^{1-\xi})$ and $z_0=0$. By unrolling the recursion from time instant $0$ to $k$, we get the following closed-form expression:
\begin{align}\label{eq:weighted_sum}
    z_{k+1} = \sum_{i=0}^k\chi_i\sqrt{\alpha_i}\prod_{l=i+1}^k(1+\alpha_lH_l) = \sum_{i=0}^k\Theta_i\chi_k. 
\end{align}
where $\Theta_i=\sqrt{\alpha_i}\prod_{l=i+1}^k(1+\alpha_lH_l)$ for $i<k$ and $\Theta_k=\sqrt{\alpha_k}$. Note that compared to the usual CLT setting $z_k=(\sum_{i=0}^{k-1}b_i)/\sqrt{k}$, the only difference is the weights with which the noise terms are summed. Specifically, in the case of CLT, $\Theta_i=1/\sqrt{k}$ for all $i$ and therefore, $\sum_{i=0}^k\Theta_i^2=1$ for all $k$. Thus, $\text{var}(z_k)=\Sigma_b$ for all $k$. However, these conditions do not hold for Eq. \eqref{eq:weighted_sum} and instead the sum of squares only converge asymptotically, $\sum_{i=0}^k\Theta_i^2\stackrel{k\uparrow \infty}{\to} H^{-1}/2$. To incorporate this discrepancy, we use Stein's method to bound the Wasserstein distance between distribution of $z_k$ and a time-varying Gaussian whose variance evolves identical to the variance of $z_k$, i.e., $d_{\mathcal{W}}(z_k, (\Sigma_k^{(\alpha, \xi)})^{1/2}Z)$, where $\Sigma_k^{(\alpha, \xi)}=\text{var}(z_k)$.
We adapt the proof for Stein's method to handle the weighted sum and obtain these rates of convergence.

For the general vector-valued case the weights $\Theta_i$ transform into matrices and one needs a handle on matrix summation. Consequently, we leverage the Hurwitz condition on $H$ and the Lyapunov equation \eqref{eq:lyap_H} to analyze this sum with respect to the $V_H$-weighted norm. In particular, under this norm, we establish $\|I+\alpha_lH_k\|_{V_H}\leq (1-\iota_{V_H}\alpha_l)$ for some $\iota_{V_H}>0$. Therefore, 
\begin{align*}
    \|\Theta_i\|_{V_H}\leq \sqrt{\alpha_i}\prod_{l=i+1}^k(1-\iota_{V_H}\alpha_l).
\end{align*}
Equipped with this property, we construct a scalar recursion that strictly upper bounds the matrix summation under this weighted norm. Consequently, controlling this scalar recursion directly bounds the original matrix terms. For more details, please refer to Lemma \ref{lemma:vec-im-results} in Appendix \ref{sec:proof_weighted_stein}.

\subsection{Wasserstein-1 convergence for SA}
Motivated by the asymptotic SDE-like behavior of $\{y_k\}_{k\geq 0}$, we consider two DOUG processes with the drift matrix $H_k$ replaced by the local linear behavior of the operator $F(\cdot)$ around $x^*$:
\begin{align*}
    \hat{z}_{k+1} &= (I + \alpha_k J_k^{(\alpha, \xi)}) \hat{z}_k + \sqrt{\alpha_k} M_k,\\
    z_{k+1} &= (I + \alpha_k J_k^{(\alpha, \xi)})z_k + \sqrt{\alpha_k} b_k.
\end{align*}
Specifically, the sequences $\{z_k\}_{k\geq 0}$ and $\{\hat{z}_k\}_{k\geq 0}$ are driven by purely additive component and the full noise sequence corrupting SA, respectively. Therefore, to establish the rate of convergence of $d_{\mathcal{W}}(y_k, (\Sigma_k^{(\alpha, \xi)})^{1/2}Z)$, we use triangle-inequality to decompose the distance as sum of three terms:
\begin{align*}
    d_{\mathcal{W}}(y_k, (\Sigma_k^{(\alpha, \xi)})^{1/2}Z)&\leq d_{\mathcal{W}}(y_k, \hat{z}_k)+d_{\mathcal{W}}(\hat{z}_k, z_k)+d_{\mathcal{W}}(z_k, (\Sigma_k^{(\alpha, \xi)})^{1/2}Z).
\end{align*}

Such a decomposition approach that separates pure noise from SA
is simple yet powerful and was only recently introduced in \cite{bravo2024stochastic} to study SA. Its potential was more fully realized in \cite{chandak20251} and other follow-up work by the same author. We adopt a variant of this decomposition approach in our paper. For clarity of presentation, we will only consider the scalar case in the following to avoid messy matrix notations. The generalization of these arguments to the multi-dimensional case is straightforward.

\subsubsection{Handling multiplicative noise}
The crucial insight in this setting lies in the observation that the multiplicative component of the noise vanishes in mean square as $k\to \infty$. Thus, the asymptotic distribution is solely characterized by the additive i.i.d. term. To illustrate this, let us again consider the following scalar recursion:
\begin{align}\label{eq:weighed_rec1_mart}
    \hat{z}_{k+1} = (1 + \alpha_kJ_k^{(\alpha, \xi)}) \hat{z}_k + \sqrt{2\alpha_k} M_k.
\end{align}
where $M_k$ satisfies Assumption \ref{assump:noise}.
Our primary goal here is to show that the distribution of $\hat{z}_k$ converges to the distribution of $z_k$ defined in Eq. \eqref{eq:weighed_rec1} as $k\to \infty$. More specifically, we aim to bound the convergence rate of $d_{\mathcal{W}}(z_k, \hat{z}_k)\to 0$. Recall that $d_{\mathcal{W}}(X, Y)=\inf_{\gamma\in \Gamma(\nu_X, \nu_Y)}\E_{(X, Y)\sim \gamma}[\|X-Y\|]$. Thus, for any arbitrary coupling $\gamma$, $d_{\mathcal{W}}(X, Y)\leq \E_{(X, Y)\sim \gamma}[\|X-Y\|]$. In particular, we choose the coupling such that the underlying additive i.i.d. component for $M_k$ in Eq. \eqref{eq:weighed_rec1_mart} and $b(w_k)$ in Eq. \eqref{eq:weighed_rec1} are identical. This leads to the following error dynamics:
\begin{align*}
    \hat{z}_{k+1}-z_{k+1} = (1 + \alpha_kJ_k^{(\alpha, \xi)}) (\hat{z}_k-z_k) + \sqrt{2\alpha_k} A_k.
\end{align*}
Now if we simply apply the triangle-inequality on the absolute value of this iteration, the resulting recursion yields a sub-optimal rate. Instead, we obtain a tight rate by squaring both sides and using the martingale difference property of $A_k$ for the cross term. More precisely,
\begin{align*}
    (\hat{z}_{k+1}-z_{k+1})^2 &= (1 + \alpha_kJ_k^{(\alpha, \xi)})^2 (\hat{z}_k-z_k)^2 + 2\alpha_kA_k^2+2\sqrt{2\alpha_k}(1+\alpha_kJ_k^{(\alpha, \xi)})(\hat{z}_k-z_k)A_k\\
    \implies \E[(\hat{z}_{k+1}-z_{k+1})^2|\mathcal{F}_{k-1}]&=(1 + \alpha_kJ_k^{(\alpha, \xi)})^2 (\hat{z}_k-z_k)^2 + 2\alpha_k \E[A_k^2|\mathcal{F}_{k-1}]\\
    \implies \E[(\hat{z}_{k+1}-z_{k+1})^2]&=(1 + \alpha_kJ_k^{(\alpha, \xi)})^2\E[(\hat{z}_k-z_k)^2]+2\alpha_k \E[A_k^2].
\end{align*}
The resulting recursion can be easily solved to obtain the convergence rate for $\E[(\hat{z}_k-z_k)^2]$. Finally, we can apply Jensen's inequality to get a bound on $\E[|\hat{z}_k-z_k|]$, which provides the desired upper bound on $d_{\mathcal{W}}(z_k, \hat{z}_k)$.

\subsubsection{Putting everything together}
For bounding the first term $d_{\mathcal{W}}(y_k, \hat{z}_k)$, we construct a similar coupling as in the previous section, i.e., the underlying noise $M_k$ in Eq. \eqref{eq:weighed_rec1_mart} and the SA are identical. To elucidate more about the analysis of $d_{\mathcal{W}}(y_k, \hat{z}_k)$, let us consider the following scalar SA:
\begin{align*}
    x_{k+1}=x_k+\alpha_k(F(x_k)+M_k).
\end{align*}
where $F(x^*)=0$. Thus, the rescale iterate $y_k$ is given by
\begin{align*}
    y_{k+1}&=y_k+\sqrt{\alpha_k}(F(x_k)+M_k)+(x_{k+1}-x^*)\left(\frac{1}{\sqrt{\alpha_{k+1}}}-\frac{1}{\sqrt{\alpha_k}}\right)\\
    &=(1+\alpha_kJ^{(\alpha, \xi)}_k)y_k+\sqrt{\alpha_k}M_k+\sqrt{\alpha_k}R(x_k)+\mathrm{T}_k.
\end{align*}
where $J^{(\alpha, \xi)}_k$ is defined in Eq. \eqref{eq:J_k}. Hence, the error dynamics $y_k-\hat{z}_k$ is given as
\begin{align*}
    y_{k+1}-\hat{z}_{k+1} &= (1 + \alpha_k J^{(\alpha, \xi)}_k) (y_k-\hat{z}_k) + \sqrt{\alpha_k} R(x_k)+\mathrm{T}_k.
\end{align*}
Note that for the polynomial decaying step-size, we have the property that $\E[|\mathrm{T}_k|]\leq \mathcal{O}(1/k^2)$ (using Lemma \ref{lem:iterate_bounds} and Lemma \ref{lem:step-size_prop}). This leads us to
\begin{align*}
    \E[|y_{k+1}-\hat{z}_{k+1}|] 
    &\leq (1 - \iota_V\alpha_k) \E[|y_k-\hat{z}_k|]+ \alpha_k\mathcal{O}(\alpha_k^{\delta/2}).
\end{align*}
Solving the above recursion yields $\E[|y_k-\hat{z}_k|]\leq \tilde{\mathcal{O}}(1/k^{\delta\xi/2})$. The analysis for constant step-size, $\alpha_k=\alpha$, follows along the same lines with the additional simplification that the last term in the update for $y_k$ is zero.
\begin{remark}
    We remark that an alternative method to establish the normality of $\{y_k\}_{k \geq 0}$ is to directly work with the iterate $x_k$ as follows:
    \begin{align*}
        x_{k+1}-x^*&=(I+\alpha_kJ_F)(x_k-x^*)+\alpha_kb_k+\alpha_k(A_k+R(x_k))\\
        \implies \frac{x_{k+1}-x^*}{\sqrt{\alpha_{k+1}}}&=\sum_{i=0}^k\tilde{\Theta}_ib_i+\frac{1}{\sqrt{\alpha_{k+1}}}\left(\prod_{i=0}^k(I+\alpha_iJ_F)\right)(x_0-x^*)+\sum_{i=0}^k\tilde{\Theta}_i(A_i+R(x_i)),
    \end{align*}
    for appropriately defined weight matrices $\tilde{\Theta}_i$.  One can alternately work with first term $\tilde{z}_k=\sum_{i=0}^k\tilde{\Theta}_ib_i$ instead of DOUG and then adapt Stein's method to establish convergence rate to normality for this term while show that the remaining terms are higher-order residuals. We introduced DOUG process as in Eq. \eqref{eq:discrete_ou} due to its connection to sampling algorithms.
\end{remark}

\section{Conclusion and Future Work}
In this work, we quantify the normal approximation for SA by establishing sharp non-asymptotic bounds for Wasserstein-1 distance between the distribution of the the rescaled SA at time $k$ and the asymptotic Gaussian limit for various choices of step-size. We obtain this result by constructing an auxiliary discrete time O-U process which we christened as DOUG, and studying the distributional rate of convergence for DOUG. We derive the convergence rate for DOUG by adapting the machinery of Stein's method to handle weighted sums of random variables. Thereafter, we use these rates as the stepping stone to obtain the final bound by employing a careful coupling argument and analyzing the error dynamics between the rescaled SA and DOUG. Nevertheless, this work opens many interesting questions for future work:
\begin{enumerate}
    \item \textbf{Handling Markov noise:} One of the most popular applications of SA is in RL and SGD with auto-regressive data, which involves sampling data points from a Markov chain. Recent work \cite{kong2026finitesamplewassersteinerrorbounds} obtained convergence bounds for SA driven by Markovian noise, although their rate is sub-optimal. Nevertheless, in the context of CLT much stronger results exists in the literature. In \cite{srikant2024rates}, the author achieved the first optimal convergence bounds (upto log factors) for CLT with Markovian random variables by using Poisson's equation. Recently, \cite{zhang2026wassersteinpcentrallimittheorem} used a coupling argument and established a tight rate of convergence by shaving off the log factors in \cite{srikant2024rates}. Building upon these approaches, a natural future work is to improve upon \cite{kong2026finitesamplewassersteinerrorbounds} and derive sharp Wasserstein-1 distance bounds for SA under Markovian noise. 
    \item \textbf{Improving the bounds for tail deviations:} 
    We show in Corollary \ref{cor:tail_bounds} that the tail of the error deviates from that of a Gaussian by a  time decaying term whose tail is $O(1/a)$. 
    We believe that this decay is not optimal and finding the optimal decay is an open problem. In the special case of CLT, it was recently shown \cite{austern2022efficient} that the deviation of the error from its limiting Gaussian also has  a sub-Gaussian decay. This result, known as efficient concentration was obtained by first proving Wasserstein-$p$ bounds on the error for all $p \geq 1$. Moreover, the dependence of the Wasserstein-$p$ error bounds on the parameters  $p$ and $k$ should be tight. Recent work \cite{kong2026finitesamplewassersteinerrorbounds} obtained Wasserstein-$p$ bounds for SA when the noise is i.i.d. and additive.
    \item \textbf{Improving the convergence rates under higher moments matching:} As discussed in Section \ref{sec:lower bound}, the asymptotic rate of convergence for central limit theorem improves to $\mathcal{O}(k^{-n/2})$ for distributions whose $n+1$ moments match with a Gaussian \cite{fathi2018higherordersteinkernelsgaussian, bobkov2018berry}. Therefore, a promising direction for future research is to extend this improvement under moment matching to SA updates for various choices of step-sizes. Preliminary empirical evidence suggests a strong possibility and indicates that the convergence rates in Theorem \ref{thm_main:main_thm} can be sharpened to $\mathcal{O}(\alpha_k)$ in the special case of symmetric distributions (three moments match). 
    \item \textbf{Extending convergence results for DOUG to non-linear update:} Recall that we defined the DOUG process \eqref{eq:discrete_ou} with a linear update. A non-linear analog of DOUG can be obtained by substituting the linear term $H_kz_k$ with the gradient of a function $f(z_k)$. Then, an interesting future work is generalizing our proof methodology to allow for non-linear updates and study the distributional convergence of DOUG to corresponding Gibbs distribution. Note that in the special case of i.i.d. Gaussian noise, DOUG reduces to the popular ULA algorithm whose convergence was studied in \cite{pages2023unadjusted, durmus2017nonasymptotic, durmus2019high}.
    \item \textbf{Phase Transition at $\xi=1$:} We observe that Corollary \ref{cor:asymp_gaussian_rate} exhibits a phase transition in the rate of convergence to asymptotic distribution, i.e., the rate abruptly improves to $\mathcal{O}(1/\sqrt{k})$ at $\xi=1$ (Figure \ref{fig:slope}). However, note that achieving this optimal convergence requires choosing the scaling parameter $\alpha$ large enough. In contrast, if $\alpha$ is not tuned properly, then MSE bounds in \cite{chen2021finitesampleanalysisstochasticapproximation} shows a degradation in the rate of convergence. This sensitivity suggests that one should investigate the phase transition at $\xi=1$ by studying the convergence rate as a function of $\alpha$.
\end{enumerate}

\bibliographystyle{alpha}
\bibliography{references}

\newpage
\appendix

\section{Proof of the Main Results}

\subsection{Proof of Theorem \ref{thm_main:weighted_stein}}\label{sec:proof_weighted_stein}
To avoid repetition, we prove Theorem \ref{thm_main:weighted_stein} by directly substituting $H=J_F$ from the beginning. However, all the steps hold for any general Hurwitz matrix independent of the SA update \eqref{eq:SA_rec}. We first present the proof of Theorem \ref{thm_main:weighted_stein} since it forms the stepping stone for the proof of Theorem \ref{thm_main:main_thm}. Define the following weight matrices:
\begin{align}\label{eq:theta_def}
    \Theta_i = \sqrt{\alpha_i} \prod_{l = i+1}^k (I+\alpha_l J_l^{(\alpha, \xi)}) \text{ and } \Theta_k=\sqrt{\alpha_k}
\end{align}
Recall, we defined $\{\Sigma_k^{(\alpha, \xi)}\}_{k\geq 0}$ as a sequence of covariance matrices for $\{y_k\}_{k\geq 0}$. Similarly, we define $\{\tilde{\Sigma}_k^{(\alpha, \xi)}\}_{k\geq 0}$ as sequence of covariance matrices which satisfy the following:
\begin{align}\label{eq:sigma_k_eq_pf}
    \tilde{\Sigma}_{k+1}^{(\alpha, \xi)}&=(I + \alpha_k J^{(\alpha, \xi)}_k)\tilde{\Sigma}_k^{(\alpha, \xi)}(I + \alpha_k J^{(\alpha, \xi)}_k)^T+\alpha_k\Sigma_b,
\end{align}
where $\tilde{\Sigma}_0^{(\alpha, \xi)}=0$.
Also, define 
\begin{align*}
    \kappa^{(\alpha, \xi)}=\max_{k\geq 1}\max_{\beta\in [0,1]}\|(\tilde{\Sigma}_k^{(\alpha, \xi)})^{\frac{1}{2}}\|\|(\tilde{\Sigma}_k^{(\alpha, \xi)})^{-\frac{1}{2}}\|^{2+\beta};~~\tilde{\sigma}_{min}^{(\alpha, \xi)}=\min_{k\geq 1}\|\tilde{\Sigma}_k^{(\alpha, \xi)}\|.
\end{align*}
where the maximum over $k$ is well-defined due to Lemma \ref{lem:bound_sigma_k}, Lemma \ref{lem:inv_bound}, and Lemma \ref{lem:matrix_theta_bound}. In the following proofs, we will assume that for $\xi=0$, $\alpha$ is chosen such that $\alpha\leq \min(1, 2\iota_V/\|J\|^2_V)$. For $\xi\in (0, 1]$, we will assume $K_0$ is large enough such that the step-size conditions in Lemma \ref{lem:theta_bound}  are satisfied for any $K\geq K_0$. We will also drop the subscript from $J_F$ wherever it is clear from context. Before beginning the proof, we present the following lemma which will be instrumental in bounding certain terms appearing in the expression for Stein's operator.

\begin{lemma}\label{lemma:vec-im-results}
Define $\tilde{\varphi}_1(\alpha, \xi)=(\lambda^{max}_V/\lambda^{min}_V)^{3/2}\kappa^{(\alpha, \xi)}\|\Sigma_b\|B_{max}^{(1)}$. Then, the following relation holds for all $k\geq 0$:
        \begin{align*}
            &\Bigg| \E \Bigg[ \mathrm{Tr}\left((  \sum_{i=0}^{k-1} \nabla^2 f(\zeta_i) \Theta_i \Sigma_b \Theta_i^T \right) - \mathrm{Tr}(\nabla^2f(z_k)\tilde{\Sigma}_k^{(\alpha, \xi)})\Bigg]\Bigg|\\
            &\leq
            \begin{cases}
                dC_1(d, \beta)\tilde{\varphi}_1(\alpha, 0)\alpha^{\beta/2}/\iota_V\text{, when }\xi=0\\
                2dC_1(d, \beta)\tilde{\varphi}_1(\alpha, \xi)\alpha_k^{\beta/2}/\iota_V\text{, when }\xi\in (0,1).\\
                2dC_1(d, \beta)\tilde{\varphi}_1(\alpha, 1)\alpha\alpha_k^{\beta/2}/(2\iota_V\alpha-1)\text{when $\xi=1$ and $\iota_V\alpha>1/2$.}
            \end{cases}
        \end{align*}
\end{lemma}

Define the following constants:
\begin{align}\label{eq:varphi}
\begin{split}
    \tilde{\varphi}_2(\alpha, \xi)&=\left(\frac{\lambda^{max}_V}{\lambda^{min}_V}\right)^{3/2}B_{max}^{(3)}\kappa^{(\alpha, \xi)};\\
    \varphi_1(\alpha, \xi)&=\tilde{C}_1(d)e\left(d\tilde{\varphi}_1(\alpha, \xi)+\tilde{\varphi}_2(\alpha, \xi)\right)\kappa^{(\alpha, \xi)}.
\end{split}
\end{align}

\begin{theorem}\label{thm:weighted_stein}
   Consider the sequence
   \begin{align*}
        z_{k+1} = (I + \alpha_k J^{(\alpha, \xi)}_k) z_k + \sqrt{\alpha_k} b(w_k).
   \end{align*}
   where $z_0=0$. Then, the iteration $\{z_k\}_{k\geq 0}$ satisfies the following
   \begin{enumerate}
       \item When $\xi=0$ and $\alpha\leq \min(1, 2\iota_V/\|J\|^2_V)${\normalfont :}
       \begin{align*}
            d_{\mathcal{W}}(z_{k}, (\Sigma_k^{(\alpha, 0)})^{1/2}Z)  \leq   \frac{\varphi_1(\alpha, 0)\sqrt{\alpha}}{\iota_V}\log\left(\frac{1}{\alpha}\right)+\sqrt{\frac{\lambda_V^{max}}{\lambda_V^{min}}}\|\E[y_0y_0^T]\|_Ve^{-\iota_V\alpha k}.
        \end{align*}
        \item When $\xi\in (0, 1)$ and $K\geq K_0${\normalfont :}
        \begin{align*}
            d_{\mathcal{W}}(z_{k}, (\Sigma_k^{(\alpha, \xi)})^{1/2}Z)  &\leq \frac{2\varphi_1(\alpha, \xi)}{\iota_V}\frac{\sqrt{\alpha}}{(k+K)^{\xi/2}}\log\left(\frac{(k+K)^{\xi}}{\alpha}\right)+\sqrt{\frac{\lambda_V^{max}}{\lambda_V^{min}}}\|\E[y_0y_0^T]\|_Ve^{-\frac{\iota_V\alpha}{1-\xi}\left((k+K)^{1-\xi}-K^{1-\xi}\right)}.
        \end{align*}
        \item When $\xi=1$, $\iota_V\alpha>1/2$, and $K\geq K_0${\normalfont :}
        \begin{align*}
            d_{\mathcal{W}}(z_{k}, (\Sigma_k^{(\alpha, 1)})^{1/2}Z)  &\leq \frac{2\varphi_1(\alpha, 1)\alpha^{3/2}}{2\iota_V\alpha-1}\frac{1}{\sqrt{k+K}}\log\left(\frac{k+K}{\alpha}\right)+\sqrt{\frac{\lambda_V^{max}}{\lambda_V^{min}}}\|\E[y_0y_0^T]\|_V\left(\frac{K}{k+K}\right)^{\iota_V\alpha}.
        \end{align*}
   \end{enumerate}
\end{theorem}
\begin{proof}
Note that the sequence $\{z_k\}_{k\geq 0}$ satisfies
\begin{align*}
    z_{k+1} = \sum_{i=0}^k  \Theta_i b(w_i),
\end{align*}
where $\Theta_i$ are defined in Eq. \eqref{eq:theta_def}. Note that the final bound in Theorem \ref{thm_main:weighted_stein} is with respect to $(\Sigma_k^{(\alpha, \xi)})^{1/2}Z$. Thus, first we will use triangle inequality to get
\begin{align*}
    d_{\mathcal{W}}(z_k, (\Sigma_k^{(\alpha, \xi)})^{1/2}Z)\leq d_{\mathcal{W}}(z_k, (\tilde{\Sigma}_k^{(\alpha, \xi)})^{1/2}Z)+d_{\mathcal{W}}((\Sigma_k^{(\alpha, \xi)})^{1/2}Z, (\tilde{\Sigma}_k^{(\alpha, \xi)})^{1/2}Z)
\end{align*}
We will first analyze the second term. To this end, we use standard inequalities for Gaussian distribution to show that it rapidly converges to zero. 

Unfortunately, there does not exist a closed-form solution to the Wasserstien-1 distance between two Gaussians. Thus, we use the relation $d_{\mathcal{W}}((\Sigma_k^{(\alpha, \xi)})^{1/2}Z, (\tilde{\Sigma}_k^{(\alpha, \xi)})^{1/2}Z)\leq d_{\mathcal{W}_2}((\Sigma_k^{(\alpha, \xi)})^{1/2}Z, (\tilde{\Sigma}_k^{(\alpha, \xi)})^{1/2}Z)$ and work with Wasserstein-2 distance to obtain the upper bound. Using Lemma \ref{lem:wass-2}, we get
\begin{align*}
    d_{\mathcal{W}_2}((\Sigma_k^{(\alpha, \xi)})^{1/2}Z, (\tilde{\Sigma}_k^{(\alpha, \xi)})^{1/2}Z)&\leq \frac{\sqrt{d}}{\sqrt{\lambda_{min}(\Sigma_k^{(\alpha, \xi)})}+\sqrt{\lambda_{min}(\tilde{\Sigma}_k^{(\alpha, \xi)})}}\|\Sigma_k^{(\alpha, \xi)}-\tilde{\Sigma}_k^{(\alpha, \xi)}\|\\
    &\leq \frac{\sqrt{d}}{\sqrt{\lambda_{min}(\tilde{\Sigma}_k^{(\alpha, \xi)})}}\|\Sigma_k^{(\alpha, \xi)}-\tilde{\Sigma}_k^{(\alpha, \xi)}\|\tag{$\Sigma_k^{(\alpha, \xi)}$ is a positive definite matrix}\\
    &\leq \sqrt{\frac{d}{\tilde{\sigma}_{min}^{(\alpha, \xi)}}}\|\Sigma_k^{(\alpha, \xi)}-\tilde{\Sigma}_k^{(\alpha, \xi)}\|.
\end{align*}
Since $\{\Sigma_k^{(\alpha, \xi)}\}_{k\geq 0}$ and $\{\tilde{\Sigma}_k^{(\alpha, \xi)}\}_{k\geq 0}$ only differ in their starting point, we have
\begin{align*}
    \tilde{\Sigma}_{k+1}^{(\alpha, \xi)}-\Sigma_{k+1}^{(\alpha, \xi)}&=(I + \alpha_k J^{(\alpha, \xi)}_k)(\tilde{\Sigma}_k^{(\alpha, \xi)}-\Sigma_k^{(\alpha, \xi)})(I + \alpha_k J^{(\alpha, \xi)}_k)^T.
\end{align*}
Now, we obtain the upper bound for each of the three cases separately,
\begin{enumerate}
    \item $\xi=0:$ Taking matrix norm both sides and applying statement 1 of Lemma \ref{lem:contraction_prop}, we obtain
    \begin{align*}
        \|\tilde{\Sigma}_{k+1}^{(\alpha, 0)}-\Sigma_{k+1}^{(\alpha, 0)}\|_V&\leq (1-\iota_V\alpha)\|\tilde{\Sigma}_k^{(\alpha, 0)}-\Sigma_k^{(\alpha, 0)}\|_V.
    \end{align*}
    Applying Lemma \ref{lem:rec_sol} with obvious choices of constants, we get
    \begin{align*}
        \|\tilde{\Sigma}_k^{(\alpha, 0)}-\Sigma_k^{(\alpha, 0)}\|_V&\leq \|\Sigma_0^{(\alpha, 0)}\|_Ve^{-\iota_V\alpha k}\\
        \implies \|\tilde{\Sigma}_k^{(\alpha, 0)}-\Sigma_k^{(\alpha, 0)}\|&\leq \sqrt{\frac{\lambda_V^{max}}{\lambda_V^{min}}}\|\Sigma_0^{(\alpha, 0)}\|_Ve^{-\iota_V\alpha k}.
    \end{align*}
    \item $\xi\in (0,1):$ Again, applying Lemma \ref{lem:rec_sol} with obvious choices of constants, we get
    \begin{align*}
        \|\tilde{\Sigma}_k^{(\alpha, \xi)}-\Sigma_k^{(\alpha, \xi)}\|_V&\leq \|\Sigma_0^{(\alpha, \xi)}\|_Ve^{-\frac{\iota_V\alpha}{1-\xi}\left((k+K)^{1-\xi}-K^{1-\xi}\right)}\\
        \implies \|\tilde{\Sigma}_k^{(\alpha, \xi)}-\Sigma_k^{(\alpha, \xi)}\|&\leq \sqrt{
        \frac{\lambda_V^{max}}{\lambda_V^{min}}}\|\Sigma_0^{(\alpha, \xi)}\|_Ve^{-\frac{\iota_V\alpha}{1-\xi}\left((k+K)^{1-\xi}-K^{1-\xi}\right)}.
    \end{align*}
    \item $\xi=1:$ Now, we first apply statement 2 of Lemma \ref{lem:contraction_prop}, to get
    \begin{align*}
        \|\tilde{\Sigma}_{k+1}^{(\alpha, 1)}-\Sigma_{k+1}^{(\alpha, 1)}\|_V&\leq (1-\iota_V\alpha)\|\tilde{\Sigma}_k^{(\alpha, 1)}-\Sigma_k^{(\alpha, 1)}\|_V.
    \end{align*}
    Then, applying Lemma \ref{lem:rec_sol}, we obtain
    \begin{align*}
        \|\tilde{\Sigma}_k^{(\alpha, 1)}-\Sigma_k^{(\alpha, 1)}\|_V&\leq \|\Sigma_0^{(\alpha, 1)}\|_V\left(\frac{K}{k+K}\right)^{\iota_V\alpha}\\
        \implies \|\tilde{\Sigma}_k^{(\alpha, 1)}-\Sigma_k^{(\alpha, 1)}\|&\leq \sqrt{\frac{\lambda_V^{max}}{\lambda_V^{min}}}\|\Sigma_0^{(\alpha, 1)}\|_V\left(\frac{K}{k+K}\right)^{\iota_V\alpha}.
    \end{align*}

\end{enumerate}

In the following, our goal is to establish the convergence bound for $d_{\mathcal{W}}(z_k, (\tilde{\Sigma}_k^{(\alpha, \xi)})^{1/2}Z)$. Let us fix $k$ in the following and define an O-U process as:
\begin{align*}
    dX_t = - X_t\,dt +(2\tilde{\Sigma}_k^{(\alpha, \xi)})^{1/2}\,dW_t,
\end{align*}
By Lemma \ref{Steinop}, the asymptotic distribution for $X_t$ is $\mathcal{N}(0, \tilde{\Sigma}_k^{(\alpha, \xi)})$. Therefore, we define the Stein's operator for the above process is $\mathcal{L}_k f(z) : =   -\langle z, \nabla f(z) \rangle + \mathrm{Tr}(\tilde{\Sigma}_k^{(\alpha, \xi)} \nabla^2f(z))$. Then, using Stein's method as discussed in Section \ref{sec:stein_method}, we get
\begin{align*}
    d_{\mathcal{W}}(z_k, (\tilde{\Sigma}_k^{(\alpha, \xi)})^{1/2}Z) \leq  \sup_{f \in \mathcal{F}_{GS}} \left| \E[\mathcal{L}_k f(z_k)]\right|, 
\end{align*}
where the function class $\mathcal{F}_{GS}$ satisfies the regularity conditions given in Lemma \ref{f bound}.
Therefore, it suffices to bound
\begin{align*}
    \E[\mathcal{L}_k f(z_k)]=\E\left[ -\langle z_k, \nabla f(z_k) \rangle + \mathrm{Tr}(\tilde{\Sigma}_k^{(\alpha, \xi)}\nabla^2f(z_k)) \right]
\end{align*}
for any $f \in \mathcal{F}_{GS}$. Define $\zeta_j := z_k - \Theta_j b(w_j) = \sum_{i \neq j } \Theta_i b(w_i)$. Then we have $\E[b(w_j)^T \nabla f(\zeta_j)] = 0$ since $w_j$ and $\zeta_j$ are independent and $\E[b(w_j)] = 0$. Thus, we get
\begin{align*}
    \E\left[ \langle z_k, \nabla f(z_k) \rangle\right] & = \E \left[ \sum_{i=0}^{k-1} \langle \Theta_i b(w_i), \nabla f(z_k) \rangle \right] \\
    &= \E \left[ \sum_{i=0}^{k-1} \langle \Theta_i b(w_i), \nabla f(z_k) -   \nabla f(\zeta_i)\rangle \right].
\end{align*}
Next, let $U\sim \mathrm{Uniform}[0,1]$ be an independent random variable. Applying the fundamental theorem of calculus gives
\begin{align*}
    \E\left[ \langle z_k, \nabla f(z_k) \rangle\right]
    &= \E \left[ \sum_{i=0}^{k-1} \langle \Theta_i b(w_i), \nabla^2 f(Uz_k+(1-U)\zeta_i) \Theta_i b(w_i)\rangle \right]\\
    &=\E \left[ \sum_{i=0}^{k-1} \langle \Theta_i b(w_i), \nabla^2 f(\zeta_i) \Theta_i b(w_i)\rangle \right]\\
    &~~+\E \left[ \sum_{i=0}^{k-1} \langle \Theta_i b(w_i), \left(\nabla^2 f(Uz_k+(1-U)\zeta_i)-\nabla^2 f(\zeta_i)\right) \Theta_i b(w_i)\rangle \right].
\end{align*}
Plugging in the expression for $\E[\mathcal{L}_kf(z_k)]$, we can derive
\begin{align*}
     |\E[\mathcal{L}_kf(z_k)]|&\leq \underbrace{\left|
 \E \left[ \sum_{i=0}^{k-1} \langle \Theta_i b(w_i),  (\nabla^2 f(Uz_k+(1-U)\zeta_i)-\nabla^2 f(\zeta_i)) \Theta_i b(w_i)\rangle  \right] \right|}_{\rm (i)} \\
&\quad + \underbrace{\left| \E \left[ \sum_{i=0}^{k-1} \langle \Theta_i b(w_i), \nabla^2 f(\zeta_i) \Theta_i b(w_i)\rangle  -\mathrm{Tr}(\tilde{\Sigma}_k^{(\alpha, \xi)}\nabla^2f(z_k))\right] \right|}_{\rm (ii)}.
\end{align*}
Now our goal is to establish the upper bounds for (i) and (ii) respectively. We first bound the second term. By the symmetry of $\nabla^2 f(z_k)$, we can rewrite (ii) as
\begin{align*}
    \mathrm{(ii)} &= \Bigg| \E \Bigg[ \sum_{i=0}^{k-1}  b(w_i)^T \Theta_i^T \nabla^2 f(\zeta_i) \Theta_i b(w_i) - \mathrm{Tr}(\tilde{\Sigma}_k^{(\alpha, \xi)}\nabla^2f(z_k))\Bigg] \Bigg|\\
    & = \Bigg| \E \Bigg[ \sum_{i=0}^{k-1} \mathrm{Tr}\left(( \nabla^2 f(\zeta_i) \Theta_i b(w_i) b(w_i)^T \Theta_i^T \right) - \mathrm{Tr}(\tilde{\Sigma}_k^{(\alpha, \xi)}\nabla^2f(z_k))\Bigg] \Bigg|.
\end{align*}
Note that $\zeta_i$ and $w_i$ are independent. Thus, we get
\begin{align*}
     \mathrm{(ii)}& = \left| \E \left[ \mathrm{Tr}\left( \sum_{i=0}^{k-1} \nabla^2 f(\zeta_i) \Theta_i \Sigma_b \Theta_i^T \right) - \mathrm{Tr}(\nabla^2f(z_k)\tilde{\Sigma}_k^{(\alpha, \xi)}) \right] \right|.
\end{align*}
By Lemma \ref{lemma:vec-im-results} and triangle-inequality, we obtain
\begin{align*}
    \mathrm{(ii)}&\leq dC_1(d, \beta)\tilde{\varphi}_1(\alpha, 0)\frac{\alpha^{\beta/2}}{\iota_V},~~\text{when } \xi=0,\\
    \mathrm{(ii)}&\leq 2dC_1(d, \beta)\tilde{\varphi}_1(\alpha, \xi)\frac{\alpha_k^{\beta/2}}{\iota_V},~~\text{when } \xi\in (0,1),\\
    \mathrm{(ii)}&\leq 2dC_1(d, \beta)\tilde{\varphi}_1(\alpha, 1)\frac{\alpha\alpha_k^{\beta/2}}{2\iota_V\alpha-1},~~\text{when $\xi=1$ and $\iota_V\alpha>1$.} 
\end{align*}

Note that Lemma \ref{lemma:vec-im-results} is specifically written for the first term in (i). However, it is easy to see that the same bound holds for the second term also. 

Now, it remains to bound (i). Using  Cauchy-Schwartz inequality and Lemma \ref{lem:derivative_bound}, we get
\begin{align*}
    \mathrm{(i)}&\leq  \E\left[\sum_{i=0}^{k-1} \| \Theta_i b(w_i)\|^{2}\|\nabla^2 f(Uz_k+(1-U)\zeta_i)-\nabla^2 f(\zeta_i)\|\right]\\
    &\leq C_1(d, \beta)\kappa^{(\alpha, \xi)} \sum_{i=0}^{k-1} \E[\| \Theta_i b(w_i)\|^{2}\|z_k-\zeta_i\|^{\beta}]  \\
    &\leq C_1(d, \beta)\kappa^{(\alpha, \xi)} \sum_{i=0}^{k-1} \E[\| \Theta_i b(w_i)\|^{2+\beta}] \tag{$z_k-\zeta_i=\Theta_ib(w_i)$}\\
    &\leq C_1(d, \beta)\kappa^{(\alpha, \xi)}B_{2+\beta} \sum_{i=0}^{k-1} \| \Theta_i\|^{2+\beta}\tag{Assumption \ref{assump:noise}}.
\end{align*}
Then, using Lemma \ref{lem:theta_bound} , we get
\begin{align*}
    \mathrm{(i)}\leq 
    \begin{cases}
        C_1(d, \beta)\tilde{\varphi}_2(\alpha, 0)\alpha^{\beta/2}/\iota_V\text{, when } \xi=0\\  
        2C_1(d, \beta)\tilde{\varphi}_2(\alpha, \xi)\alpha_k^{\beta/2}/\iota_V\text{, when } \xi\in (0,1)\\  
        2C_1(d, \beta)\tilde{\varphi}_2(\alpha, 1)\alpha\alpha_k^{\beta/2}/(2\iota_V\alpha-1)\text{, when $\xi=1$ and $\iota_V\alpha>1/2$.}   
    \end{cases}
\end{align*}
Setting $\beta$ as $1 + 2/\log(\alpha_k)$ leads to
    \begin{align*}
        \alpha_k^{\beta/2}=\sqrt{\alpha_k}\alpha_k^{1/\log(\alpha_k)}=e\sqrt{\alpha_k}.
    \end{align*}
Furthermore, $C_1(d, \beta)=(\tilde{C}_1(d)-\log(\alpha_k))$. 
The claim follows by combining all the bounds.
\end{proof}

\subsubsection{Proof of Lemma \ref{lemma:vec-im-results}}
\begin{proof}
Note that using Eq. \eqref{eq:sigma_k_eq_pf}, we have the following closed form expression for the covariance matrix
\begin{align*}
    \tilde{\Sigma}_k^{(\alpha, \xi)} = \sum_{i=0}^{k-1} \Theta_i \Sigma_b \Theta_i^T.
\end{align*}
Thus, we have the following identity
\begin{align*}
     \E \left[\mathrm{Tr}\left(\sum_{i=0}^{k-1} \nabla^2 f(z_k) \Theta_i \Sigma_b \Theta_i^T \right)\right]=  \E\left[\mathrm{Tr}(\nabla^2f(z_k)\tilde{\Sigma}_k^{(\alpha, \xi)})  \right].
\end{align*}

Using the above relation and the fact$|\mathrm{Tr}(Q)|\leq d\|Q\|$ for any matrix $Q$, we get
\begin{align*}
    &\Bigg| \E \Bigg[ \mathrm{Tr}\left(\sum_{i=0}^{k-1} \nabla^2 f(\zeta_i) \Theta_i \Sigma_b \Theta_i^T \right) - \mathrm{Tr}\left(\sum_{i=0}^{k-1} \nabla^2 f(z_k) \Theta_i \Sigma_b \Theta_i^T \right)\Bigg]\Bigg|\\
    &\quad\leq d\E \left[ \left(\sum_{i=0}^{k-1}\|\nabla^2 f(\zeta_i)-\nabla^2 f(z_k)\| \|\Theta_i \Sigma_b \Theta_i^T\| \right)\right]\\
    &\quad\leq dC_1(d, \beta)\kappa^{(\alpha, \xi)}\E \left[ \left(  \sum_{i=0}^{k-1} \|\zeta_i-z_k\|^{\beta} \|\Theta_i \Sigma_b \Theta_i^T\| \right)\right]\tag{Lemma \ref{lem:derivative_bound}}\\
    &\quad\leq dC_1(d, \beta)\kappa^{(\alpha, \xi)}\|\E \left[ \left(  \sum_{i=0}^{k-1} \|\Theta_i\|^\beta\|b(w_i)\|^\beta \|\Theta_i \Sigma_b \Theta_i^T\| \right)\right]\tag{$\zeta_i-z_k=\Theta_ib(w_i)$}\\
    &\quad\leq dC_1(d, \beta)\kappa^{(\alpha, \xi)}\|\Sigma_b\|\E[\|b(w_1)\|^\beta] \left(  \sum_{i=0}^{k-1} \|\Theta_i\|^{2+\beta}\right).
\end{align*}
By applying Lemma \ref{lem:theta_bound}  for the last term, we obtain the bounds.
\end{proof}

\subsection{Lower Bound for DOUG}\label{sec:lower_bound}
Let $\{b_k\}_{k\geq 0}$ be a sequence of zero mean i.i.d. random variables in $\mathbb{R}^d$ such that $\E[\|b_i\|^4]<\infty$. For simplicity, we will assume that $\E[b_ib_i^T]=I$. Furthermore, we suppose that there exists a vector $\mathfrak{u}\in \mathbb{R}^d$ such that $\|\mathfrak{u}\|=1$ and $\E[(\mathfrak{u}^Tb_1)^3]\neq 0$. We set $J_k^{(\alpha, \xi)}$ as
\begin{align*}
    J_k^{(\alpha, \xi)} = -1+\frac{\alpha^{-1}\xi}{2(k+K)^{1-\xi}}.
\end{align*}
Consider the following recursion:
\begin{align*}
    z_{k+1} = \left(1 + \alpha_kJ_k^{(\alpha, \xi)}\right) z_k + \sqrt{\alpha_k} b_k,
\end{align*}
where $z_0=0$, $\alpha_k=\alpha/(k+K)^{\xi}\leq 1$, and $\xi\in [0, 1)$. Let $\Sigma_k$ denote the covariance of $z_k$ and let $Y_k\sim \mathcal{N}(0, \Sigma_k)$. Recall that the Wasserstein-1 distance between two distributions $\nu_X$ and $\nu_Y$ is given by Eq. \eqref{eq:wass_func}
\begin{align*}
    d_{\mathcal{W}}(X, Y)=\sup_{h\in Lip_1}\E[h(X)-h(Y)].
\end{align*}
We will choose a test function $h(\cdot)$ with some desired regularity properties. We will list them at the key steps of the proof for lower bound.

\textbf{Lipschitz Continuous:} Clearly, the test function $h(\cdot)$ should be from the class of 1-Lipschitz functions, i.e., 
\begin{align}\label{eq:lipschitz}
    |h(x)-h(y)|\leq |x-y|.
\end{align}
This gives us
\begin{align*}
    d_{\mathcal{W}}(z_k, Y_k)\geq |\E[h(z_k)-h(Y_k)]|.
\end{align*}

Recall that we can re-write $z_k$ as
\begin{align*}
    z_k &= \sum_{i=0}^{k-1}\sqrt{\alpha_i}b_i\prod_{l=i+1}^{k-1}(1 + \alpha_lJ_l^{(\alpha, \xi)})\\
    &= \sum_{i=0}^{k-1}\Theta_ib_i,
\end{align*}
where $\Theta_i=\sqrt{\alpha_i}\prod_{l=i+1}^{k-1}(1 + J_l^{(\alpha, \xi)}\alpha_l)$ and $\Theta_k=\sqrt{\alpha_k}$. Similarly, we can write $Y_k$ as
\begin{align*}
    Y_k = \sum_{i=0}^{k-1}\Theta_iZ_i,
\end{align*}
where $Z_i\sim \mathcal{N}(0, I)$ independent of all the other randomness. Next, inspired from Lindeberg's decomposition, we will re-write $\E[h(\mathfrak{u}^Tz_k)]$ as a telescoping sum. Define $U_i$ as follows:
\begin{align*}
    U_i = \sum_{j=0}^{i-1} \Theta_j b_j + \sum_{j=i+1}^{k-1} \Theta_jZ_j.
\end{align*}
Then, we have the following decomposition
\begin{align*}
    \E[h(\mathfrak{u}^Tz_k)-h(\mathfrak{u}^TY_k)] = \sum_{i=0}^{k-1}\E[h(\mathfrak{u}^T(U_i+\Theta_ib_i))-h( \mathfrak{u}^T(U_i+\Theta_iZ_i))].
\end{align*}

\textbf{Bounded derivatives up to fourth order:} We assume that the test function $h(\cdot)$ has bounded derivatives up to fourth derivative
\begin{align}\label{eq:derivatives}
    \max_{x\in \mathbb{R}^d}|h^{(i)}(x)|<\infty,~~\forall i\in \{1,2,3,4\}
\end{align}
Using Taylor series expansion for $h(\cdot)$ for any $X\in \mathbb{R}$, we get
\begin{align*}
    h(\mathfrak{u}^T(U_i+X)) = h(\mathfrak{u}^TU_i)+\mathfrak{u}^TXh^{(1)}(\mathfrak{u}^TU_i)+\frac{(\mathfrak{u}^TX)^2}{2}h^{(2)}(\mathfrak{u}^TU_i)+\frac{(\mathfrak{u}^TX)^3}{6}h^{(3)}(\mathfrak{u}^TU_i)+\mathfrak{R}_4(X, U_i),
\end{align*}
where $|\mathfrak{R}_4(X, U_k)|\leq |\mathfrak{u}^TX|^4h^{(4)}_{max}/24$ and $h^{(4)}_{max}=\max_{x\in \mathbb{R}} h^{(4)}(x)$. 

Substituting $X=\Theta_iv^Tb_i$ for the first term and $X=\Theta_iv^TZ_i$ for the second term gives us
\begin{align*}
    h(\mathfrak{u}^Tz_k)-h(\mathfrak{u}^TY_k) &= \sum_{i=0}^{k-1}(\Theta_iv^Tb_i-\Theta_iv^TZ_i)h^{(1)}(\mathfrak{u}^TU_i)+\frac{\Theta_i^2(\mathfrak{u}^Tb_i)^2-\Theta_i^2(\mathfrak{u}^TZ_i)^2}{2}h^{(2)}(\mathfrak{u}^TU_i)\\
    &~+\frac{\Theta_i^3(\mathfrak{u}^Tb_i)^3-\Theta_i^3(\mathfrak{u}^TZ_i)^3}{6}h^{(3)}(\mathfrak{u}^TU_i)
    +\mathfrak{R}_4(\Theta_ib_i, U_i)-\mathfrak{R}_4(\Theta_iZ_i, U_i).
\end{align*}
Recall that $b_i$ and $Z_i$ are independent of $U_i$ by construction. Furthermore, $\E[b_i]=\E[Z_i]=0$ and $\E[b_ib_i^T]=\E[Z_iZ_i^T]=I$. Thus, taking expectation both sides gives us
\begin{align*}
    \E[h(\mathfrak{u}^Tz_k)-h(\mathfrak{u}^TY_k)] &= \sum_{i=0}^{k-1}\frac{(\Theta_i^3\E[(\mathfrak{u}^Tb_i)^3]-\Theta_i^3\E[(\mathfrak{u}^TZ_i)^3])}{6}\E[h^{(3)}(\mathfrak{u}^TU_i)]\\
    &~+\E[\mathfrak{R}_4(\Theta_ib_i, U_i)]-\E[\mathfrak{R}_4(\Theta_iZ_i, U_i)].
\end{align*}
Using reverse triangle inequality, we have
\begin{align*}
    |\E[h(\mathfrak{u}^Tz_k)-h(\mathfrak{u}^TY_k)]| &\geq  \frac{|\E[(\mathfrak{u}^Tb_1)^3]-\E[(\mathfrak{u}^TZ_1)^3]|}{6}\left|\sum_{i=0}^{k-1}\Theta_i^3\E[h^{(3)}(\mathfrak{u}^TU_i)]\right|\\
    &~-\left|\sum_{i=0}^{k-1}\E[\mathfrak{R}_4(\Theta_ib_i, U_i)]-\E[\mathfrak{R}_4(\Theta_iZ_i, U_i)]\right|.
\end{align*}

\textbf{$\E[h^{(3)}(\mathfrak{u}^TU_i)]$ are sign-consistent across $i$ and are uniformly bounded away from zero:} Denote the lower bound on the third derivative as 
\begin{align}\label{eq:third_derivative}
    |\E[h^{(3)}(\mathfrak{u}^TU_i)]|\geq h^{(3)}_{min}    
\end{align}
Since $\Theta_i^3>0$, we have
\begin{align*}
    |\E[h(\mathfrak{u}^Tz_k)-h(\mathfrak{u}^TY_k)]| &\geq  \frac{|\E[(\mathfrak{u}^Tb_1)^3]-\E[(\mathfrak{u}^TZ_1)^3]|}{6}h^{(3)}_{min}\left|\sum_{i=0}^{k-1}\Theta_i^3\right|\\
    &~-\left|\sum_{i=0}^{k-1}\E[\mathfrak{R}_4(\Theta_ib_i, U_i)]-\E[\mathfrak{R}_4(\Theta_iZ_i, U_i)]\right|.
\end{align*}

Furthermore, using the multinomial theorem, we have
\begin{align*}
    (\mathfrak{u}^TZ_1)^3&=\left(\sum_{i=1}^d\mathfrak{u}(i)Z_1(i)\right)^3\\
    &=\sum_{\sum_{i=1}^da_i=3}\frac{3!}{\prod_{i=1}^d(a_i)!}\prod_{i=1}^d\mathfrak{u}(i)^{a_i}Z_1(i)^{a_i}.
\end{align*}
Note that the sum $\sum_{i=1}^da_i=3$ will always have odd terms and since $Z_1$ is a sample from standard Gaussian, the odd moments are zero. Therefore, $\E[(\mathfrak{u}^TZ_1)^3]=0$. Since $\Theta_i^3>0$, we obtain the following
\begin{align*}
     |\E[h(\mathfrak{u}^Tz_k)-h(\mathfrak{u}^TY_k)]| &\geq  \frac{|\E[(\mathfrak{u}^Tb_1)^3]|h^{(3)}_{min}}{6}\sum_{i=0}^{k-1}\Theta_i^3-\left|\sum_{i=0}^{k-1}\E[\mathfrak{R}_4(\Theta_ib_i, U_i)]-\E[\mathfrak{R}_4(\Theta_iZ_i, U_i)]\right|\\
     &\geq  \frac{|\E[(\mathfrak{u}^Tb_1)^3]|h^{(3)}_{min}}{6}\sum_{i=0}^{k-1}\Theta_i^3-\sum_{i=0}^{k-1}\E[|\mathfrak{R}_4(\Theta_ib_i, U_i)|]-\E[|\mathfrak{R}_4(\Theta_iZ_i, U_i)|]\\
     &\geq \frac{|\E[(\mathfrak{u}^Tb_1)^3]|h^{(3)}_{min}}{6}\sum_{i=0}^{k-1}\Theta_i^3-\frac{h^{(4)}_{max}}{24}\sum_{i=0}^{k-1}(\Theta_i^4\E[(\mathfrak{u}^Tb_i)^4]+\Theta_i^4\E[(\mathfrak{u}^TZ_i)^4])\\
     &=\frac{|\E[(\mathfrak{u}^Tb_1)^3]|h^{(3)}_{min}}{6}\sum_{i=0}^{k-1}\Theta_i^3-\frac{|\E[(\mathfrak{u}^Tb_1)^4]|+|\E[(\mathfrak{u}^TZ_1)^4]|}{24}h^{(4)}_{max}\sum_{i=0}^{k-1}\Theta_i^4
\end{align*}
Now, we can apply Lemma \ref{lem:bounds} with $p=3$ and $p=4$ for each term respectively.
\begin{align*}
    |\E[h(\mathfrak{u}^Tz_k)-h(\mathfrak{u}^TY_k)]| \geq \frac{|\E[(\mathfrak{u}^Tb_1)^3]|h^{(3)}_{min}}{18}\sqrt{\alpha_k}-\frac{|\E[(\mathfrak{u}^Tb_1)^4]|+|\E[(\mathfrak{u}^TZ_1)^4]|}{32}h^{(4)}_{max}\alpha_k.
\end{align*}

Some functions that satisfy the properties listed in Eq. \eqref{eq:lipschitz}, \eqref{eq:derivatives}, and \eqref{eq:third_derivative} are $\sin(\mathfrak{u}^Tx)$ and the complex exponential $e^{ju^Tx}$ (here $j=\sqrt{-1}$). Both functions are 1-Lipschitz and infinitely differentiable with all the derivatives upper bounded by 1. For the third derivative condition, we note that we have the universal bound on $\cos$
\begin{align*}
    \cos \mathfrak{u}^TU_i \geq 1-\frac{1}{2}(\mathfrak{u}^TU_i)^2.
\end{align*}
Taking expectation on both sides, we obtain
\begin{align*}
    \E[\cos \mathfrak{u}^TU_i] \geq 1-\frac{1}{2}\E[(\mathfrak{u}^TU_i)^2] = 1-\frac{\|\mathfrak{u}\|^2}{2}\left(\sum_{j=0}^{k-1}\Theta_j^2-\Theta_i^2\right).
\end{align*}
Using Lemma \ref{lem:bounds} for $p=2$ and $\|\mathfrak{u}\|=1$, we get
\begin{align*}
    \E[\cos \mathfrak{u}^TU_i] \geq 1-\frac{1}{2}\left(\frac{3}{2}-\Theta_i^2\right)=\frac{1}{4}+\frac{1}{2}\Theta_i^2>\frac{1}{4}.
\end{align*}
Thus, both functions have a sign-consistent third derivative for all $i$ and are uniformly bounded away from zero.

\begin{lemma}\label{lem:bounds}
    Let $\Theta_i=\sqrt{\alpha_i}\prod_{l=i+1}^{k-1}(1 + \alpha_lJ_l^{(\alpha, \xi)})$ and $\Theta_k=\sqrt{\alpha_k}$. Then, for any $p\geq 2$ and small enough $\alpha_0$, we have the following for all $k\geq 1$
    \begin{align*}
        \frac{1}{p}\alpha_k^{p/2-1}\leq \sum_{j=0}^{k-1}\Theta_j^p\leq \frac{3}{p}\alpha_k^{p/2-1}.
    \end{align*}
\end{lemma}
\begin{proof}
    Note that the given sum can be written in a recursive form as follows
    \begin{align*}
        v_{k+1}=(1+\alpha_kJ_k^{(\alpha, \xi)})^pv_k+\alpha_k^{p/2},
    \end{align*}
    where $v_0=0$. 
    
    To show the lower bound, we will use the inequality $(1+\alpha_kJ_k^{(\alpha, \xi)})^p\geq 1+p\alpha_kJ_k^{(\alpha, \xi)}$ for small values of $\alpha_kJ_k^{(\alpha, \xi)}$. Assume that $v_k\geq \alpha_k^{p/2-1}/p$. Then, for $k+1$, we have
    \begin{align*}
        v_{k+1}-\frac{\alpha_{k+1}^{p/2-1}}{p}&\geq (1+p\alpha_kJ_k^{(\alpha, \xi)})v_k+\alpha_k^{p/2}\\
        &\geq \left(1-p\alpha_k+\frac{p\xi}{2(k+K)}\right)\frac{\alpha_{k}^{p/2-1}}{p}+\alpha_k^{p/2}-\frac{\alpha_{k+1}^{p/2-1}}{p}\\
        &=\frac{\alpha_{k}^{p/2-1}}{p}-\frac{\alpha_{k+1}^{p/2-1}}{p}+\frac{\xi\alpha_{k}^{p/2-1}}{2(k+K)}\geq 0.
    \end{align*}

    For upper bound, we use $(1+\alpha_kJ_k^{(\alpha, \xi)})^p =(1-\alpha_k+p\xi/(2(k+K)))^p \leq 1-2p\alpha_k/3$ for small enough $\alpha_k$, to get
    \begin{align*}
        v_{k+1}\leq (1-2p\alpha_k/3)v_k+\alpha_k^{p/2}.
    \end{align*}
    Now, we can apply Lemma \ref{lem:rec_sol} with $\mu_1=2p/3$, $u_0=0$, $\mu_2=0$, $\mu_3=1$, and $\rho_2=p/2-1$, to get
    \begin{align*}
        v_k\leq \frac{3}{p}\alpha_k^{p/2-1}.
    \end{align*}
\end{proof}

\subsection{Proof of Theorem \ref{thm_main:main_thm}}\label{sec:proof_main_thm}
We first state the following lemma which upper bounds the MSE error for the iterates given by SA \eqref{eq:SA_rec2}. We define the following constants to characterize the MSE bounds.
\begin{align}\label{eq:var_sig}
    \alpha^{(0)}=\frac{(2-\varsigma_1)\gamma}{uu_{2s}L_s(L_F+A_1)};~~\varsigma_0=\frac{u}{l};~~\varsigma_2=uu_{2s}L_sB_2.
\end{align}
where $u_{2s}$ is the norm equivalence constant such that $\|x\|_s\leq u_{2s}\|x\|$ and the other constants are defined in Assumptions \ref{assump:iterate}-\ref{assump:noise}.

\begin{lemma}\label{lem:iterate_bounds}
Let $1<\varsigma_1<2$ be any arbitrary constant. Then, the following bounds hold for all $k\geq 0$.
    \begin{enumerate}
        \item When $\xi=0$ and $\alpha\leq \alpha^{(0)}$, we have{\normalfont :}  
        \begin{align*}
            \E[\|x_k-x^*\|^2]\leq \varsigma_0\E[\|x_0-x^*\|^2]\exp(-\varsigma_1\gamma\alpha k)+\frac{\varsigma_2\alpha}{\varsigma_1\gamma}.
        \end{align*}
        \item When $\xi\in (0, 1)$ and $K$ is large enough such that $K\geq (1/(\gamma\varsigma_1\alpha))^{1/(1-\xi)}$ and $\alpha_k\leq \alpha^{(0)}$, we have{\normalfont :}  
        \begin{align*}
            \E[\|x_k-x^*\|^2]\leq \varsigma_0\E[\|x_0-x^*\|^2]\exp\left(-\frac{\varsigma_1\gamma\alpha}{1-\xi}\left((k+K)^{1-\xi}-K^{1-\xi}\right)\right)+\frac{2\varsigma_2\alpha_k}{\varsigma_1\gamma}.
        \end{align*}
        \item When $\xi=1$, $\varsigma_1\gamma\alpha> 1$, and $K$ is large enough such that $\alpha_k\leq \alpha^{(0)}$, we have{\normalfont :}
        \begin{align*}
            \E[\|x_k-x^*\|^2]\leq \varsigma_0\E[\|x_0-x^*\|^2]\left(\frac{K}{k+K}\right)^{\varsigma_1\gamma\alpha}+\frac{\alpha\varsigma_2\alpha_k}{\varsigma_1\gamma\alpha-1}.
        \end{align*}
    \end{enumerate}
    
\end{lemma}

In the following, we will fix the value of $\varsigma_1=3/2$ for ease of exposition. Furthermore, for $\xi\in (0, 1]$, we will assume $K_0$ is large enough such that the step-size conditions in the Lemma \ref{lem:iterate_bounds} and Lemma \ref{lem:theta_bound}  are satisfied for any $K\geq K_0$. We now state the following lemma that handles multiplicative noise in the algorithm and makes the intuition presented in Section \ref{sec:proof_sketch} precise. Define the following constants:
\begin{align}\label{eq:var_theta}
    \vartheta_1=\lambda^{max}_V\sqrt{2A_1\varsigma_0};~~~\vartheta_2=\lambda^{max}_V\sqrt{2A_1\varsigma_2}/\sqrt{3}.
\end{align}

\begin{lemma}\label{lem:multiplicative_noise}
    Consider the sequence:
    \begin{align*}
        \hat{z}_{k+1}=(I+\alpha_kJ^{(\alpha, \xi)})\hat{z}_k+\sqrt{\alpha_k}M_k
    \end{align*}
    where $\hat{z}_0=z_0$ and $M_k$ satisfies Assumption \ref{assump:noise}. Then, for $\{z_k\}_{k\geq 0}$ defined in Theorem \ref{thm:weighted_stein} and sharing the same data sequence $\{w_k\}_{k\geq 0}$, we have for all $k\geq 0$
    \begin{enumerate}
        \item When $\xi=0$ and $\alpha\leq \min(1, 2\iota_V/\|J\|^2_V, \alpha^{(0)})${\normalfont :}  
        \begin{align*}
            \E[\|z_k-\hat{z}_k\|^2]\leq \frac{\vartheta_1^2\E[\|x_0-x^*\|^2]e^{\iota_V\alpha}}{2\iota_V-3\gamma}\left(e^{-3\gamma k/2}-e^{-\iota_V\alpha k}\right)+\frac{\vartheta_2^2\alpha}{\iota_V\gamma}.
        \end{align*}
        \item When $\xi\in (0, 1)$ and $K\geq K_0${\normalfont :}
        \begin{align*}
            \E[\|z_k-\hat{z}_k\|^2] &\leq \frac{\vartheta_1^2\E[\|x_0-x^*\|^2]e^{\iota_V\alpha}}{2\iota_V-3\gamma} \left(e^{-\frac{3\gamma\alpha}{2(1-\xi)}(k+K)^{1-\xi}}-e^{-\frac{\iota_V\alpha}{1-\xi}(k+K)^{1-\xi}}\right)+\frac{4\vartheta_2^2\alpha_k}{\iota_V\gamma}.
        \end{align*}
        \item When $\xi=1$, $3\alpha\gamma>2$, $\iota_V\alpha>1$, and $K\geq K_0$, we have the following three cases:
        \begin{enumerate}
            \item[(a)] When $2\iota_V>3\gamma${\normalfont :} 
            \begin{align*}
                \E[\|z_k-\hat{z}_k\|^2]\leq \frac{ \vartheta_1^2\E[\|x_0-x^*\|^2]K^{3\gamma\alpha/2}\alpha_k^{3\gamma\alpha/2}}{\alpha^{3\gamma/2}(2\iota_V-3\gamma)}+\frac{3\vartheta_2^2 \alpha^2\alpha_k}{(3\gamma\alpha-2)(\iota_V\alpha-1)}.
            \end{align*}
            \item[(b)] When $2\iota_V<3\gamma${\normalfont :}
            \begin{align*}
                \E[\|z_k-\hat{z}_k\|^2]&\leq \frac{2\vartheta_1^2\E[\|x_0-x^*\|^2](2K)^{3\gamma\alpha/2}\alpha_k^{\iota_V\alpha}}{\alpha^{\iota_V\alpha}(3\gamma-2\iota_V)}+\frac{ 3\vartheta_2^2\alpha^2\alpha_k}{(3\gamma\alpha-2)(\iota_V\alpha-1)}.
            \end{align*}
            \item[(c)] When $2\iota_V=3\gamma${\normalfont :}
            \begin{align*}
                \E[\|z_k-\hat{z}_k\|^2]&\leq \vartheta_1^2\E[\|x_0-x^*\|^2](2K)^{\iota_V\alpha}\alpha^{1-\iota_V\alpha}\log(k+K)\alpha_k^{\iota_V\alpha}+\frac{ 3\vartheta_2^2\alpha^2\alpha_k}{(3\gamma\alpha-2)(\iota_V\alpha-1)}.
            \end{align*}
        \end{enumerate}
    \end{enumerate}
\end{lemma}

\begin{proposition}\label{prop:global_linear}
    Suppose that the operator $F(\cdot)$ satisfies Assumption \ref{assump:operator}. Then, it also admits a global linear approximation given as:
    \begin{align*}
        F(x)&=J_F(x-x^*)+R(x)~~ \forall x\in \mathbb{R}^d,
    \end{align*}
    where $\|R(x)\|\leq R_1\|x-x^*\|^{1+\delta}$ and $R_1=\max(C_r, (L_F+\|J_F\|)/r^{\delta})$.
\end{proposition}
Note that for ease of exposition, we reuse the same notation for the remainder term $R(x)$ in Assumption \ref{assump:operator} and the above proposition. In the following, $R(\cdot)$ will denote the remainder term for the global linear approximation. We remark that the above proposition is a direct consequence of Assumption \ref{assump:operator}. The Lipschitz continuity of $F(\cdot)$ immediately gives us $\|F(x)-J_F(x-x^*)\|\leq (L_F+\|J_F\|)\|x-x^*\|$. Thus, error due to the linear approximation is at most linear globally. And therefore, for $\|x-x^*\|>\max(1, r)$, $\|F(x)-J_F(x-x^*)\|\leq \mathcal{O}(\|x-x^*\|^{1+\delta})$ is in fact a ``worse'' upper bound on the approximation error. On the other hand, in the local neighborhood of $x^*$, $\|x-x^*\|^{1+\delta}$ provides a tighter approximation than purely from the Lipschitz condition. The motive behind the above proposition is to consolidate these varying order of approximations into a single upper bound which helps to simplify some of the calculations in the following analysis.

Define the following constants:
\begin{align}\label{eq:varrho}
\begin{split}
    \varrho_1&=4\lambda^{max}_V\varsigma_0^{(1+\delta)/2}R_1;~~\varrho_2=\frac{4}{3}\lambda^{max}_V\varsigma_2^{(1+\delta)/2}R_1;\\
    \varrho_3( \alpha)&=\sqrt{\varsigma_0\lambda^{max}_V}\left(\frac{1}{4\alpha^2}+\frac{\sqrt{2}(L_F+A_1)}{\alpha}\right);\\
    \varrho_4(\alpha, \gamma)&=\frac{\sqrt{2\lambda^{max}_V}B_1}{\alpha}+\left(\frac{1}{4\alpha^2}+\frac{\sqrt{2}(L_F+A_1)}{\alpha}\right)\sqrt{\frac{4\varsigma_2\alpha_0\lambda^{max}_V}{2\gamma}};\\
    \varrho_5(\alpha, \gamma)&=\frac{\sqrt{2\lambda^{max}_V}B_1}{\alpha}+\left(\frac{1}{4\alpha^2}+\frac{\sqrt{2}(L_F+A_1)}{\alpha}\right)\sqrt{\frac{2\alpha\varsigma_2\alpha_0}{3\gamma\alpha-2}}.
\end{split}
\end{align}

\begin{theorem}\label{thm:main_thm}
The sequence $\{y_k\}_{k\geq 0}$ given by rescaling the iterate $x_k$ in update equation \eqref{eq:SA_rec2} satisfies the following bounds for all $k\geq 1$.
    \begin{enumerate}
        \item When $\xi=0$ and $\alpha\leq \min(1, 2\iota_V/\|J\|^2_V, \alpha^{(0)})$, we have{\normalfont :}
        \begin{align*}
            d_{\mathcal{W}}(y_k, (\Sigma_k^{(\alpha, 0)})^{1/2}Z)&\leq \frac{\varrho_2}{\iota_V\gamma}\alpha^{\delta/2}+\frac{\varphi_1(\alpha, 0)\sqrt{\alpha}}{\iota_V}\log\left(\frac{1}{\alpha}\right)+\frac{\vartheta_2\sqrt{\alpha}}{\sqrt{\iota_V\gamma}}\\
            &~+\frac{2\sqrt{\lambda^{max}_V}}{\alpha}\E[\|x_0-x^*\|_{V}]e^{-\iota_V\alpha k/2}\\
            &~+\vartheta_1\sqrt{\E[\|x_0-x^*\|^2]}e^{\iota_V\alpha/2}\sqrt{\frac{e^{-3\gamma \alpha k/2}-e^{-\iota_V\alpha k}}{2\iota_V-3\gamma}}\\
            &~+\frac{\varrho_1\E[\|x_0-x^*\|^2]^{\frac{1+\delta}{2}}e^{\iota_V\alpha/2}}{\sqrt{\alpha}(2\iota_V-3\gamma)}\left(e^{-3\gamma\alpha k/4}-e^{-\iota_V\alpha k/2}\right).
        \end{align*}
        \item When $\xi\in (0, 1)$ and $K$ is large enough such that $K\geq K_0$ and $\exp\left(-\frac{3\delta\gamma\alpha}{4(1-\xi)} (k+K)^{1-\xi} \right)\leq \sqrt{\alpha_k}$, we have{\normalfont :}
        \begin{align*}
            &d_{\mathcal{W}}(y_k, (\Sigma_k^{(\alpha, \xi)})^{1/2}Z)\leq 
            \left(\frac{8\varrho_2}{\iota_V}+\frac{4\varrho_3(\alpha)}{\iota_V}\right)\frac{\alpha^{\frac{\delta}{2}}}{(k+K)^{\frac{\delta\xi}{2}}}\\
            &~+\frac{\varphi_1(\alpha, \xi)}{\iota_V}\frac{\sqrt{\alpha}}{(k+K)^{\frac{\xi}{2}}}\log\left(\frac{(k+K)^{\xi}}{\alpha}\right)+\frac{2\vartheta_2}{\sqrt{\iota_V\gamma}}\frac{\sqrt{\alpha}}{(k+K)^{\frac{\xi}{2}}}\\
            &~+\frac{2K^\xi\sqrt{\lambda^{max}_V}}{\alpha}\E[\|x_0-x^*\|_{V}]e^{-\frac{\iota_V\alpha}{2(1-\xi)}\left((k+K)^{1-\xi}-K^{1-\xi}\right)}\\
            &~+\vartheta_1\sqrt{\E[\|x_0-x^*\|^2]}e^{\iota_V\alpha/2}\sqrt{\frac{e^{-\frac{3\gamma\alpha}{2(1-\xi)}(k+K)^{1-\xi}}-e^{-\frac{\iota_V\alpha}{1-\xi}(k+K)^{1-\xi}}}{2\iota_V-3\gamma}}\\
            &~+\frac{4e^{\iota_V\alpha/2}}{2\iota_V-3\gamma}\Bigg(\frac{\varrho_1}{2}\E[\|x_0-x^*\|^2]^{\frac{1+\delta}{2}}e^{\frac{3\delta\gamma\alpha K^{1-\xi}}{4(1-\xi)}}+2\varrho_4(\alpha, \gamma)\sqrt{\E[\|x_0-x^*\|^2]}\Bigg)\times\\
            &~\left(e^{-\frac{3\gamma\alpha}{4(1-\xi)}\left((k+K)^{1-\xi}-K^{1-\xi}\right)}-e^{-\frac{\iota_V\alpha}{2(1-\xi)}\left((k+K)^{1-\xi}-K^{1-\xi}\right)}\right).
        \end{align*}
        \item When $\xi=1$, $\iota_V\alpha>2$, $3\gamma\delta\alpha>2$ and $K$ is large enough such that $K\geq K_0$ {\normalfont :}
        \begin{enumerate}
            \item[(a)] When $2\iota_V>3\gamma$, we have{\normalfont :}
            \begin{align*}
                &d_{\mathcal{W}}(y_k, (\Sigma_k^{(\alpha, 1)})^{1/2}Z)\leq \left(3\varrho_2\left(\frac{\alpha}{3\gamma\alpha-2}\right)^{\frac{1+\delta}{2}}+\varrho_5(\alpha, \gamma)\right)\frac{2\alpha^{1+\frac{\delta}{2}}}{\iota_V\alpha-2}\frac{1}{(k+K)^{\frac{\delta}{2}}}\\
                &~+\frac{2\varphi_1(\alpha, 1)\alpha^{3/2}}{2\iota_V\alpha-1}\frac{1}{\sqrt{k+K}}\log\left(\frac{k+K}{\alpha}\right)+\sqrt{\frac{3 \vartheta_2^2\alpha^3}{(3\gamma\alpha-2)(\iota_V\alpha-1)}}\frac{1}{\sqrt{k+K}}\\
                &~+\frac{2K\sqrt{\lambda^{max}_V}}{\alpha}\E[\|x_0-x^*\|_{V}]\left(\frac{K}{k+K}\right)^{\frac{\iota_V\alpha}{2}}\\
                &~+\left(\varrho_1\E[\|x_0-x^*\|^2]^{\frac{1+\delta}{2}}\left(\frac{K}{\alpha}\right)^{\frac{3\delta\gamma\alpha}{4}}+4\varrho_3(\alpha)\sqrt{\E[\|x_0-x^*\|^2]}\right)\times\\
                &\frac{1}{2\iota_V-3\gamma}\left(\frac{K}{k+K}\right)^{\frac{3\gamma\alpha}{4}}+\sqrt{\frac{\vartheta_1^2 \E[\|x_0-x^*\|^2]}{(2\iota_V-3\gamma)}}\left(\frac{K}{k+K}\right)^{\frac{3\gamma\alpha}{4}}.
            \end{align*}
            \item[(b)] When $2\iota_V<3\gamma$, we have{\normalfont :}
            \begin{align*}
                &d_{\mathcal{W}}(y_k, (\Sigma_k^{(\alpha, 1)})^{1/2}Z)\leq
                \left(3\varrho_2\left(\frac{\alpha}{3\gamma\alpha-2}\right)^{\frac{1+\delta}{2}}+\varrho_5(\alpha, \gamma)\right)\frac{2\alpha^{1+\frac{\delta}{2}}}{\iota_V\alpha-2}\frac{1}{(k+K)^{\frac{\delta}{2}}}\\
                &~+\frac{2\varphi_1(\alpha, 1)\alpha^{3/2}}{2\iota_V\alpha-1}\frac{1}{\sqrt{k+K}}\log\left(\frac{k+K}{\alpha}\right)+\sqrt{\frac{3 \vartheta_2^2\alpha^3}{(3\gamma\alpha-2)(\iota_V\alpha-1)}}\frac{1}{\sqrt{k+K}}\\
                &~+\frac{2K\sqrt{\lambda^{max}_V}}{\alpha}\E[\|x_0-x^*\|_{V}]\left(\frac{K}{k+K}\right)^{\frac{\iota_V\alpha}{2}}\\
                &\left(2\varrho_1(J[\|x_0-x^*\|^2]^{\frac{1+\delta}{2}}\left(\frac{K}{\alpha}\right)^{\frac{3\delta\gamma\alpha}{4}}+8\varrho_3(\alpha)\sqrt{\E[\|x_0-x^*\|^2]}\right)\times\\
                &\frac{(2K)^{\frac{(3\gamma-2\iota_V)\alpha}{4}}}{(3\gamma-2\iota_V)}\left(\frac{K}{k+K}\right)^{\frac{\iota_V\alpha}{2}}+\sqrt{\frac{2\vartheta_1^2 \E[\|x_0-x^*\|^2]}{(3\gamma-2\iota_V)}}\frac{(2K)^{\frac{3\gamma\alpha}{4}}}{(k+K)^{\frac{\iota_V\alpha}{2}}}.
            \end{align*}
            \item[(c)] When $2\iota_V=3\gamma$, we have{\normalfont :}
            \begin{align*}
                &d_{\mathcal{W}}(y_k, (\Sigma_k^{(\alpha, 1)})^{1/2}Z)\leq 
                \left(3\varrho_2\left(\frac{\alpha}{3\gamma\alpha-2}\right)^{\frac{1+\delta}{2}}+\varrho_5(\alpha, \gamma)\right)\frac{2\alpha^{1+\frac{\delta}{2}}}{\iota_V\alpha-2}\frac{1}{(k+K)^{\frac{\delta}{2}}}\\
                &~+\frac{2\varphi_1(\alpha, 1)\alpha^{3/2}}{2\iota_V\alpha-1}\frac{1}{\sqrt{k+K}}\log\left(\frac{k+K}{\alpha}\right)+\sqrt{\frac{3 \vartheta_2^2\alpha^3}{(3\gamma\alpha-2)(\iota_V\alpha-1)}}\frac{1}{\sqrt{k+K}}\\
                &~+\frac{2K\sqrt{\lambda^{max}_V}}{\alpha}\E[\|x_0-x^*\|_{V}]\left(\frac{K}{k+K}\right)^{\frac{\iota_V\alpha}{2}}\\
                &~+\left(2\varrho_1\E[\|x_0-x^*\|^2]^{\frac{1+\delta}{2}}\left(\frac{K}{\alpha}\right)^{\frac{3\delta\gamma\alpha}{4}}+8\varrho_3(\alpha)\sqrt{\E[\|x_0-x^*\|^2]}\right)\times\\
                &~\alpha\log(k+K)\left(\frac{2K}{k+K}\right)^{\frac{\iota_V\alpha}{2}}+\vartheta_1\sqrt{\E[\|x_0-x^*\|^2]\alpha\log(k+K)}\left(\frac{2K}{k+K}\right)^{\frac{\iota_V\alpha}{2}}.
            \end{align*}
        \end{enumerate}
    \end{enumerate}
\end{theorem}
\begin{proof}
For ease of exposition, we drop the subscript $F$ from $J_F$ in the following wherever it is clear from context.
    \begin{enumerate}
        \item Recall that the SA update is given by Eq. \eqref{eq:SA_rec} as follows:
        \begin{align*}
            x_{k+1}=x_k+\alpha_k(F(x_k)+M_k).
        \end{align*}
         Define $y_k=(x_k-x^*)/\sqrt{\alpha_k}$ as the centered rescaled version of $x_k$. In this case, $\alpha_k=\alpha$. Then, the update equation for $y_k$ is given by:
        \begin{align*}
            y_{k+1}&=y_k+\sqrt{\alpha}(F(x_k)+M_k)\\
            &=(I+\alpha J)y_k+\sqrt{\alpha}M_k\tag{Assumption \ref{assump:operator}}
        \end{align*}
        Next, we define $z_k$ and $\hat{z}_k$ as in Theorem \ref{thm:weighted_stein} and Lemma \ref{lem:multiplicative_noise}, respectively. Then, using triangle inequality, we get
        \begin{align*}
            d_{\mathcal{W}}(y_k, \mathcal{N}(0, \Sigma_k^{(\alpha, \xi)}))&\leq d_{\mathcal{W}}(y_k, z_k)+d_{\mathcal{W}}(z_k, \mathcal{N}(0, \Sigma_k^{(\alpha, \xi)}))\\
            &\leq d_{\mathcal{W}}(y_k, \hat{z}_k)+d_{\mathcal{W}}(\hat{z}_k, z_k)+d_{\mathcal{W}}(z_k, \mathcal{N}(0, \Sigma_k^{(\alpha, \xi)})).
        \end{align*}
        Since we have a bound on $d_{\mathcal{W}}(z_k, \mathcal{N}(0, \Sigma_k^{(\alpha, \xi)}))$ from Theorem \ref{thm:weighted_stein}, our goal now is to bound $d_{\mathcal{W}}(y_k, \hat{z}_k)+d_{\mathcal{W}}(\hat{z}_k, z_k)$.
        Recall that $d_{\mathcal{W}}(Y, Z)=\inf_{\gamma\in \Gamma(\nu_X, \nu_Y)}\E_{(X, Y)\sim \gamma}[\|X-Y\|]$. Thus, for any arbitrary coupling $\gamma$, $d_{\mathcal{W}}(Y, Z)\leq \E_{(X, Y)\sim \gamma}[\|X-Y\|]$. Using this relation and the coupling such that $\{y_k\}_{k\geq 0}$, $\{z_k\}_{k\geq 0}$, and $\{\hat{z}_k\}_{k\geq 0}$ share the same underlying data stream $\{w_k\}_{k\geq 0}$, we get
        \begin{align}\label{eq:wass_bound}
            d_{\mathcal{W}}(y_k, \hat{z}_k)+d_{\mathcal{W}}(\hat{z}_k, z_k)\leq \E[\|y_k-\hat{z}_k\|]+\E[\|\hat{z}_k-z_k\|].
        \end{align}
        We handle the first term by studying the convergence of $y_k-\hat{z}_k$ as follows. This leads to
        \begin{align*}
            y_{k+1}-\hat{z}_{k+1}&=(I+\alpha J)(y_k-\hat{z}_k)+\sqrt{\alpha}R(x_k).
        \end{align*}
        Taking norm $\|\cdot\|_{V}$ on both sides and using triangle-inequality, we obtain
        \begin{align*}
            \|y_{k+1}-\hat{z}_{k+1}\|_V&\leq \|(I+\alpha J)\|_{V}\|y_k-\hat{z}_k\|_{V}+\|\sqrt{\alpha}R(x_k)\|_{V}.
        \end{align*}
        Using statement 1 of Lemma \ref{lem:contraction_prop}  for the first term and Proposition \ref{prop:global_linear} for the second term, we get
        \begin{align*}
            \|y_{k+1}-\hat{z}_{k+1}\|_V&\leq \left(1-\frac{\iota_V\alpha}{2}\right)\|y_k-\hat{z}_k\|_{V}+\sqrt{\alpha}R_1\sqrt{\lambda^{max}_V}\|x_k-x^*\|^{1+\delta}.
        \end{align*}
        Taking expectation on both sides and using Lemma \ref{lem:iterate_bounds} and Jensen's inequality for the second term, we get
        \begin{align*}
            \E[\|y_{k+1}-\hat{z}_{k+1}\|_{V}]&\leq \left(1-\frac{\iota_V\alpha}{2}\right)\E[\|y_k-\hat{z}_k\|_V]+\sqrt{\alpha}R_1\sqrt{\lambda^{max}_V}\times\\
            &~\left(\varsigma_0^{\frac{1+\delta}{2}}\E[\|x_0-x^*\|^2]^{\frac{1+\delta}{2}}\exp(-3(1+\delta)\gamma\alpha k/4)+\left(\frac{2\varsigma_2\alpha}{3\gamma}\right)^{\frac{1+\delta}{2}}\right)\\
            &\leq \left(1-\frac{\iota_V\alpha}{2}\right)\E[\|y_k-\hat{z}_k\|_V]+\sqrt{\alpha}R_1\sqrt{\lambda^{max}_V}\times\\
            &~\left(\varsigma_0^{\frac{1+\delta}{2}}\E[\|x_0-x^*\|^2]^{\frac{1+\delta}{2}}\exp(-3\gamma\alpha k/4)+\left(\frac{2\varsigma_2\alpha}{3\gamma}\right)^{\frac{1+\delta}{2}}\right).
        \end{align*}
    where for the last inequality, we used the fact that $\delta>0$.
        
        Now, we can apply Lemma \ref{lem:rec_sol} with $u_0=\E[\|y_0\|_{V}]$, $\mu_1=\iota_V/2$, $\mu_2=R_1\sqrt{\lambda^{max}_V}\varsigma_0^{\frac{1+\delta}{2}}\E[\|x_0-x^*\|^2]^{\frac{1+\delta}{2}}/\sqrt{\alpha}$, $\mu_3=2\sqrt{\lambda^{max}_V}\varsigma_2^{(1+\delta)}R_1/3$, $\rho_1=3\gamma/4$ and $\rho_2=\delta/2$ to obtain 
        \begin{align*}
            \E[\|y_k-\hat{z}_k\|_{V}]&\leq 
                \E[\|y_0\|_{V}]e^{-\iota_V\alpha k/2}+\frac{4R_1\sqrt{\lambda^{max}_V}\varsigma_0^{\frac{1+\delta}{2}}\E[\|x_0-x^*\|^2]^{\frac{1+\delta}{2}}e^{\iota_V\alpha/2}}{\sqrt{\alpha}(2\iota_V-3(1+\delta)\gamma)}\times\\
                &~\left(e^{-3\gamma\alpha k/4}-e^{-\iota_V\alpha k/2}\right)+\frac{4\sqrt{\lambda^{max}_V}\varsigma_2^{(1+\delta)/2}R_1}{3\iota_V\gamma}\alpha^{\delta/2}\\
            \implies \E[\|y_k-\hat{z}_k\|]&\leq 
                \sqrt{\lambda^{max}_V}\E[\|y_0\|_{V}]e^{-\iota_V\alpha k/2}+\frac{\varrho_1\E[\|x_0-x^*\|^2]^{\frac{1+\delta}{2}}e^{\iota_V\alpha/2}}{\sqrt{\alpha}(2\iota_V-3\gamma)}\times\\
                &~\left(e^{-3\gamma\alpha k/4}-e^{-\iota_V\alpha k/2}\right)+\frac{\varrho_2}{\iota_V\gamma}\alpha^{\delta/2}.
        \end{align*}
        For the second term in Eq. \eqref{eq:wass_bound}, we apply Jensen's inequality. Combining all the bounds, we have the upper bound for the first case.
        \item In this case $\alpha_k\neq \alpha_{k+1}$ and thus the update equation for $y_k$ is given by:
        \begin{align*}
            y_{k+1}&=y_k+\sqrt{\alpha_k}(F(x_k)+M_k)+\sqrt{\alpha_k}y_k\left(\frac{1}{\sqrt{\alpha_{k+1}}}-\frac{1}{\sqrt{\alpha_k}}\right)\\
            &~+\alpha_k(F(x_k)+M_k)\left(\frac{1}{\sqrt{\alpha_{k+1}}}-\frac{1}{\sqrt{\alpha_k}}\right)\\
            &=\left(I+\alpha_kJ_k^{(\alpha, \xi)}\right)y_k+\sqrt{\alpha_k}M_k+\sqrt{\alpha_k}R(x_k)+\underbrace{y_k\left(\frac{\sqrt{\alpha_k}}{\sqrt{\alpha_{k+1}}}-1-\frac{\xi}{2(k+K)}\right)}_{T_1}\\
            &~+\underbrace{\alpha_k(F(x_k)+M_k)\left(\frac{1}{\sqrt{\alpha_{k+1}}}-\frac{1}{\sqrt{\alpha_k}}\right)}_{T_2}\tag{Assumption \ref{assump:operator}}
        \end{align*}
        Before moving forward, we will provide a mean error bound for the third and last term. First of all, using Lemma \ref{lem:step-size_prop}, we have
        \begin{align*}
            \left|\frac{\sqrt{\alpha_k}}{\sqrt{\alpha_{k+1}}}-1-\frac{\xi}{2(k+K)}\right|&\leq \frac{\xi}{4}\left(1-\frac{\xi}{2}\right)\frac{1}{(k+K)^2}\leq \frac{\alpha_k^2}{4\alpha^2}.
        \end{align*}
         Thus, $\E[\|T_1\|]\leq \E[\|x_k-x^*\|]\alpha_k^{3/2}/(4\alpha^2)$. For $T_2$, we proceed as follows:
        \begin{align*}
            \E[\|T_2\|]&\leq \alpha_k(\E[\|F(x_k)\|+\|M_k\|])\left|\frac{1}{\sqrt{\alpha_{k+1}}}-\frac{1}{\sqrt{\alpha_k}}\right|.
        \end{align*}
        We use Assumption \ref{assump:operator}, Assumption \ref{assump:noise} and Lemma \ref{lem:step-size_prop}, to get
        \begin{align*}
            \E[\|T_2\|]&\leq \sqrt{2\alpha}(L_F\E[\|x_k-x^*\|]+A_1\E[\|x_k-x^*\|]+B_1)\frac{1}{(k+K)^{1+\xi/2}}\\
            &\leq \frac{\sqrt{2}}{\alpha}((L_F+A_1)\E[\|x_k-x^*\|]+B_1)\alpha_k^{3/2}.
        \end{align*}
        
        We define $z_k$ and $\hat{z}_k$ in a similar way as in the previous case but with $J_k^{(\alpha, \xi)}$ as the linear operator. For $d_{\mathcal{W}}(z_k, \mathcal{N}(0, \Sigma_k^{(\alpha, \xi)}))$, we use the second case in Theorem \ref{thm:weighted_stein}. Thus, our main goal now is to bound $\E[\|y_k-\hat{z}_k\|]$. To this end, we proceed as follows.
        \begin{align*}
            y_{k+1}-\hat{z}_{k+1}&=(I+\alpha_kJ_k^{(\alpha, \xi)})(y_k-\hat{z}_k)+\sqrt{\alpha_k}R(x_k)+T_1+T_2.
        \end{align*}
        Taking norm $\|\cdot\|_{V}$ on both sides and using triangle-inequality, we obtain
        \begin{align*}
            \|y_{k+1}-\hat{z}_{k+1}\|_{V}&\leq \|(I+\alpha_kJ_k^{(\alpha, \xi)})\|_V\|y_k-\hat{z}_k\|_{V}+\|\sqrt{\alpha_k}R(x_k)\|_{V}+\|T_1\|_V+\|T_2\|_V.
        \end{align*}
        Using statement 1 of Lemma \ref{lem:contraction_prop}  for the first term and Proposition \ref{prop:global_linear} for the second term , we get
        \begin{align*}
            \|y_{k+1}-\hat{z}_{k+1}\|_V&\leq \left(1-\frac{\iota_V\alpha_k}{2}\right)\|y_k-\hat{z}_k\|_{V}+\sqrt{\alpha_k}R_1\sqrt{\lambda^{max}_V}\|x_k-x^*\|^{1+\delta}+\|T_1\|_V+\|T_2\|_V.
        \end{align*}
        Taking expectation on both sides and using the upper bounds on $T_1$ and $T_2$, we get
        \begin{align*}
            \E[\|y_{k+1}-\hat{z}_{k+1}\|_{V}]&\leq \left(1-\frac{\iota_V\alpha_k}{2}\right)\E[\|y_k-\hat{z}_k\|_{V}]+\sqrt{\alpha_k}R_1\sqrt{\lambda^{max}_V}\E[\|x_k-x^*\|^{1+\delta}]\\
            &~+\frac{\sqrt{2}B_1}{\alpha}\alpha_k^{3/2}+\left(\frac{1}{4\alpha^2}+\frac{\sqrt{2}(L_F+A_1)}{\alpha}\right)\alpha_k^{3/2}\times\\
            &~\left(\sqrt{\varsigma_0\E[\|x_0-x^*\|^2]}e^{-\frac{3\gamma\alpha}{4(1-\xi)}\left((k+K)^{1-\xi}-K^{1-\xi}\right)}+\sqrt{\frac{4\varsigma_2\alpha_k}{3\gamma}}\right)\\
            &\leq \left(1-\frac{\iota_V\alpha}{2}\right)\E[\|y_k-\hat{z}_k\|_{V}]+\sqrt{\alpha_k}R_1\sqrt{\lambda^{max}_V}\E[\|x_k-x^*\|^{1+\delta}]\\
            &~+\frac{\varrho_3(\alpha)}{\sqrt{\lambda^{max}_V}}\alpha_k\sqrt{\E[\|x_0-x^*\|^2]}e^{-\frac{3\gamma\alpha}{4(1-\xi)}\left((k+K)^{1-\xi}-K^{1-\xi}\right)}\\
            &~+\frac{\varrho_4(\alpha, \gamma)}{\sqrt{\lambda^{max}_V}}\alpha_k^{3/2}.
        \end{align*}
        For the second term, we proceed in a similar fashion as in the previous case, to get
        \begin{align*}
            \E[\|x_k-x^*\|^{1+\delta}]\leq\varsigma_0^{\frac{1+\delta}{2}}\E[\|x_0-x^*\|^2]^{\frac{1+\delta}{2}}e^{-\frac{3(1+\delta)\gamma\alpha}{4(1-\xi)}\left((k+K)^{1-\xi}-K^{1-\xi}\right)}+\left(\frac{4\varsigma_2\alpha_k}{3\gamma}\right)^{\frac{1+\delta}{2}}
        \end{align*}
        Recall that $K$ is chosen large enough such that $e^{-\frac{3\delta\gamma\alpha}{4(1-\xi)} (k+K)^{1-\xi}}\leq \sqrt{\alpha_k}$ for all $k\geq 0$. Thus, we get
        \begin{align*}
            \sqrt{\alpha_k}\E[\|x_k-x^*\|^{1+\delta}]&\leq \varsigma_0^{\frac{1+\delta}{2}}\alpha_k\E[\|x_0-x^*\|^2]^{\frac{1+\delta}{2}}e^{\frac{3\delta\gamma\alpha K^{1-\xi}}{4(1-\xi)}}e^{-\frac{3\gamma\alpha}{4(1-\xi)}\left((k+K)^{1-\xi}-K^{1-\xi}\right)}\\
            &~+\left(\frac{4\varsigma_2}{3\gamma}\right)^{\frac{1+\delta}{2}}\alpha_k^{1+\delta/2}.
        \end{align*}
        Using these relations and the constants defined in Eq. \eqref{eq:varrho}, we get
        \begin{align*}
            \E[\|y_{k+1}-\hat{z}_{k+1}\|_{V}]
            &\leq \left(1-\frac{\iota_V\alpha_k}{2}\right)\E[\|y_k-\hat{z}_k\|_{V}]+\alpha_k\Bigg(\frac{\varrho_1}{4\sqrt{\lambda^{max}_V}}\E[\|x_0-x^*\|^2]^{\frac{1+\delta}{2}}e^{\frac{3\delta\gamma\alpha K^{1-\xi}}{4(1-\xi)}}\\
            &~+\frac{\varrho_3(\alpha)}{\sqrt{\lambda^{max}_V}}\sqrt{\E[\|x_0-x^*\|^2]}\Bigg)e^{-\frac{3\gamma\alpha}{4(1-\xi)}\left((k+K)^{1-\xi}-K^{1-\xi}\right)}\\
            &~+\left(\frac{2\varrho_2}{\sqrt{\lambda^{max}_V}}+\frac{\varrho_4(\alpha, \gamma)}{\sqrt{\lambda^{max}_V}}\right)\alpha_k^{1+\delta/2}.
        \end{align*}
        Now we apply Lemma \ref{lem:rec_sol} with $u_0=\E[\|y_0\|_V]$, $\mu_1=\iota_V/2$, $\rho_1=3\gamma/4$, $\rho_2=\delta/2$ and obvious choices for $\mu_2$ and $\mu_3$, to get
        \begin{align*}
            &\E[\|y_{k+1}-\hat{z}_{k+1}\|_{V}]
            \leq \E[\|y_0\|_V]e^{-\frac{\iota_V\alpha}{2(1-\xi)}\left((k+K)^{1-\xi}-K^{1-\xi}\right)}\\
            &~+\frac{4e^{\iota_V\alpha/2}}{2\iota_V-3\gamma}\Bigg(\frac{\varrho_1}{2\sqrt{\lambda^{max}_V}}\E[\|x_0-x^*\|^2]^{\frac{1+\delta}{2}}e^{\frac{3\delta\gamma\alpha K^{1-\xi}}{4(1-\xi)}}+\frac{2\varrho_3( \alpha)}{\sqrt{\lambda^{max}_V}}\sqrt{\E[\|x_0-x^*\|^2]}\Bigg)\times\\
            &~\left(e^{-\frac{3\gamma\alpha}{4(1-\xi)}\left((k+K)^{1-\xi}-K^{1-\xi}\right)}-e^{-\frac{\iota_V\alpha}{2(1-\xi)}\left((k+K)^{1-\xi}-K^{1-\xi}\right)}\right)\\
            &~+\left(\frac{8\varrho_2}{\iota_V\sqrt{\lambda^{max}_V}}+\frac{4\varrho_4(\alpha, \gamma)}{\iota_V\sqrt{\lambda^{max}_V}}\right)\alpha_k^{\delta/2}\\
            \implies &\E[\|y_{k+1}-\hat{z}_{k+1}\|]
            \leq \sqrt{\lambda^{max}_V}\E[\|y_0\|_V]e^{-\frac{\iota_V\alpha}{1-\xi}\left((k+K)^{1-\xi}-K^{1-\xi}\right)}\\
            &~+\frac{4e^{\iota_V\alpha}}{4\iota_V-3\gamma}\Bigg(\frac{\varrho_1}{4}\E[\|x_0-x^*\|^2]^{\frac{1+\delta}{2}}e^{\frac{3\delta\gamma\alpha K^{1-\xi}}{4(1-\xi)}}+\varrho_3( \alpha)\sqrt{\E[\|x_0-x^*\|^2]}\Bigg)\times\\
            &~\left(e^{-\frac{3\gamma\alpha}{4(1-\xi)}\left((k+K)^{1-\xi}-K^{1-\xi}\right)}-e^{-\frac{\iota_V\alpha}{(1-\xi)}\left((k+K)^{1-\xi}-K^{1-\xi}\right)}\right)\\
            &~+\left(\frac{8\varrho_2}{\iota_V}+\frac{4\varrho_4(\alpha, \gamma)}{\iota_V}\right)\alpha_k^{\delta/2}.
        \end{align*}
        
        For $\E[\|\hat{z}_k-z_k\|]$, we apply Jensen's inequality and Lemma \ref{lem:multiplicative_noise}. Combining all the bounds, we have the upper bound for the second case.
    \item For this case, using similar steps as the previous case, we get
    \begin{align*}
        y_{k+1}&=\left(I+\alpha_kJ^{(\alpha, 1)}\right)y_k+\sqrt{\alpha_k}M_k+\sqrt{\alpha_k}R(x_k)+T_1+T_2.
    \end{align*}
    where $\E[\|T_1\|]\leq \E[\|x_k-x^*\|]\alpha_k^{3/2}/(8\alpha^2)$ and $\E[\|T_2\|]\leq \sqrt{2}((L_F+A_1)\E[\|x_k-x^*\|]+B_1)\alpha_k^{3/2}/\alpha$. 
   
    We define $z_k$ and $\hat{z}_k$ in a similar way as in the first case but with $J_F^{(\alpha)}$ as the linear operator. For $d_{\mathcal{W}}(z_k, \mathcal{N}(0, \Sigma_2))$, we use the third case in Theorem \ref{thm:weighted_stein}. As a result, our main goal again is to bound $\E[\|y_k-\hat{z}_k\|]$. We proceed in a similar fashion as in the previous case to obtain
    \begin{align*}
        \E[\|y_{k+1}-\hat{z}_{k+1}\|_{V}]&\leq \left(1-\frac{\iota_V\alpha_k}{2}\right)\E[\|y_k-\hat{z}_k\|_{V}]+\sqrt{\alpha_k}R_1\sqrt{\lambda^{max}_V}\E[\|x_k-x^*\|^{1+\delta}]\\
        &~+\frac{\sqrt{2}B_1}{\alpha}\alpha_k^{3/2}+\left(\frac{1}{8\alpha^2}+\frac{\sqrt{2}(L_F+A_1)}{\alpha}\right)\alpha_k^{3/2}\times\\
        &~\left(\sqrt{\varsigma_0\E[\|x_0-x^*\|^2]}\left(\frac{K}{k+K}\right)^{\frac{3\gamma\alpha}{4}}+\sqrt{\frac{2\alpha\varsigma_2\alpha_k}{3\gamma\alpha-2}}\right)\\
        &\leq \left(1-\frac{\iota_V\alpha_k}{2}\right)\E[\|y_k-\hat{z}_k\|_{V}]+\sqrt{\alpha_k}R_1\sqrt{\lambda^{max}_V}\E[\|x_k-x^*\|^{1+\delta}]\\
        &~+\frac{\varrho_3(\alpha)}{\sqrt{\lambda^{max}_V}}\alpha_k\sqrt{\E[\|x_0-x^*\|^2]}\left(\frac{K}{k+K}\right)^{\frac{3\gamma\alpha}{4}}+\frac{\varrho_5(\alpha, \gamma)}{\sqrt{\lambda^{max}_V}}\alpha_k^{3/2}.
    \end{align*}
    For the second term, we proceed in a similar fashion as in the previous case, to get
    \begin{align*}
        \E[\|x_k-x^*\|^{1+\delta}]\leq\varsigma_0^{\frac{1+\delta}{2}}\E[\|x_0-x^*\|^2]^{\frac{1+\delta}{2}}\left(\frac{K}{k+K}\right)^{\frac{3(1+\delta)\gamma\alpha}{4}}+\left(\frac{2\alpha\varsigma_2\alpha_k}{3\gamma\alpha-2}\right)^{\frac{1+\delta}{2}}
    \end{align*}
    Thus, we get
    \begin{align*}
        \sqrt{\alpha_k}\E[\|x_k-x^*\|^{1+\delta}]&\leq \left(\frac{\varsigma_0K^{\frac{3\gamma\alpha}{2}}}{\alpha^{\frac{3\gamma\alpha}{2}}}\right)^{\frac{1+\delta}{2}}\E[\|x_0-x^*\|^2]^{\frac{1+\delta}{2}}\alpha_k^{1+\frac{3(1+\delta)\gamma\alpha-2}{4}}\\
        &~+\left(\frac{2\alpha\varsigma_2}{3\gamma\alpha-2}\right)^{\frac{1+\delta}{2}}\alpha_k^{1+\delta/2}.
    \end{align*}
    Using these relations and the constants defined in Eq. \eqref{eq:varrho}, we get
        \begin{align*}
            \E[\|y_{k+1}-\hat{z}_{k+1}\|_{V}]
            &\leq \left(1-\frac{\iota_V\alpha_k}{2}\right)\E[\|y_k-\hat{z}_k\|_{V}]+\Bigg(\frac{\varrho_1}{4\sqrt{\lambda^{max}_V}}\E[\|x_0-x^*\|^2]^{\frac{1+\delta}{2}}\left(\frac{K}{\alpha}\right)^{\frac{3\delta\gamma\alpha}{4}}\\
            &~+\frac{\varrho_3(\alpha)}{\sqrt{\lambda^{max}_V}}\sqrt{\E[\|x_0-x^*\|^2]}\Bigg)\left(\frac{K}{\alpha}\right)^{\frac{3\gamma\alpha}{4}}\alpha_k^{1+\frac{3\gamma\alpha}{4}}\\
            &~+\left(\frac{3\varrho_2}{\sqrt{\lambda^{max}_V}}\left(\frac{\alpha}{3\gamma\alpha-2}\right)^{\frac{1+\delta}{2}}+\frac{\varrho_5(\alpha, \gamma)}{\sqrt{\lambda^{max}_V}}\right)\alpha_k^{1+\delta/2}.
        \end{align*}    
    where we used the fact that $3\delta\gamma\alpha>2$ and $\delta<1$.
    
    Define $\{v^{(1)}_k\}_{k\geq 0}$ and $\{v^{(2)}_k\}_{k\geq 0}$ as follows:
        \begin{align*}
            v^{(1)}_0=\frac{\E[\|y_0\|_V]}{2};&~~v^{(1)}_{k+1}=\left(1-\frac{\iota_V\alpha_k}{2}\right)v^{(1)}_k+a_1\alpha_k^{1+3\gamma\alpha/4},\\
            v^{(2)}_0=\frac{\E[\|y_0\|_V]}{2};&~~v^{(2)}_{k+1}=\left(1-\frac{\iota_V\alpha_k}{2}\right)v^{(2)}_k+a_2\alpha_k^{1+\delta/2}.
        \end{align*}
        for obvious choices of $a_1$ and $a_2$. Then, it is easy to verify that $\E[\|y_k-\hat{z}_k\|]\leq v^{(1)}_k+v^{(2)}_k$. Now it remains to bound these two terms.

        For $v_k^{(1)}$, we use Lemma \ref{lem:rec_sol} with $u_0=\E[\|y_0\|_V]$, $\mu_1=\iota_V/2$, $\mu_2=0$, $\mu_3=a_2$, and $\rho_2=3\gamma\alpha/4$, to get
         \begin{enumerate}
            \item[(a)] When $2\iota_V>3\gamma$, we have
            \begin{align*}
                v_k^{(1)}\leq v_0^{(1)}\left(\frac{K}{k+K}\right)^{\iota_V\alpha/2}+\frac{ 4a_1\alpha_k^{3\gamma\alpha/4}}{2\iota_V-3\gamma}.
            \end{align*}
            \item[(b)] When $2\iota_V<3\gamma$, we have
            \begin{align*}
                v_k^{(1)}\leq v_0^{(1)}\left(\frac{K}{k+K}\right)^{\iota_V\alpha/2}+\frac{2^{3+3\gamma\alpha/4}\alpha^{(3\gamma-2\iota_V)\alpha/4}a_1\alpha_k^{\iota_V\alpha/2}}{(3\gamma-2\iota_V)}.
            \end{align*}
            \item[(c)] When $2\iota_V=3\gamma$, we have
            \begin{align*}
                v_k^{(1)}\leq v_0^{(1)}\left(\frac{K}{k+K}\right)^{\iota_V\alpha/2}+2^{1+\iota_V\alpha/2}a_1\alpha\log(k+K)\alpha_k^{\iota_V\alpha/2}.
            \end{align*}
        \end{enumerate}

        For $v_k^{(2)}$, we use Lemma \ref{lem:rec_sol} with $u_0=\E[\|y_0\|_V]$, $\mu_1=\iota_V/2$, $\mu_2=0$, $\mu_3=a_2$, and $\rho_2=\delta/2$, to get
        \begin{align*}
            v_k^{(2)}\leq v_0^{(2)}\left(\frac{K}{k+K}\right)^{\iota_V\alpha/2}+\frac{2\alpha a_2\alpha_k^{\delta/2}}{\iota_V\alpha-2\delta}.
        \end{align*}

    For $\E[\|z_k-\hat{z}_k\|^2]$, we apply Jensen's inequality and Lemma \ref{lem:multiplicative_noise}. Combining all the bounds, we have the upper bound for the second case.
    \end{enumerate}
\end{proof}

\subsubsection{Discussion on the decay of the initial error}
Now, we will briefly discuss the competing transient effects of $\gamma$ and $\iota_V$ on the decay rate of the initial error in the upper bound. We will only discuss the case $\xi=0$ since all the other cases follows in an analogous manner. We will divide the discussion into three cases:
\begin{enumerate}
    \item $\iota_V>3\gamma/2:$ In this case, $\exp(-3\gamma\alpha k/2)>\exp(-\iota_V\alpha k)$. Thus, we obtain
    \begin{align*}
        \max\left(\frac{e^{-3\gamma \alpha k/4}-e^{-\iota_V\alpha k/2}}{2\iota_V-3\gamma}, \sqrt{\frac{e^{-3\gamma \alpha k/2}-e^{-\iota_V\alpha k}}{2\iota_V-3\gamma}}\right)=\mathcal{O}(e^{-3\gamma\alpha k/4}).
    \end{align*}
    \item $\iota_V<3\gamma/2:$ In this case, $\exp(-3\gamma\alpha k/2)<\exp(-\iota_V\alpha k)$. Thus, we obtain
    \begin{align*}
         \max\left(\frac{e^{-3\gamma \alpha k/4}-e^{-\iota_V\alpha k/2}}{2\iota_V-3\gamma}, \sqrt{\frac{e^{-3\gamma \alpha k/2}-e^{-\iota_V\alpha k}}{2\iota_V-3\gamma}}\right)=\mathcal{O}(e^{-\iota_V\alpha k/2}).
    \end{align*}
    \item $\iota_V=3\gamma/2:$ Taking the limit $\iota_V\to 3\gamma/2$ and using L'Hopital's rule, we get
    \begin{align*}
         \lim_{\iota_V\to 3\gamma/2}\max\left(\frac{e^{-3\gamma \alpha k/4}-e^{-\iota_V\alpha k/2}}{2\iota_V-3\gamma}, \sqrt{\frac{e^{-3\gamma \alpha k/2}-e^{-\iota_V\alpha k}}{2\iota_V-3\gamma}}\right)=\mathcal{O}(ke^{-\iota_V\alpha k/2}).
    \end{align*}
\end{enumerate}

\subsubsection{Proof of the intermediate results}\label{sec:proof_intermediate_results}
\begin{proof}[Proof of Lemma \ref{lem:iterate_bounds}] 
    Using smoothness of the Lyapunov function $\Phi(x)$ at $x_{k+1}$ and $x_k$, we get
    \begin{align*}
        \Phi(x_{k+1}-x^*)&\leq \Phi(x_k-x^*)+\langle x_{k+1}-x_k, \nabla \Phi(x_k-x^*) \rangle+\frac{L_s}{2}\|x_{k+1}-x_k\|_s^2\\
        &=\Phi(x_k-x^*)+\alpha_k\langle F(x_k)+M_k, \nabla \Phi(x_k-x^*) \rangle+\frac{L_s\alpha_k^2}{2}\|F(x_k)+M_k\|_s^2.
    \end{align*}
    Taking conditional expectation with respect to $\mathcal{F}_k$ and using the martingale difference property for $M_k$ in the cross term leads us to
    \begin{align*}
        \E[\Phi(x_{k+1}-x^*)|\mathcal{F}_k]&\leq \Phi(x_k-x^*)+\alpha_k\langle F(x_k), \nabla \Phi(x_k-x^*) \rangle\\
        &~+u_{2s}L_s\alpha_k^2(\|F(x_k)\|^2+\E[\|M_k\|^2|\mathcal{F}_k]),
    \end{align*}
    where for the last term we used norm equivalence and the fact that $(a+b)^2\leq 2a^2+2b^2$. Using the negative drift of the Lyapunov function, the Lipschitz property of the operator and the linear growth of the noise, we get
    \begin{align*}
        \E[\Phi(x_{k+1}-x^*)|\mathcal{F}_k]&\leq (1-2\gamma\alpha_k)\Phi(x_k-x^*)+u_{2s}L_s\alpha_k^2(L_F\|x_k-x^*\|^2+A_1\|x_k-x^*\|^2+B_2)\\
        &\leq (1-2\gamma\alpha_k)\Phi(x_k-x^*)+uu_{2s}L_s(L_F+A_1)\alpha_k^2\Phi(x_k-x^*)+u_{2s}L_s\alpha_k^2B_2.
    \end{align*}
    Recall that $\alpha_k$ is chosen such that $(2-\varsigma_1)\gamma\geq uu_{2s}L_s(L_F+A_1)\alpha_k$. Thus, after taking expectation, we obtain
    \begin{align*}
        \E[\Phi(x_{k+1}-x^*)]
        &\leq (1-\varsigma_1\gamma\alpha_k)\E[\Phi(x_k-x^*)]+u_{2s}L_sB_2\alpha_k^2.
    \end{align*}
    The claim follows by applying Lemma \ref{lem:rec_sol} for various choices of $\xi$ by setting $u_0=\E[\Phi(x_0-x^*)]$, $\mu_1=\varsigma_1\gamma$, $\mu_2=0$, $\mu_3=u_{2s}L_sB_2$, and $\rho_2=1$. The final bound follows by using the equivalence relation in Assumption \ref{assump:iterate}.
\end{proof}

\begin{proof}[Proof of Lemma \ref{lem:multiplicative_noise}]
    \begin{align*}
    z_{k+1}-\hat{z}_{k+1}=(I+\alpha_kJ^{(\alpha, \xi)})(z_k-\hat{z}_k)+\sqrt{\alpha_k}A(v_k, x_k).
    \end{align*}
    Squaring both sides w.r.t. $\|\cdot\|_V$ norm, we get
    \begin{align*}
        \|z_{k+1}-\hat{z}_{k+1}\|_V^2&=\|(I+\alpha_kJ^{(\alpha, \xi)})(z_k-\hat{z}_k)\|_V^2+\alpha_k\|A(v_k, x_k)\|_V^2\\
        &~~+2\sqrt{\alpha_k}\langle (I+\alpha_kJ^{(\alpha, \xi)})(z_k-\hat{z}_k), A(v_k, x_k)\rangle_V.
    \end{align*}
    Taking conditional expectation w.r.t. to $\mathcal{F}_{k-1}$, we note that the cross term goes to zero because of the martingale difference property. 
    \begin{enumerate}
        \item $\xi=0:$ We use statement 1 of Lemma \ref{lem:contraction_prop}  for the first term and get
        \begin{align*}
            \E[\|z_{k+1}-\hat{z}_{k+1}\|_V^2|\mathcal{F}_{k-1}]&\leq (1-\iota_V\alpha)\|z_k-\hat{z}_k\|_V^2+\alpha\E[\|A(v_k, x_k)\|_V^2|\mathcal{F}_k]\\
            &\leq (1-\iota_V\alpha)\|z_k-\hat{z}_k\|_V^2+\alpha A_1\lambda^{max}_V\|x_k-x^*\|^2.
        \end{align*}
        Taking expectation, and using statement 1 of Lemma \ref{lem:iterate_bounds}, we get
        \begin{align*}
            \E[\|z_{k+1}-\hat{z}_{k+1}\|_V^2]&\leq (1-\iota_V\alpha)\E[\|z_k-\hat{z}_k\|_V^2]\\
            &~+\alpha A_1\lambda^{max}_V\left(\varsigma_0\E[\|x_0-x^*\|^2]e^{-3\gamma\alpha k/2}+\frac{2\varsigma_2\alpha}{3\gamma}\right).
        \end{align*}
        Applying Lemma \ref{lem:rec_sol} with $u_0=0$, $\mu_1=\iota_V$, $\mu_2=A_1\varsigma_0\lambda^{max}_V\E[\|x_0-x^*\|^2]$, $\mu_3=2A_1\varsigma_2\lambda^{max}_V/(3\gamma)$, $\rho_1=3\gamma/2$, and $\rho_2=1$, we obtain the final bound. 
        \item $\xi\in (0,1):$ We follow the same steps as in the previous case but now use statement 2 of Lemma \ref{lem:iterate_bounds} to get
        \begin{align*}
            &\E[\|z_{k+1}-\hat{z}_{k+1}\|_V^2]\leq (1-\iota_V\alpha_k)\E[\|z_k-\hat{z}_k\|_V^2]\\
            &~~+\alpha_k A_1\lambda^{max}_V\left(\varsigma_0\E[\|x_0-x^*\|^2]\exp\left(-\frac{3\gamma\alpha}{2(1-\xi)}\left((k+K)^{1-\xi}-K^{1-\xi}\right)\right)+\frac{4\varsigma_2\alpha_k}{3\gamma}\right).
        \end{align*}
        Applying Lemma \ref{lem:rec_sol} with $u_0=0$, $\mu_1=\iota_V$, $\mu_2=A_1\varsigma_0\lambda^{max}_V\E[\|x_0-x^*\|^2]$, $\mu_3=4A_1\varsigma_2\lambda^{max}_V/(3\gamma)$, $\rho_1=3\gamma/2$, and $\rho_2=1$, we obtain the final bound. 
        \item $\xi=1:$ For this case, we use statement 2 of Lemma \ref{lem:contraction_prop}  and statement 3 of Lemma \ref{lem:iterate_bounds} to get
        \begin{align*}
            &\E[\|z_{k+1}-\hat{z}_{k+1}\|_V^2]\leq (1-\iota_V\alpha_k)\E[\|z_k-\hat{z}_k\|_V^2]\\
            &~~+\alpha_k A_1\lambda^{max}_V\left(\varsigma_0\E[\|x_0-x^*\|^2]\left(\frac{K}{k+K}\right)^{3\gamma\alpha/2}+\frac{2\alpha\varsigma_2\alpha_k}{3\gamma\alpha-2}\right).
        \end{align*}
        Define $\{v^{(1)}_k\}_{k\geq 0}$ and $\{v^{(2)}_k\}_{k\geq 0}$ as follows:
        \begin{align*}
            v^{(1)}_0=0;&~~v^{(1)}_{k+1}=(1-\iota_V\alpha_k)v^{(1)}_k+\frac{A_1\lambda^{max}_V\varsigma_0}{\alpha^{3\gamma/2}}\E[\|x_0-x^*\|^2]K^{3\gamma\alpha/2}\alpha_k^{1+3\gamma\alpha/2},\\
            v^{(2)}_0=0;&~~v^{(2)}_{k+1}=(1-\iota_V\alpha_k)v^{(2)}_k+A_1\lambda^{max}_V\frac{2\alpha\varsigma_2\alpha_k^2}{3\gamma\alpha-2}.
        \end{align*}
        Then, it is easy to verify that $\E[\|z_k-\hat{z}_k\|^2]\leq v^{(1)}_k+v^{(2)}_k$. Now it remains to bound these two terms.
        
        For $v^{(1)}_k$, we use Lemma \ref{lem:rec_sol} with $u_0=0$, $\mu_1=\iota_V$, $\mu_2=0$, $\mu_3=A_1\varsigma_0\lambda^{max}_V\E[\|x_0-x^*\|^2]K^{3\gamma\alpha/2}/\alpha^{3\gamma/2}$, and $\rho_2=3\gamma/2$ to have the following three cases:
        \begin{enumerate}
            \item[(a)] $2\iota_V>3\gamma$: 
            \begin{align*}
                v_k^{(1)}\leq \frac{ 2A_1\varsigma_0\lambda^{max}_V\E[\|x_0-x^*\|^2]K^{3\gamma\alpha/2}\alpha_k^{3\gamma\alpha/2}}{\alpha^{3\gamma/2}(2\iota_V-3\gamma)}
            \end{align*}
            \item[(b)] $2\iota_V<3\gamma$:
            \begin{align*}
                v_k^{(1)}\leq \frac{4A_1\varsigma_0\lambda^{max}_V\E[\|x_0-x^*\|^2](2K)^{3\gamma\alpha/2}\alpha_k^{\iota_V\alpha}}{\alpha^{\iota_V\alpha}(3\gamma-2\iota_V)}.
            \end{align*}
            \item[(c)] $2\iota_V=3\gamma$:
            \begin{align*}
                v_k^{(1)}\leq 2A_1\varsigma_0\lambda^{max}_V\E[\|x_0-x^*\|^2](2K)^{\iota_V\alpha}\alpha^{1-\iota_V\alpha}\log(k+K)\alpha_k^{\iota_V\alpha}.
            \end{align*}
        \end{enumerate}
        For $v_k^{(2)}$, we again use Lemma \ref{lem:rec_sol} with $u_0=0$, $\mu_1=\iota_V$, $\mu_2=0$, $\mu_3=A_1\lambda^{max}_V2\alpha\varsigma_2/(3\gamma\alpha-2)$, and $\rho_2=1$ to have
        \begin{align*}
            v_k^{(2)}\leq \frac{ 2A_1\lambda^{max}_V\alpha^2\varsigma_2\alpha_k}{(3\gamma\alpha-2)(\iota_V\alpha-1)}.
        \end{align*}
        The lemma follows by combining the bounds for the different cases.
    \end{enumerate}
\end{proof}

\begin{proof}[Proof of Proposition \ref{prop:global_linear}]
    Let $x\in \mathbb{R}^d$ be any arbitrary vector. For $x\in \mathbb{B}_r(x^*)$, linear approximation already holds due the assumption. For $\|x-x^*\|>r$, we have
    \begin{align*}
        R(x)=F(x)-J_F(x-x^*).
    \end{align*}
    Using the Lipschitz property of $F$ and the fact that $F(x^*)=0$, we get
    \begin{align*}
        \|R(x)\|&=\|F(x)-F(x^*)-J_F(x-x^*)\|\\
        &\leq L_F\|x-x^*\|+\|J_F\|\|x-x^*\|\\
        &= (L_F+\|J_F\|)\|x-x^*\|\\
        &\leq \frac{L_F+\|J_F\|}{r^\delta}\|x-x^*\|^{1+\delta}.\tag{$\|x-x^*\|/r>1$}\\
    \end{align*}
    Thus, combining the upper bounds for $R(x)$ from the above analysis and Assumption \ref{assump:operator}, we get the claim.
\end{proof}

\subsection{Proof of Corollary \ref{cor:asymp_gaussian_rate}}\label{sec:pf_asymp_gaussian_rate}

\begin{corollary}
   The sequence $\{y_k\}_{k\geq 0}$ given by rescaling the iterate $x_k$ in update equation \eqref{eq:SA_rec2} satisfies the following bounds for all $k\geq 0$.
    \begin{enumerate}
        \item When $\xi=0$, and $\alpha$ is small enough, we have{\normalfont :}
        \begin{align*}
            d_{\mathcal{W}}(y_k, \Sigma^{1/2}Z)&\leq d_{\mathcal{W}}(y_k, (\Sigma_k^{(\alpha, 0)})^{1/2}Z)+ \sqrt{\frac{d\lambda^{max}_V}{\lambda^{min}_V\lambda_{min}(\Sigma^{(\alpha, \xi)})}}\left(\|\Sigma\|_Ve^{-\iota_V\alpha k}+\frac{\|J_F\|_V^2\|\Sigma\|_V\alpha}{\iota_V}\right).
        \end{align*}
        \item When $\xi\in (0, 1)$ and $K$ is large enough, we have{\normalfont :}
        \begin{align*}
            d_{\mathcal{W}}(y_k, \Sigma^{1/2}Z)&\leq 
            d_{\mathcal{W}}(y_k, (\Sigma_k^{(\alpha, \xi)})^{1/2}Z)+\sqrt{\frac{d\lambda^{max}_V}{\lambda^{min}_V\lambda_{min}(\Sigma^{(\alpha, \xi)})}}\Bigg(\|\Sigma\|_Ve^{-\frac{\iota_V\alpha}{1-\xi}\left((k+K)^{1-\xi}-K^{1-\xi}\right)}\\
            &~+\frac{2\left(\|J_F\|_V+\frac{\alpha^{-1}}{2}\right)^2\|\Sigma\|_V\alpha_k}{\iota_V}+\frac{2\alpha^{-\xi}\|\Sigma\|_V}{(k+K)^{1-\xi}\iota_V}\Bigg).
        \end{align*}
        \item When $\xi=1$, $\iota_V\alpha>2$, $3\delta\gamma\alpha>2$, and $K$ is large enough, we have {\normalfont :}
        \begin{align*}
            d_{\mathcal{W}}(y_{k}, (\Sigma^{(\alpha)})^{1/2}Z)&\leq d_{\mathcal{W}}(y_k, (\Sigma_k^{(\alpha, 1)})^{1/2}Z)+\sqrt{\frac{d\lambda^{max}_V}{\lambda^{min}_V\lambda_{min}(\Sigma^{(\alpha, \xi)})}}\Bigg(\|\Sigma^{(\alpha)}\|_V\left(\frac{K}{k+K}\right)^{\iota_V\alpha}\\
            &~+\frac{\alpha\|\left(\|J_F\|_V+\frac{\alpha^{-1}}{2}\right)^2\|\Sigma^{(\alpha)}\|_V\alpha_k}{\iota_V\alpha-1}\Bigg).
        \end{align*}
    \end{enumerate}
\end{corollary}
\begin{proof}
    Using triangle inequality, we obtain
    \begin{align*}
        d_{\mathcal{W}}(y_k, (\Sigma^{(\alpha, \xi)})^{1/2}Z)\leq d_{\mathcal{W}}(y_k, (\Sigma_k^{(\alpha, \xi)})^{1/2}Z)+d_{\mathcal{W}}((\Sigma_k^{(\alpha, \xi)})^{1/2}Z, (\Sigma^{(\alpha, \xi)})^{1/2}Z).
    \end{align*}
    Using Lemma \ref{lem:wass-2} with $\Sigma_k^{(\alpha, \xi)}$ and $\Sigma^{(\alpha, \xi)}$ to obtain the following bound on the Wasserstein-2 distance
    \begin{align*}
        d_{\mathcal{W}_2}((\Sigma_k^{(\alpha, \xi)})^{1/2}Z, (\Sigma^{(\alpha, \xi)})^{1/2}Z)&\leq \frac{\sqrt{d}}{\sqrt{\lambda_{min}(\Sigma_k^{(\alpha, \xi)})}+\sqrt{\lambda_{min}(\Sigma^{(\alpha, \xi)})}}\|\Sigma_k^{(\alpha, \xi)}-\Sigma^{(\alpha, \xi)}\|\\
        &\leq \frac{\sqrt{d}}{\sqrt{\lambda_{min}(\Sigma^{(\alpha, \xi)})}}\|\Sigma_k^{(\alpha, \xi)}-\Sigma^{(\alpha, \xi)}\|.\tag{$\Sigma_k^{(\alpha, \xi)}$ is a positive definite matrix}
    \end{align*}
    Now, we apply Lemma \ref{lem:matrix_theta_bound} to obtain the claim for all cases.
\end{proof}

\subsection{Proof of Corollary \ref{cor:sigmak_and_hatsigmak}}\label{sec:sigmak_and_hatsigmak}
Before beginning the proof of the corollary, we define some constants which will be used later. Denote $\Psi_k^{(\alpha, \xi)}$ as the solution to the following Lyapunov equation:
\begin{align*}
    J_{k+1}^{(\alpha, \xi)}\Psi_k^{(\alpha, \xi)}+\Psi_k^{(\alpha, \xi)}(J_{k+1}^{(\alpha, \xi)})^T+\hat{\Sigma}^{(\alpha, \xi)}_k=0.
\end{align*}
Then, we define the following
\begin{align}\label{eq:const_time_varying_lyap}
    \sigma_{min}^{(\alpha, \xi)} = \min_k\|\Sigma_k^{(\alpha, \xi)}\|;~~\hat{\sigma}_{max}^{(\alpha, \xi)} = \max_k\|\hat{\Sigma}_k^{(\alpha, \xi)}\|;~~\psi_{max}^{(\alpha, \xi)}=\max_{k}\|\Psi_k^{(\alpha, \xi)}\|
\end{align}

\begin{corollary}
    Let $\xi\in (0, 1)$. The sequence $\{y_k\}_{k\geq 0}$ given by rescaling the iterate $x_k$ in update equation \eqref{eq:SA_rec2} satisfies the following bounds for all $k\geq 0$.
    \begin{align*}
        d_{\mathcal{W}}(y_k, (\hat{\Sigma}^{(\alpha, \xi)}_k)^{1/2}Z)&\leq 
        d_{\mathcal{W}}(y_k, (\Sigma_k^{(\alpha, \xi)})^{1/2}Z)+\sqrt{\frac{d\lambda^{max}_V}{\lambda^{min}_V\sigma_{min}^{(\alpha, \xi)}}}\Bigg(\|\tilde{\Sigma}_0^{(\alpha, \xi)}-\hat{\Sigma}_0^{(\alpha, \xi)})\|_Ve^{-\frac{\iota_V\alpha}{1-\xi}\left((k+K)^{1-\xi}-K^{1-\xi}\right)}\\
        &~+\sqrt{\frac{\lambda^{max}_V}{\lambda^{min}_V}}\frac{2\left(\|J_F\|_V+\frac{\alpha^{-1}}{2}\right)^2\hat{\sigma}_{max}^{(\alpha, \xi)}\alpha_k}{\iota_V}+\sqrt{\frac{\lambda^{max}_V}{\lambda^{min}_V}}\frac{2\alpha^{-2}\xi(1-\xi)\psi_{max}^{(\alpha, \xi)}}{(k+K)^{2-2\xi}\iota_V}\Bigg).
    \end{align*}
\end{corollary}
\begin{proof}
    Using triangle inequality, we obtain
    \begin{align*}
        d_{\mathcal{W}}(y_k, (\hat{\Sigma}^{(\alpha, \xi)}_k)^{1/2}Z)\leq d_{\mathcal{W}}(y_k, (\Sigma_k^{(\alpha, \xi)})^{1/2}Z)+d_{\mathcal{W}}((\Sigma_k^{(\alpha, \xi)})^{1/2}Z, (\hat{\Sigma}^{(\alpha, \xi)}_k)^{1/2}Z).
    \end{align*}
    Using Lemma \ref{lem:wass-2} with $\Sigma_k^{(\alpha, \xi)}$ and $\hat{\Sigma}^{(\alpha, \xi)}_k$ to obtain the following bound on the Wasserstein-2 distance
    \begin{align*}
        d_{\mathcal{W}_2}((\Sigma_k^{(\alpha, \xi)})^{1/2}Z, (\Sigma^{(\alpha, \xi)})^{1/2}Z)&\leq \frac{\sqrt{d}}{\sqrt{\lambda_{min}(\Sigma_k^{(\alpha, \xi)})}+\sqrt{\lambda_{min}(\hat{\Sigma}^{(\alpha, \xi)}_k})}\|\Sigma_k^{(\alpha, \xi)}-\hat{\Sigma}^{(\alpha, \xi)}_k\|\\
        &\leq \frac{\sqrt{d}}{\sqrt{\lambda_{min}(\Sigma_k^{(\alpha, \xi)})}}\|\Sigma_k^{(\alpha, \xi)}-\hat{\Sigma}^{(\alpha, \xi)}_k\|\tag{$\hat{\Sigma}_k^{(\alpha, \xi)}$ is a positive definite matrix}\\
        &\leq \sqrt{\frac{d}{\sigma_{min}^{(\alpha, \xi)}}}\|\Sigma_k^{(\alpha, \xi)}-\hat{\Sigma}^{(\alpha, \xi)}_k\|.
    \end{align*}
    Now, we apply Lemma \ref{lem:sigmak_and_hatsigmak} to obtain the claim.
\end{proof}

\subsection{Proof of Proposition \ref{prop:clt}}\label{sec:clt_proof}
\begin{proposition}
    The sequence $\{y_k\}_{k\geq 0}$ given by rescaling the iterate $x_k$ in update equation \eqref{eq:clt_weighted} satisfies the following bounds for all $k\geq 0$.
    \begin{enumerate}
        \item When $\xi=0$, and $\alpha$ is small enough, we have{\normalfont :}
        \begin{align*}
            d_{\mathcal{W}}\left(y_k, \frac{\Sigma_b^{1/2}}{2}Z\right)&\leq \mathcal{O}\left(\sqrt{\alpha}\log\left(\frac{1}{\alpha}\right)\right).
        \end{align*}
        \item When $\xi\in (0, 1)$ and $K$ is large enough, we have{\normalfont :}
        \begin{align*}
            d_{\mathcal{W}}\left(y_k, \frac{\Sigma_b^{1/2}}{2}Z\right)&\leq 
            \mathcal{O}\left(\frac{\log(k+K)}{(k+K)^{\frac{\xi}{2}}}\right)
        \end{align*}
        \item When $\xi=1$, $\alpha\geq 1$ and $K\geq 2\alpha-1$, we have {\normalfont :}
        \begin{align*}
            d_{\mathcal{W}}(y_{k}, \Sigma_b^{1/2}Z)\leq \mathcal{O}\left(\frac{\log(k+K)}{\sqrt{k+K}}\right).
        \end{align*}
    \end{enumerate}
\end{proposition}
\begin{proof}
    We only present the proof for the case of $\xi=1$. The modifications are exactly the same for other choices of $\xi$. Furthermore, the proof in this special setting closely mirrors the general Hurwitz case. Therefore, we will use the $\mathcal{O}(\cdot)$ notation for analysis and focus our derivation only on the steps where the symmetric nature of the matrix is exploited.  

    The SA recursion for the CLT setting is given by
    \begin{align*}
        x_{k+1}&=x_k+\frac{\alpha}{k+K}(-x_k+b_k).
    \end{align*}
    where $\{b_k\}_{k\geq 0}$ is a zero i.i.d. sequence of noise with bounded third moment. In this case, we can substitute $\Phi(x)=\|x\|^2/2$ to get $\gamma=1$ and use Lemma \ref{lem:iterate_bounds} to obtain $\E[\|x_k\|^2]\leq \mathcal{O}(1/(k+K))$. Next, the rescaled iterates $y_k=x_k\sqrt{k+K}$ are given as
    \begin{align*}
        y_{k+1}&=\left(1-\frac{2\alpha-1}{2(k+1)}\right)y_k+\frac{\sqrt{\alpha}}{\sqrt{k+K}}b_k+\underbrace{y_k\left(\frac{\sqrt{k+K+1}}{\sqrt{k+K}}-1-\frac{1}{2(k+K)}\right)}_{T_1}\\
                &~+\underbrace{\frac{\sqrt{\alpha}}{k+K}(-x_k+b_k)\left(\sqrt{k+K}-\sqrt{k+K}\right)}_{T_2}
    \end{align*}
    Using a straightforward analysis, we obtain $\E[\|T_1\|]\leq \mathcal{O}(1/(k+K)^2)$ and $\E[\|T_2\|]\leq \mathcal{O}(1/(k+K)^{3/2})$.
    The DOUG iterates in this instance is 
    \begin{align*}
        z_{k+1}=\left(1-\frac{2\alpha-1}{2(k+K)}\right)z_k+\frac{\sqrt{\alpha}}{\sqrt{k+K}}b_k.
    \end{align*}
    Thus, the error dynamics are
    \begin{align*}
         y_{k+1}-z_{k+1}&=\left(1-\frac{2\alpha-1}{2(k+1)}\right)(y_k-z_k)+T_1+T_2.
    \end{align*}
    Taking norm and then expectation both sides, we get
    \begin{align*}
        \E[\|y_{k+1}-z_{k+1}\|]&\leq \left(1-\frac{2\alpha-1}{2(k+K)}\right)\E[\|y_k-z_k\|]+\mathcal{O}\left(\frac{1}{(k+K)^{3/2}}\right).
    \end{align*}
    Recall that our aim is to cover typical time-averaging $\sum_{i=0}^{k-1}b(w_k)/k$ as a special case of step-size based averaging. Thus, we need to consider $\alpha=1$ as a valid choice for $\alpha$. Therefore, we apply a combination of statement 3(a) and 3(c) of Lemma \ref{lem:rec_sol} with $u_0=\E[y_0]$, $\mu_1=(2\alpha-1)/(2\alpha)$, $\mu_2=0$ and $\rho_2=1/2$ to get
    \begin{align}\label{eq:err_clt}
        \E[\|y_k-z_k\|]\leq \frac{\E[y_0]}{\sqrt{k+K}}+\mathcal{O}\left(\frac{\log(k+K)}{\sqrt{k+K}}\right).
    \end{align}
    
    Now, we turn our attention to the distributional convergence of the corresponding DOUG process. The covariance matrix for $z_k$ satisfies
    \begin{align*}
        \Sigma^{(\alpha, 1)}_{k+1}=\left(1-\frac{2\alpha-1}{2(k+1)}\right)^2\Sigma^{(\alpha, 1)}_k+\alpha_k\Sigma_b.
    \end{align*}
    Thus, the Stein's operator is given by
    \begin{align*}
        \mathcal{L}_k f(x) : =   -\langle x, \nabla f(x) \rangle + 
     \mathrm{Tr}(\Sigma^{(\alpha, 1)}_k \nabla^2f(x)).
    \end{align*}
    Following the exact decomposition as in the proof of Theorem \ref{thm_main:weighted_stein}, we get
    \begin{align*}
             |\E[\mathcal{L}_k f(z_k)]|&\leq \underbrace{\left|
         \E \left[ \sum_{i=0}^{k-1} \langle \Theta_i b(w_i),  (\nabla^2 f(Uz_k+(1-U)\zeta_i)-\nabla^2 f(\zeta_i)) \Theta_i b(w_i)\rangle  \right] \right|}_{\rm (i)} \\
        &\quad +\underbrace{\left| \E \left[ \sum_{i=0}^{k-1} \langle  \Theta_i b(w_i), \nabla^2 f(\zeta_i) \Theta_i b(w_i)\rangle  - \mathrm{Tr}(\Sigma^{(\alpha, 1)}_k\nabla^2f(z_k))\right] \right|}_{\rm (ii)}.
    \end{align*}
    Let $\kappa^{(\alpha, \xi)}=\max_{k\geq 1}\max_{\beta\in [0,1]}\|(\Sigma_k^{(\alpha, \xi)})^{\frac{1}{2}}\|\|(\Sigma_k^{(\alpha, \xi)})^{-\frac{1}{2}}\|^{2+\beta}$, where the maximum over $k$ is well-defined due to Lemma \ref{lem:bound_sigma_k}, Lemma \ref{lem:inv_bound}, and Lemma \ref{lem:matrix_theta_bound}. For $\rm(i)$, we use Cauchy-Schwartz, Lemma \ref{lem:derivative_bound} and follow identical steps as in the proof of Theorem \ref{thm_main:weighted_stein}, to get
    \begin{align*}
        \mathrm{(i)}\leq C_1(d, \beta)\kappa^{(\alpha, \xi)}B_{2+\beta} \sum_{i=0}^{k-1} \| \Theta_i\|^{2+\beta}.
    \end{align*}
    Now, we use Lemma \ref{lem:theta_bound_symm} to get
    \begin{align*}
        \sum_{i=0}^{k-1} \| \Theta_i\|^{2+\beta}\leq \frac{4\alpha^{1+\beta/2}}{(k+K)^{\beta/2}(6\alpha-5)}\tag{$\iota_J=1$}
    \end{align*}
    which leads us to 
    \begin{align*}
         \mathrm{(i)}\leq  \frac{4\alpha^{1+\beta/2}C_1(d, \beta)\kappa^{(\alpha, \xi)}B_{2+\beta}}{(k+K)^{\beta/2}(6\alpha-5)}.
    \end{align*}
    For $\rm(ii)$ also, we use the following lemma which is similar to Lemma \ref{lemma:vec-im-results} to upper bound.
    \begin{lemma}\label{lem:vec-im-results_clt}
    The following relation hold for all $k\geq 0$:
        \begin{align*}
            \Bigg| \E \Bigg[ \mathrm{Tr}\left(  \sum_{i=0}^{k-1} \nabla^2 f(\zeta_i) \Theta_i \Sigma_b \Theta_i^T \right) - \mathrm{Tr}\left(\Sigma^{(\alpha, 1)}_k\nabla^2 f(z_k) \right)\Bigg]\Bigg|\leq 2\alpha\alpha_k^{\beta/2}dC_1(d, \beta)\kappa^{(\alpha, \xi)}B_2B_{max}.
        \end{align*}
    \end{lemma}
    \begin{align*}
        \mathrm{(ii)}& \leq 2\alpha\alpha_k^{\beta/2}dC_1(d, \beta)\kappa^{(\alpha, \xi)}B_2B_{max}.
    \end{align*}
    Setting $\beta$ as $1 + 2/\log(\alpha_k)$ leads to
    \begin{align*}
        \alpha_k^{\beta/2}=\sqrt{\alpha_k}\alpha_k^{1/\log(\alpha_k)}=e\sqrt{\alpha_k}.
    \end{align*}
    Furthermore, $C_1(d, \beta)=(\tilde{C}_1(d)+\log(k+1))$. 
    Combining all the relations, we obtain
    \begin{align*}
        d_{\mathcal{W}}(z_k, (\Sigma^{(\alpha, 1)}_k)^{1/2}Z)\leq \mathcal{O}\left(\frac{\log{(k+1)}}{\sqrt{k+1}}\right).
    \end{align*}
    Using Eq. \eqref{eq:err_clt} for the error dynamics $\E[\|y_k-z_k\|]$ and the standard coupling argument, we obtain
    \begin{align*}
        d_{\mathcal{W}}(y_k, (\Sigma^{(\alpha, 1)}_k)^{1/2}Z)\leq \mathcal{O}\left(\frac{\log{(k+1)}}{\sqrt{k+1}}\right).
    \end{align*}

     To obtain the rate of convergence to the asymptotic distribution we use triangle inequality and Lemma \ref{lem:wass-2} to get
    \begin{align*}
        d_{\mathcal{W}}(y_k, \Sigma_b^{1/2}Z)&\leq d_{\mathcal{W}}(y_k, (\Sigma^{(\alpha, 1)}_k)^{1/2}Z)+d_{\mathcal{W}}(\Sigma_b^{1/2}Z, (\Sigma^{(\alpha, 1)}_k)^{1/2}Z) \\
        &\leq d_{\mathcal{W}}(y_k, (\Sigma^{(\alpha, 1)}_k)^{1/2}Z)+\frac{\sqrt{d}}{\sqrt{\lambda_{min}(\Sigma_b)}}\|\Sigma_k^{(\alpha, 1)}-\Sigma_b\|.
    \end{align*}
        Note that now we could apply Lemma \ref{lem:matrix_theta_bound_symm} to bound the last term. However, since we are specifically given $\iota_J=1$ and $\alpha\geq 1$, the condition $\frac{3}{2}\left(\iota_J-\frac{\alpha^{-1}}{2}\right)\alpha> 1$ is not satisfied for all choices of $\alpha$. Thus, we proceed with a more delicate analysis which will eventually give us $\mathcal{O}(1/(k+K)^{2/3})$ rate of convergence instead of faster $\mathcal{O}(1/(k+K))$ in Lemma \ref{lemma:vec-im-results}. Nevertheless, this loose bound is sufficient for us as we know that the dominant term for CLT convergence is $\mathcal{O}(1/\sqrt{k+K})$.

        Using same steps as in the proof of Lemma \ref{lem:matrix_theta_bound_symm}, we get
        \begin{align*}
            \Sigma_{k+1}^{(\alpha, 1)}-\Sigma_b&=\left(1 - \frac{2\alpha-1}{2(k+K)}\right)^2(\Sigma_{k}^{(\alpha, 1)}-\Sigma_b)+\frac{\alpha^2}{(k+K)^2}\Sigma_b.
        \end{align*}
        Taking matrix norm both sides, we get
        \begin{align*}
            \|\Sigma_{k+1}^{(\alpha, 1)}-\Sigma_b\|&\leq \left(1 - \frac{2\alpha-1}{k+K}+\frac{(2\alpha-1)^2}{4(k+K)^2}\right)\|\Sigma_{k}^{{(\alpha, 1)}}-\Sigma_b\|+\frac{\alpha^2}{(k+K)^2}\|\Sigma_b\|\\
            &\leq \left(1 - \frac{3(2\alpha-1)}{4(k+K)}\right)\|\Sigma_{k}^{(\alpha, 1)}-\Sigma_b\|+\frac{\alpha^{5/3}}{(k+K)^{5/3}}B_2\tag{$\|\Sigma_b\|\leq \textrm{Tr}(\Sigma_b)=\E[\|b(w_1)\|^2]$}.
        \end{align*}
        where for the last inequality we used the fact that $K\geq 2\alpha-1$ and $\alpha_k\leq 1$. Note that $3/4>2/3$, thus we can apply statement 3(a) of Lemma \ref{lem:rec_sol} with $u_0=B_2$, $\mu_1=3(2\alpha-1)/(4\alpha)$, $\mu_2=0$, $\mu_3=B_2$, and $\rho_2=2/3$, to get
        \begin{align*}
            \|\Sigma_{k}^{(\alpha, 1)}-\Sigma_b\|\leq B_2\left(\left(\frac{K}{k+K}\right)^{\frac{3(2\alpha-1)}{4}}+\frac{12\alpha^{5/3}}{(k+K)^{2/3}(18\alpha-17)}\right).
        \end{align*}
        Combining all the bounds, we have the claim.
\end{proof}

\begin{proof}[Proof of Lemma \ref{lem:vec-im-results_clt}]
        We follow the exact steps as in the proof for the second statement of Lemma \ref{lemma:vec-im-results}, to get
        \begin{align*}
            \Bigg| \E \Bigg[ \mathrm{Tr}\left(  \sum_{i=0}^{k-1} \nabla^2 f(\zeta_i) \Theta_i \Sigma_b \Theta_i^T \right) &- \mathrm{Tr}\left(  \sum_{i=0}^{k-1} \nabla^2 f(z_k) \Theta_i \Sigma_b \Theta_i^T \right)\Bigg]\Bigg|\\
            &\leq dC_1(d, \beta)\kappa^{(\alpha, \xi)}B_2B_{max} \left(  \sum_{i=0}^{k-1} \|\Theta_i\|^{2+\beta}\right).
        \end{align*}
        Now, we use Lemma \ref{lem:theta_bound_symm} to get
        \begin{align*}
            \Bigg| \E \Bigg[ \mathrm{Tr}\left(  \sum_{i=0}^{k-1} \nabla^2 f(\zeta_i) \Theta_i \Sigma_b \Theta_i^T \right) &- \mathrm{Tr}\left(  \sum_{i=0}^{k-1} \nabla^2 f(z_k) \Theta_i \Sigma_b \Theta_i^T \right)\Bigg]\Bigg|\leq \frac{2\alpha^{1+\beta/2}dC_1(d, \beta)\kappa^{(\alpha, \xi)}B_2B_{max}}{(k+K)^{\beta/2}(6\alpha-5)}. 
        \end{align*}
\end{proof}

\subsection{Lower bound for SA}\label{sec:lower_boundSA}
\textbf{Time-varying Gaussian:} Let $\{b_k\}_{k\geq 0}$ be a sequence of zero mean i.i.d. random variables in $\mathbb{R}^d$ such that $\E[\|b_i\|^4]<\infty$. For simplicity, we will assume that $\E[b_ib_i^T]=I$. Consider the following recursion:
\begin{align*}
    x_{k+1} = \left(1 - \alpha_k\right) x_k + \alpha_k b_k,
\end{align*}
where $x_0$ is arbitrary, $\alpha_k=\alpha/(k+K)^{\xi}\leq 1$, and $\xi\in [0, 1)$. Let $\Sigma_k$ denote the covariance of $z_k$ and let $Y_k\sim \mathcal{N}(0, \Sigma_k)$. Let $y_k=x_k/\sqrt{\alpha_k}$. Then, the rescaled iterate $y_k$ is given by
\begin{align*}
    y_{k+1}=(1+\alpha_kJ_k^{(\alpha, \xi)})y_k+\sqrt{\alpha_k}b_k+\underbrace{y_k\left(\frac{\sqrt{\alpha_k}}{\sqrt{\alpha_{k+1}}}-1-\frac{\xi}{2(k+K)}\right)}_{T_1}+\underbrace{\alpha_k(-x_k+b_k)\left(\frac{1}{\sqrt{\alpha_{k+1}}}-\frac{1}{\sqrt{\alpha_k}}\right)}_{T_2}.
\end{align*}
where 
\begin{align*}
    J_k^{(\alpha, \xi)} = -1+\frac{\alpha^{-1}\xi}{2(k+K)^{1-\xi}}.
\end{align*}
We assume that $K$ is chosen large enough such that $J_k^{(\alpha, \xi)}\leq 1/2$ for $\xi\in (0,1)$. Using reverse triangle inequality for Wasserstein-1 distance $d_{\mathcal{W}}(y_k, Y_k)$, we get
\begin{align*}
    d_{\mathcal{W}}(y_k, Y_k)\geq d_{\mathcal{W}}(z_k, Y_k)-d_{\mathcal{W}}(y_k, z_k)
\end{align*}
From Proposition \ref{lem:lower_bound}, we know that the first term is lower bounded by $\Omega(\sqrt{\alpha_k})$. For the second term, we will use the same coupling trick as in the proof of Theorem \ref{thm_main:main_thm}. Specifically, using the coupling such that $\{y_k\}_{k\geq 0}$ and $\{z_k\}_{k\geq 0}$ share the same underlying data stream $\{b_k\}_{k\geq 0}$, we get
\begin{align*}
    d_{\mathcal{W}}(y_k, z_k)\leq \E[\|y_k-z_k\|].
\end{align*}
Now, our goal is to show that $d_{\mathcal{W}}(y_k, z_k)$ is a higher order term and does not affect the lower bound significantly. To this end, we will divide the analysis into two cases $\xi=0$ and $\xi\in (0,1)$. 
\begin{enumerate}
    \item $\xi=0:$ In this case, both $T_1$ and $T_2$ are $0$ since $\alpha_k=\alpha$ for all $k\geq 0$. Thus, the iteration is simplified to
    \begin{align*}
        y_{k+1}&=(1+\alpha_kJ_F)y_k+\sqrt{\alpha_k}b_k\\
        &=(1-\alpha_k)y_k+\sqrt{\alpha_k}b_k.
    \end{align*}
    By replicating the proof of part (1) of Theorem \ref{thm_main:main_thm} for this special case, we get
    \begin{align*}
        \E[\|y_k-z_k\|]\leq \E[\|y_0\|]e^{-\iota_V\alpha k/2}.
    \end{align*}
    Thus, we have
    \begin{align*}
        d_{\mathcal{W}}(y_k, Y_k)&\geq \Omega(\sqrt{\alpha})-e^{-\iota_V\alpha k/2}=\Omega(\sqrt{\alpha}).
    \end{align*}

    \item $\xi\in (0, 1):$ Using Lemma \ref{lem:step-size_prop} for $T_1$, we have
    \begin{align*}
        \left|\frac{\sqrt{\alpha_k}}{\sqrt{\alpha_{k+1}}}-1-\frac{\xi}{2(k+K)}\right|&\leq \frac{\alpha_k^2}{4\alpha}.
    \end{align*}
    Furthermore, recall that $\E[\|y_k\|]=\E[\|x_k\|/\sqrt{\alpha_k}]\leq \sqrt{\E[\|x_k\|^2/\alpha_k]}$. Thus, from Lemma \ref{lem:iterate_bounds}, we get
    \begin{align*}
        \E[\|y_k\|]\leq \sqrt{\E[\|x_k\|^2/\alpha_k]}\leq \mathcal{O}(1).
    \end{align*}
    Combining all the bounds leads us to
    \begin{align*}
        \E[\|T_1\|]\leq \mathcal{O}(\alpha_k^2).
    \end{align*}
    For $T_2$, we proceed as follows:
    \begin{align*}
        \E[\|T_2\|]&\leq \alpha_k(\E[\|x_k\|+\|b_k\|])\left|\frac{1}{\sqrt{\alpha_{k+1}}}-\frac{1}{\sqrt{\alpha_k}}\right|.
    \end{align*}
    Using Lemma \ref{lem:iterate_bounds}, Assumption \ref{assump:noise} on $\E[\|b_k\|]$, and Lemma \ref{lem:step-size_prop}, we get
    \begin{align*}
        \E[\|T_2\|]\leq \mathcal{O}\left(\frac{\sqrt{\alpha_k}}{k+K}\right).
    \end{align*}
    Using the aforementioned coupling, the error dynamics is given by
    \begin{align*}
        y_{k+1}-z_{k+1}=(1+\alpha_kJ_k^{(\alpha, \xi)})(y_k-z_k)+T_1+T_2.
    \end{align*}
    Taking norm $\|\cdot\|$ on both sides and using triangle-inequality, we obtain
    \begin{align*}
        \|y_{k+1}-z_{k+1}\|&\leq |I+\alpha_kJ_k^{(\alpha, \xi)}|\|y_k-
        z_k\|+\|T_1\|+\|T_2\|.
    \end{align*}
    Using Assumption \ref{assump:operator} for the second term, we get
    \begin{align*}
        \|y_{k+1}-z_{k+1}\|&\leq \left(1-\frac{\alpha_k}{2}\right)\|y_k-z_k\|+\|T_1\|+\|T_2\|.
    \end{align*}
    Taking expectation on both sides and using the upper bounds on $T_1$ and $T_2$, we get
    \begin{align*}
        \E[\|y_{k+1}-z_{k+1}\|]&\leq \left(1-\frac{\alpha_k}{2}\right)\E[\|y_k-z_k\|]+\mathcal{O}(\alpha_k^2)+\mathcal{O}\left(\frac{\sqrt{\alpha_k}}{k+K}\right).
    \end{align*}
    Now we apply Lemma \ref{lem:rec_sol} with $u_0=\E[\|y_0\|]$, $\mu_1=1/2$, $\mu_2=0$, $\rho_2=\min(1, 1/\xi-1/2)$ and for some constant $\mu_3$, we get
    \begin{align*}
        \E[\|y_k-z_k\|]\leq \E[\|y_0\|]e^{-\frac{\alpha}{2(1-\xi)}\left((k+K)^{1-\xi}-K^{1-\xi}\right)}+\mathcal{O}\left(\max\left(\alpha_k, \alpha_k^{(2-\xi)/(2\xi)}\right)\right)\\
        \implies \E[\|y_k-z_k\|]\leq \E[\|y_0\|]e^{-\frac{\alpha}{2(1-\xi)}\left((k+K)^{1-\xi}-K^{1-\xi}\right)}+\mathcal{O}\left(\max\left(\alpha_k, \alpha_k^{(2-\xi)/(2\xi)}\right)\right).
    \end{align*}
    Note that $\sqrt{\alpha_k}$ is dominant over $\mathcal{O}\left(\max\left(\alpha_k, \alpha_k^{(2-\xi)/(2\xi)}\right)\right)$ since $\xi<1$. Thus, we have
    \begin{align*}
        d_{\mathcal{W}}(y_k, Y_k)\geq \Omega(\sqrt{\alpha_k}).
    \end{align*}
\end{enumerate}

\noindent\textbf{Asymptotic Gaussian:} Again using reverse-triangle inequality, we get
\begin{align*}
    d_{\mathcal{W}}(y_k, \Sigma^{1/2}Z)\geq \left|d_{\mathcal{W}}(y_{k}, (\Sigma_k^{(\alpha, \xi)})^{1/2}Z)-d_{\mathcal{W}}((\Sigma_k^{(\alpha, \xi)})^{1/2}Z, \Sigma^{1/2}Z)\right|
\end{align*}
Recall from Corollary \ref{cor:asymp_gaussian_rate}, we showed that $d_{\mathcal{W}}((\Sigma_k^{(\alpha, \xi)})^{1/2}Z, \Sigma^{1/2}Z)\leq \mathcal{O}(\alpha_k, (k+K)^{\xi-1})$. Thus, our goal is to match this rate with the convergence rate for lower bound. Combining this tight lower bound with the above relation yields the desired lower bound for $d_{\mathcal{W}}(y_k, \Sigma^{1/2}Z)$.

Let $\mathfrak{u}\in \mathbb{R}^d$ be a unit vector. To establish a sharp lower bound, we will use Lemma \ref{lem:wass1_lower_bound} which shows that 1-D projection of Wasserstein-1 distance is less than the Wasserstein-1 distance in the full space. Thus, we have
\begin{align*}
    d_{\mathcal{W}}((\Sigma_k^{(\alpha, \xi)})^{1/2}Z, \Sigma^{1/2}Z)\geq d_{\mathcal{W}}(\mathfrak{u}^T(\Sigma_k^{(\alpha, \xi)})^{1/2}Z, \mathfrak{u}^T\Sigma^{1/2}Z).
\end{align*}
Since the closed form solution for the Wasserstein-1 distance between two Gaussians is known, we get
\begin{align*}
    d_{\mathcal{W}}(\mathfrak{u}^T(\Sigma_k^{(\alpha, \xi)})^{1/2}Z, \mathfrak{u}^T\Sigma^{1/2}Z)&=\sqrt{\frac{2}{\pi}}|\sqrt{\mathfrak{u}^T\Sigma_k^{(\alpha, \xi)}\mathfrak{u}}-\sqrt{\mathfrak{u}^T\Sigma \mathfrak{u}}|\\
    &=\sqrt{\frac{2}{\pi}}\frac{|\mathfrak{u}^T\Sigma_k^{(\alpha, \xi)}\mathfrak{u}-\mathfrak{u}^T\Sigma \mathfrak{u}|}{\sqrt{\mathfrak{u}^T\Sigma_k^{(\alpha, \xi)}\mathfrak{u}}+\sqrt{\mathfrak{u}^T\Sigma \mathfrak{u}}}\\
    &\geq \sqrt{\frac{2}{\pi}}\frac{|\mathfrak{u}^T\Sigma_k^{(\alpha, \xi)}\mathfrak{u}-\mathfrak{u}^T\Sigma \mathfrak{u}|}{\sqrt{\mathfrak{u}^T\Sigma \mathfrak{u}}}.\tag{$\sqrt{\mathfrak{u}^T\Sigma_k^{(\alpha, \xi)}\mathfrak{u}}\geq 0$}
\end{align*}
Using Lemma \ref{lem:sigma_diff_lower_bound}, we obtain
\begin{align*}
    d_{\mathcal{W}}(\mathfrak{u}^T(\Sigma_k^{(\alpha, \xi)})^{1/2}Z, \mathfrak{u}^T\Sigma^{1/2}Z)\geq \Omega(\max(\alpha_k, (k+K)^{\xi-1})).
\end{align*}
Therefore, the convergence rate for upper and lower bound match. Combining this bound with the lower bound for Wasserstein-1 distance to the time-varying Gaussian gives us
\begin{align*}
    d_{\mathcal{W}}(y_k, \Sigma^{1/2}Z)\geq \Omega(\max(\sqrt{\alpha_k}, (k+K)^{\xi-1})).
\end{align*}

\begin{lemma}\label{lem:sigma_diff_lower_bound}
    The following relation holds:
    \begin{align*}
        |\mathfrak{u}^T\Sigma_k^{(\alpha, \xi)}\mathfrak{u}-\mathfrak{u}^T\Sigma \mathfrak{u}|\geq \Omega(\max(\alpha_k, (k+K)^{\xi-1})).
    \end{align*}
\end{lemma}
\begin{proof}
    Recall that $\Sigma_k^{(\alpha, \xi)}$ is given by the following recursion:
    \begin{align*}
        \Sigma_{k+1}^{(\alpha, \xi)}&=(1 + \alpha_k J^{(\alpha, \xi)}_k)^2\Sigma_k^{(\alpha, \xi)}+\alpha_{k}\Sigma_b\\
        \implies \Sigma_{k+1}^{(\alpha, \xi)}-\Sigma&=(1 + \alpha_{k} J^{(\alpha, \xi)}_k)^2(\Sigma_k^{(\alpha, \xi)}-\Sigma)+\alpha_{k}^2(J^{(\alpha, \xi)}_k)^2\Sigma+\frac{\xi\mathbbm{1}_{\xi<1}}{(k+K)}\Sigma.
    \end{align*}
    where for the last equation, we used the fact that $\Sigma$ is the solution to Lyapunov equation \eqref{eq:lyap_eq}. Define $v_k=\mathfrak{u}^T\Sigma_k^{(\alpha, \xi)}\mathfrak{u}-\mathfrak{u}^T\Sigma \mathfrak{u}$. Then, from the above equation, we get
    \begin{align*}
        v_{k+1}&=(1 + \alpha_{k} J^{(\alpha, \xi)}_k)^2v_k+\alpha_{k}^2(J^{(\alpha, \xi)}_k)^2u^T\Sigma \mathfrak{u}+\frac{\xi \mathfrak{u}^T\Sigma \mathfrak{u}}{(k+K)}\\
        &=v_0\prod_{i=0}^k(1 + \alpha_{i} J^{(\alpha, \xi)}_i)^2+\mathfrak{u}^T\Sigma \mathfrak{u}\sum_{i=0}^k\alpha_i^2(J^{(\alpha, \xi)}_i)^2\prod_{l=i+1}^k(1 + \alpha_{l} J^{(\alpha, \xi)}_l)^2\\
        &~+\xi \mathfrak{u}^T\Sigma \mathfrak{u}\sum_{i=0}^k\frac{1}{k+K}\prod_{l=i+1}^k(1 + \alpha_{l} J^{(\alpha, \xi)}_l)^2\\
        &\geq v_0\prod_{i=0}^k(1 + \alpha_{i} J^{(\alpha, \xi)}_i)^2+\frac{\mathfrak{u}^T\Sigma \mathfrak{u}}{4}\underbrace{\sum_{i=0}^k\alpha_i^2\prod_{l=i+1}^k(1 + \alpha_{l} J^{(\alpha, \xi)}_l)^2}_{\Omega(\alpha_k)}+\xi \mathfrak{u}^T\Sigma \mathfrak{u}\underbrace{\sum_{i=0}^k\frac{1}{k+K}\prod_{l=i+1}^k(1 + \alpha_{l} J^{(\alpha, \xi)}_l)^2}_{\Omega((k+K)^{\xi-1})}.
    \end{align*}
    where for the last inequality, we suppose that $K$ is large enough such that $J^{(\alpha, \xi)}_k\geq 1/2$. It is a straightforward exercise to show that the first term in the above expression is a higher order term and decays rapidly. Thus, the dominant rates in the lower order are given by the second and third terms. The explicit rates for these sums can be derived by mirroring the arguments in the proof of Lemma \ref{lem:bounds}; we omit the repetitive details here for brevity.
\end{proof}

\begin{lemma}\label{lem:wass1_lower_bound}
    Let $\nu_X$ and $\nu_Y$ be two distributions in $\mathbb{R}^d$. Let $\mathfrak{u}\in \mathbb{R}^d$ such that $\|\mathfrak{u}\|=1$. Then, we have
    \begin{align*}
        d_{\mathcal{W}}(\mathfrak{u}^TX, \mathfrak{u}^TY)\leq d_{\mathcal{W}}(X, Y)
    \end{align*}
\end{lemma}
\begin{proof}
    Let $\gamma^*$ be the optimal coupling between $\nu_X$ and $\nu_Y$. Then, by definition, we have
    \begin{align*}
        d_{\mathcal{W}}(\mathfrak{u}^TX, \mathfrak{u}^TY)&\leq \E_{(X, Y)\sim \gamma^*}[|\mathfrak{u}^TX-\mathfrak{u}^TY|]\\
        &\leq \|\mathfrak{u}\|\E_{(X, Y)\sim \gamma^*}[\|X-Y\|]\\
        &\leq \E_{(X, Y)\sim \gamma^*}[|\|X-Y\|]=d_{\mathcal{W}}(X, Y).
    \end{align*}
\end{proof}

\subsection{Proof of Proposition \ref{cor:tail_bounds}}\label{sec:tail_bounds}
\begin{proposition} \label{lem: from_W1_to_tail_bound}
Let $Y$ be a real-valued random vector in $\mathbb{R}^d$. Then, for any $a>0$, $\rho\in[0,1)$, positive symmetric matrix $\Sigma\in \mathbb{R}^d$, a standard Gaussian vector $Z\in \mathbb{R}^d$, and unit vector $\mathfrak{u}\in\mathbb{R}^d$, we have
\begin{equation*}
\big|\mathbb{P}(\langle Y, \mathfrak{u}\rangle > a) - \mathbb{P}(\langle \Sigma^{1/2}Z, \mathfrak{u}\rangle > a)\big|
\;\le\;
\frac{(1-\rho) a}{\sqrt{\mathfrak{u}^T \Sigma \mathfrak{u}}}\,\phi( \frac{\rho a}{\sqrt{\mathfrak{u}^T \Sigma \mathfrak{u}}}) + \frac{d_{\mathcal{W}}(Y,\Sigma^{1/2}Z)}{(1-\rho)a} 
\end{equation*}
where $\phi(x)=\frac{1}{\sqrt{2\pi}}e^{-x^2/2}$ is the standard normal density.
\end{proposition}
\begin{proof}
Denote $\Phi(a)$ as cumulative distribution function (CDF) of standard normal distribution, $\Phi^c(a) = 1 - \Phi(a)$ as its complementary CDF, and $\phi(a)=e^{-\frac{x^2}{2}}$ as the probability density function (PDF) of standard normal distribution. Then 
for any $\rho \in [0,1)$ and $a\geq 0$, we have
\begin{align*}
    \mathbb{P}(\langle Y, \mathfrak{u}\rangle > a) &\leq \mathbb{P}(\langle Y - \Sigma^{1/2}Z, \mathfrak{u}\rangle \geq (1-\rho)a) + \mathbb{P}(\langle \Sigma^{1/2}Z,\mathfrak{u}\rangle \geq \rho a) \\
    &\leq \mathbb{P}(|\langle \mathfrak{u}, Y - \Sigma^{1/2}Z\rangle |\geq (1-\rho)a) + \mathbb{P}(\langle \Sigma^{1/2}Z,\mathfrak{u}\rangle \geq \rho a) \\
    &\overset{(a)}{\leq} \frac{\|\mathfrak{u}\| \mathbb{E}[\| Y - \Sigma^{1/2}Z\|]}{(1-\rho) a} + \mathbb{P}(\langle \Sigma^{1/2}Z,\mathfrak{u}\rangle \geq \rho a) \\
    &\overset{(b)}{=} \frac{\|\mathfrak{u}\| d_{\mathcal{W}}(Y,\Sigma^{1/2}Z)}{(1-\rho) a} + \mathbb{P}(\langle \Sigma^{1/2}Z,\mathfrak{u}\rangle \geq \rho a) \\
    &\overset{(c)}{=} \frac{d_{\mathcal{W}}(Y,\Sigma^{1/2}Z)}{(1-\rho) a} + \Phi^c(\frac{\rho a}{\sqrt{\mathfrak{u}^T \Sigma \mathfrak{u}}}) \\
    &\overset{(d)}{\leq} \frac{d_{\mathcal{W}}(Y,\Sigma^{1/2}Z)}{(1-\rho) a} + \frac{(1-\rho) a}{\sqrt{\mathfrak{u}^T \Sigma \mathfrak{u}}}\,\phi( \frac{\rho a}{\sqrt{\mathfrak{u}^T \Sigma \mathfrak{u}}}) + \Phi^c(\frac{ a}{\sqrt{\mathfrak{u}^T \Sigma \mathfrak{u}}}) 
\end{align*}

Where inequality $(a)$ is from Markov inequality. Inequality $(b)$ follows from the definition of Wasserstein-$1$ distance. Note that we can choose a coupling $(Y,Z)$ such that the $L^1$ distance between $Y$ and $Z$ is within error $\epsilon$ to the Wasserstein-$1$ distance $d_{\mathcal{W}}(Y,\Sigma^{1/2}Z)$. Such error $\epsilon$ can be arbitrary small since Wasserstein-$1$ distance is defined as infimum over all couplings. Thus, letting $\epsilon \downarrow 0$ we have the equality in $(b)$.
 Equality $(c)$ is from the fact $\|\mathfrak{u}\|=1$.
Inequality
$(d)$ follows from Taylor Expansion $\Phi^c(\rho b) \leq \Phi^c(b) + (b - \rho b) \sup_{\tilde{x} \in [\rho b, b]} \phi(\tilde{x})$. Meanwhile, similar argument shows the lower bound
\begin{align*} 
    \mathbb{P}(\langle \Sigma^{1/2}Z,\mathfrak{u}\rangle  > a) &\leq \mathbb{P}(\langle Y, \mathfrak{u}\rangle > a) + \mathbb{P}(|\langle Y, \mathfrak{u}\rangle - \langle \Sigma^{1/2}Z,\mathfrak{u}\rangle| \geq (1 - \rho) a)\\
    &+ \mathbb{P}(a \leq \langle \Sigma^{1/2}Z,\mathfrak{u}\rangle\leq(2-\rho)a) \\
    &\leq \mathbb{P}(\langle Y, \mathfrak{u}\rangle > a) + [(1-\rho)a]^{-1}d_{\mathcal{W}}(Y,\Sigma^{1/2}Z) + \frac{(1-\rho) a}{\sqrt{\mathfrak{u}^T \Sigma \mathfrak{u}}}\,\phi( \frac{\rho a}{\sqrt{\mathfrak{u}^T \Sigma \mathfrak{u}}})
\end{align*}
The claim follows.
\end{proof}
To proceed, we optimize the choice of $\rho$ in Lemma~\ref{lem: from_W1_to_tail_bound} by setting $\rho := 1 - \sqrt{d_{\mathcal{W}}(Y, \Sigma^{1/2}Z)}$. 
With such choice of $\rho$, we have the following concentration bound.
\begin{align*}
  |\mathbb{P}(\langle Y, \mathfrak{u}\rangle > a) - \mathbb{P}(\langle \Sigma^{1/2} Z,\mathfrak{u}\rangle > a) | 
  &\leq \frac{(1-\rho) a}{\sqrt{\mathfrak{u}^T \Sigma \mathfrak{u}}}\,\phi( \frac{\rho a}{\sqrt{\mathfrak{u}^T \Sigma \mathfrak{u}}}) + \frac{d_{\mathcal{W}}(Y, \Sigma^{1/2}Z)}{(1-\rho)a} \\
  &\overset{(a)}{\leq} \frac{(1-\rho) a}{\sqrt{\mathfrak{u}^T \Sigma \mathfrak{u}}}\,\phi( \frac{a}{2\sqrt{\mathfrak{u}^T \Sigma \mathfrak{u}}}) + \frac{d_{\mathcal{W}}(Y, \Sigma^{1/2}Z)}{(1-\rho)a} \\
  &= \sqrt{d_{\mathcal{W}}(Y, \Sigma^{1/2}Z)} \left( \frac{ a}{\sqrt{\mathfrak{u}^T \Sigma \mathfrak{u}}}\exp(-\frac{a^2}{8u^T \Sigma \mathfrak{u}}) + \frac{1}{a} \right) \\
  &\overset{(b)}{\leq} (8\sqrt{\mathfrak{u}^T \Sigma \mathfrak{u}}+1)\cdot\frac{\sqrt{d_{\mathcal{W}}(Y, \Sigma^{1/2}Z)}}{a}.
\end{align*}
Inequality $(a)$ holds for sufficiently small $d_{\mathcal{W}}(Y, \Sigma^{1/2}Z)$. According to Theorem \ref{thm_main:main_thm}, we can achieve such small Wasserstein distance once $\alpha$ is small enough in case (1), or
$k$ is large enough in cases (2) and (3). Inequality $(b)$ follows from the fact that $be^{-b^2/8} \leq 8/b$ for all $b > 0$.
Finally, we have $0\leq \mathfrak{u}^T \Sigma \mathfrak{u} \leq \|\Sigma\|_{op}$ for any unit vector $\mathfrak{u}$.

We will use the above bounds and set $Y=y_k$, $\Sigma=\Sigma^{(\alpha, \xi)}_k$. We can now plug in the upper bounds from Theorem \ref{thm_main:main_thm} to achieve non-asymptotic concentration bounds for each of the three cases.

\subsubsection{Tightness of the tail bound}\label{sec:tightness_tail}
It appears that the tail bounds in Proposition \ref{cor:tail_bounds} provide a tighter characterization of the tail, especially for moderate values of $k$ than directly using Corollary \ref{cor:asymp_gaussian_rate}. Consider the one-dimensional case for simplicity. Then, with Corollary \ref{cor:asymp_gaussian_rate}, the tail bound is given as
\begin{align}\label{eq:tail_asymp}
    \mathbb{P}(y_k > a)  &\leq \frac{1}{\sqrt{2\pi\Sigma^{(\alpha)}}}\frac{e^{-\frac{a^2}{2\Sigma^{(\alpha)}}}}{a}+\tilde{\mathcal{O}}\left(\frac{1}{(k+K)^{\delta/4}}\right) \frac{1}{a}+\mathcal{O}\left(\frac{1}{k+K}\right)\frac{1}{a}.
\end{align}
In other words, the above relation is established by using triangle inequality on the Wasserstein-1 distance: 
$$d_{\mathcal{W}}(y_k, (\Sigma^{(\alpha, \xi)})^{1/2}Z)\leq d_{\mathcal{W}}(y_k, (\Sigma_k^{(\alpha, \xi)})^{1/2}Z)+d_{\mathcal{W}}((\Sigma_k^{(\alpha, \xi)})^{1/2}Z, (\Sigma^{(\alpha, \xi)})^{1/2}Z).$$ 
Therefore, the additional $\mathcal{O}\left(1/(k+K)\right)/a$ term appears due to the mismatch between the time-varying and the asymptotic Gaussians, captured by the second term.
On the other hand, we show in Lemma \ref{lem:matrix_theta_bound} that $|\Sigma_k^{(\alpha, 1)}-\Sigma^{(\alpha)}|\leq \mathcal{O}(1/(k+K))$. By using Taylor series expansion for the exponential term around $\Sigma^{(\alpha)}$, we get
\begin{align*}
    \frac{1}{\sqrt{2\pi\Sigma_k^{(\alpha, 1)}}}\frac{e^{-\frac{a^2}{2\Sigma_k^{(\alpha, 1)}}}}{a} \approx \frac{1}{\sqrt{2\pi\Sigma^{(\alpha)}}}\frac{e^{-\frac{a^2}{2\Sigma^{(\alpha)}}}}{a}+\mathcal{O}\left(\frac{1}{k+K}\right)\frac{e^{-\frac{a^2}{2\Sigma^{(\alpha)}}}}{2a\sqrt{2\pi\Sigma^{(\alpha)}}}\left(\frac{a^2}{(\Sigma^{(\alpha)})^2}-\frac{1}{\Sigma^{(\alpha)}}\right).
\end{align*}
In particular, the difference between the tail of two Gaussians is of the order of $\mathcal{O}(a\exp(-a^2/(2\Sigma^{(\alpha)})))$. Substituting the above approximation into the upper bound \eqref{eq:mill's}, we get
\begin{align}\label{eq:tail_finite}
    \mathbb{P}(y_k > a)  &\lessapprox  \frac{1}{\sqrt{2\pi\Sigma^{(\alpha)}}}\frac{e^{-\frac{a^2}{2\Sigma^{(\alpha)}}}}{a}+\tilde{\mathcal{O}}\left(\frac{1}{(k+K)^{\delta/4}}\right) \frac{1}{a}+\mathcal{O}\left(\frac{1}{k+K}\right)\mathcal{O}(ae^{-\frac{a^2}{2\Sigma^{(\alpha)}}}).
\end{align}
Since $\mathcal{O}(a\exp(-a^2/(2\Sigma^{(\alpha)})))\ll a^{-1}$ as $a$ grows, it becomes evident that the tail bound in \eqref{eq:tail_finite} is strictly tighter with respect to $a$ compared to the tail bound in \eqref{eq:tail_asymp}. While extending this argument to the multidimensional matrix setting appears to follow from standard algebraic techniques, we omit the formal derivation here for brevity.

\section{Simulations}\label{sec:simulations}
\subsection{Simulation details for Figure \ref{fig:comparison}}
\begin{figure*}[t!]
     \centering
     \begin{subfigure}[b]{0.25\textwidth}
         \centering
         \includegraphics[width=\linewidth]{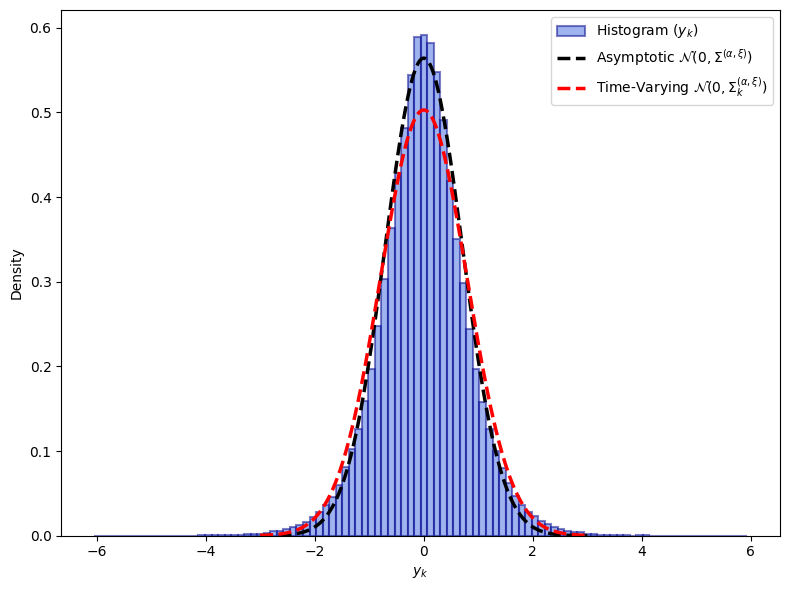}
         \caption{$\xi=0.3$, $k=100$}
         \label{fig:0.3-100}
     \end{subfigure}
     \hspace{5mm} 
     \begin{subfigure}[b]{0.25\textwidth}
         \centering
         \includegraphics[width=\linewidth]{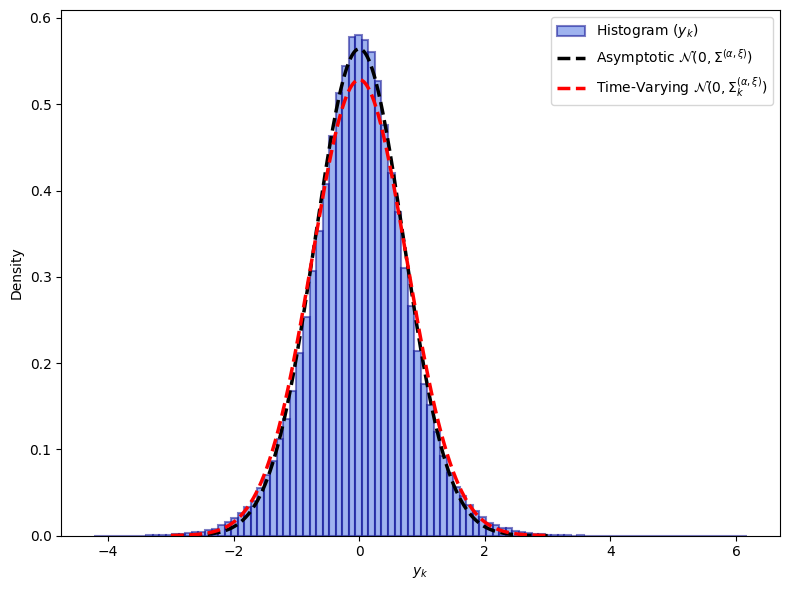}
         \caption{$\xi=0.3$, $k=1000$}
         \label{fig:0.3-1000}
     \end{subfigure}
     \hspace{5mm}
     \begin{subfigure}[b]{0.25\textwidth}
         \centering
         \includegraphics[width=\linewidth]{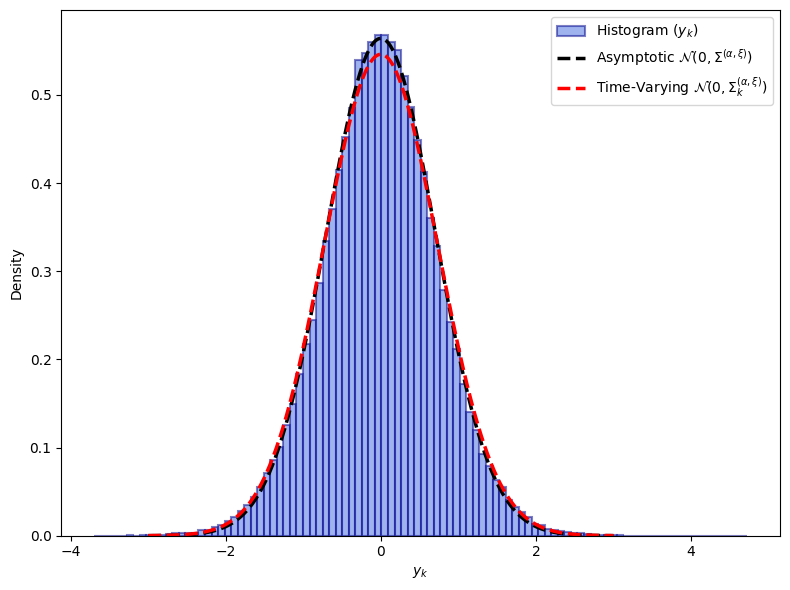}
         \caption{$\xi=0.3$, $k=5000$}
         \label{fig:0.3-5000}
     \end{subfigure}
     \\
     \vspace{2mm}
     \begin{subfigure}[b]{0.25\textwidth}
         \centering
         \includegraphics[width=\linewidth]{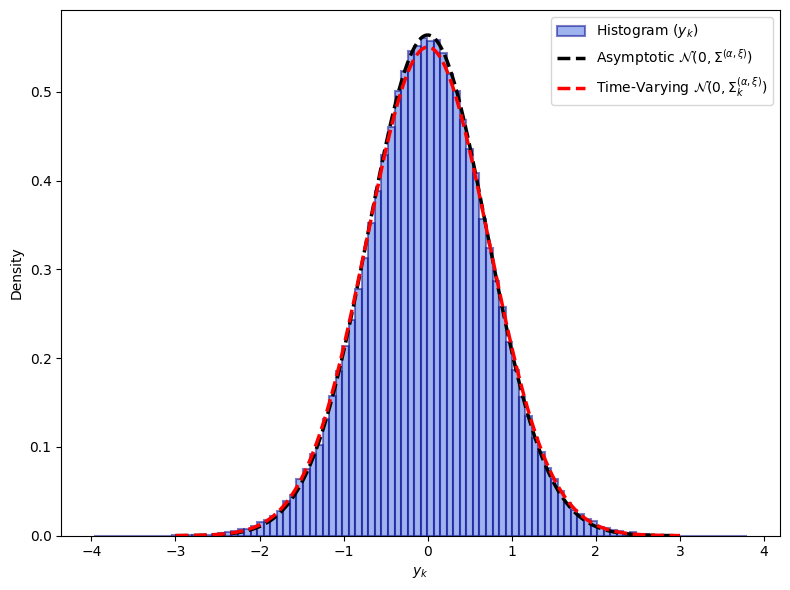}
         \caption{$\xi=0.6$, $k=100$}
         \label{fig:0.6-100}
     \end{subfigure}
     \hspace{5mm}
     \begin{subfigure}[b]{0.25\textwidth}
         \centering
         \includegraphics[width=\linewidth]{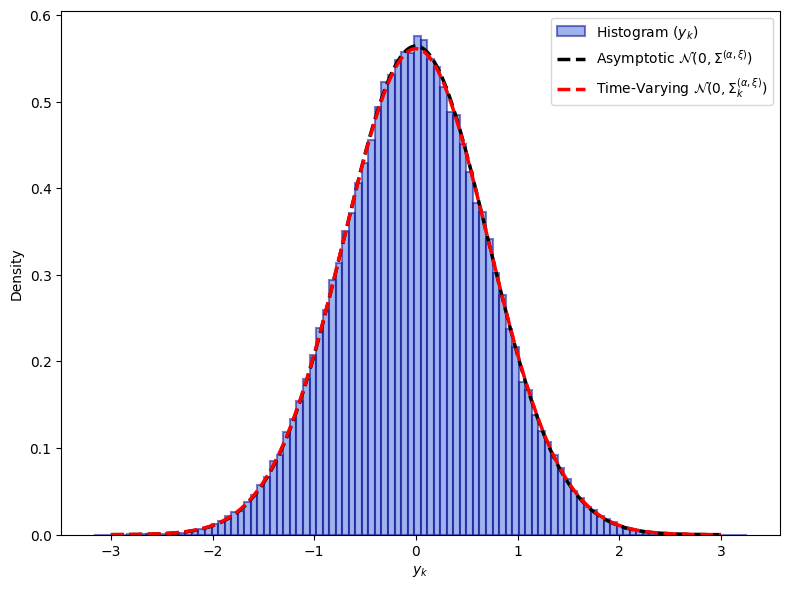}
         \caption{$\xi=0.6$, $k=1000$}
         \label{fig:0.6-1000}
     \end{subfigure}
     \hspace{5mm}
     \begin{subfigure}[b]{0.25\textwidth}
         \centering
         \includegraphics[width=\linewidth]{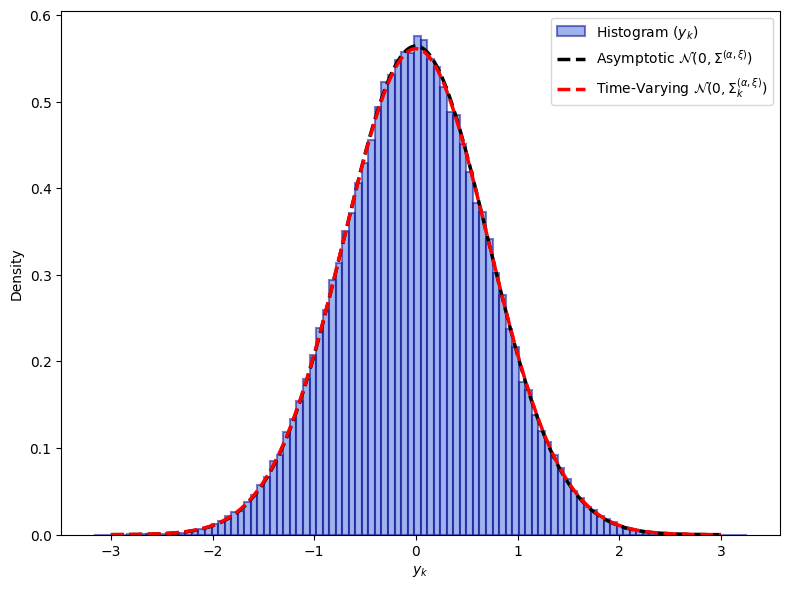}
         \caption{$\xi=0.6$, $k=5000$}
         \label{fig:0.6-5000}
     \end{subfigure}
        \caption{Comparing the Gaussian approximations for $y_k$ at various time instants.}
        \label{fig:comparison_bad}
\end{figure*}
The experiments were conduction using JAX on Tensor Processing Units (TPUs) to efficiently track the empirical distribution of the iterates. 
\begin{itemize}
    \item Algorithm Dynamics: We simulated a one-dimensional linear scalar SA update $x_{k+1} = (1 - \alpha_k A)x_k + \alpha_k M_k$ with the drift coefficient $A = 2.0$.
    \item Step-Size Schedule: The step-size schedule is set $\alpha_k = 1 / (k + 100)^\xi$ for various choice of $\xi$.
    \item Noise Distribution: We only have additive noise in the SA update. The noise was drawn from a Laplace distribution with mean zero and scale parameter as $1$. 
    \item Number of sample paths: For each value of $\xi$, we simulated $N = 2\times 10^5$ independent paths with $k_{max} = 10,000$ iterations.
    \item Evaluation Metric: The precise 1D Wasserstein-1 distance was recorded every 50 iterations. We calculated this by tracking the exact scaled theoretical variance $\Sigma_k^{(\alpha, \beta)}$ at each evaluation checkpoint, drawing $N$ target Gaussian samples, and computing the exact $L_1$ distance between the sorted empirical rescaled iterates $y_k = x_k / \sqrt{\alpha_k}$ and the sorted theoretical Gaussian samples.
\end{itemize}

\subsubsection{Simulation for $\xi<2/3$}
We also ran similar simulations for $\xi<2/3$ and identical choice of other parameters. As shown in Figure \ref{fig:comparison_bad}, when $\xi=0.3$, the asymptotic distribution provides a better approximation for the distribution of $y_k$ than the proposed time-varying distribution. However, as $\xi$ increases to $0.6$, the two approximations are almost equally good. 

Furthermore, Figure \ref{fig:comparison_good} reveals an interesting dynamic: for a fixed time instant, the accuracy of the time-varying Gaussian approximation can be significantly enhanced by appropriately tuning the hyperparameter $K$ in the step-size schedule. This indicates that Theorem \ref{thm_main:main_thm} provides a tight characterization of the distribution of $y_k$ for appropriate choice of step-size schedules. Intuitively, the reason for this competing approximation comes from the transient effect of dropping higher order terms. For certain values of $\xi$, the influence of these higher-order terms for $\Sigma_k$ dominates over the approximation error from the asymptotic Gaussian in finite time, making the tuning of $K$ critical to minimizing the approximation gap.

\begin{figure*}[t!]
     \centering
     \begin{subfigure}[b]{0.25\textwidth}
         \centering
         \includegraphics[width=\linewidth]{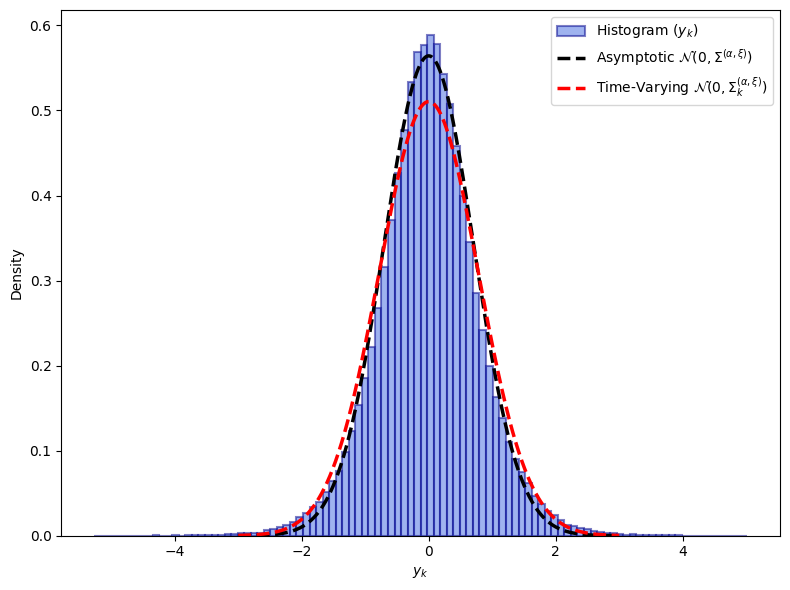}
         \caption{$\xi=0.3$, $K=100$}
         \label{fig:0.3-K100}
     \end{subfigure}
     \hspace{5mm} 
     \begin{subfigure}[b]{0.25\textwidth}
         \centering
         \includegraphics[width=\linewidth]{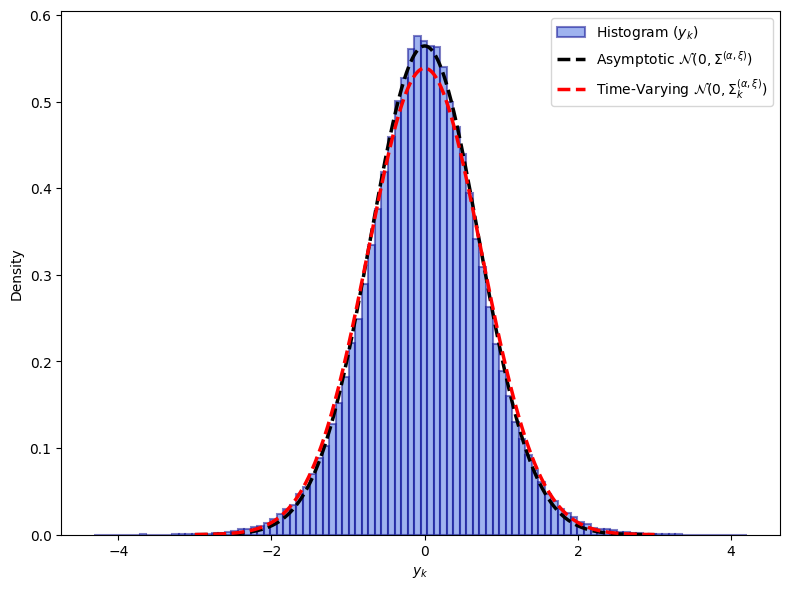}
         \caption{$\xi=0.3$, $K=3000$}
         \label{fig:0.3-K3000}
     \end{subfigure}
     \hspace{5mm}
     \begin{subfigure}[b]{0.25\textwidth}
         \centering
         \includegraphics[width=\linewidth]{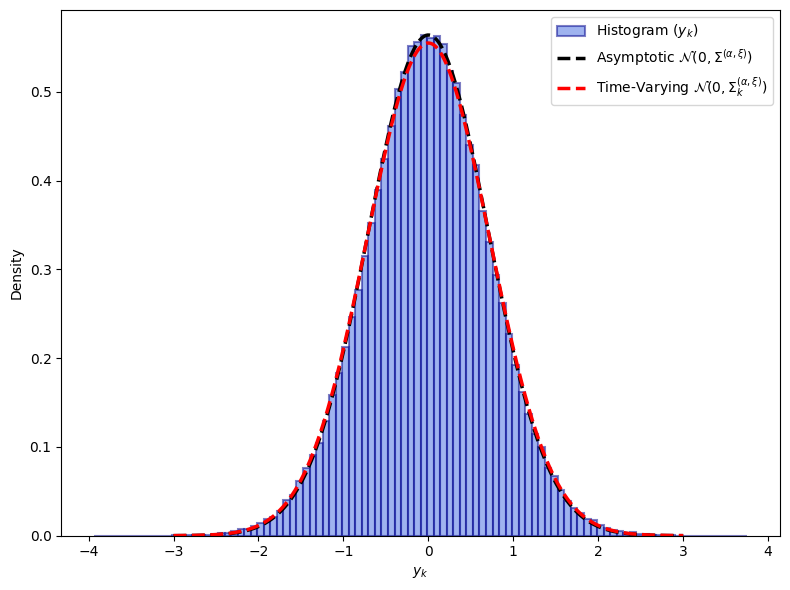}
         \caption{$\xi=0.3$, $K=100000$}
         \label{fig:0.3-K100000}
     \end{subfigure}
     \\
     \vspace{2mm}
     \begin{subfigure}[b]{0.25\textwidth}
         \centering
         \includegraphics[width=\linewidth]{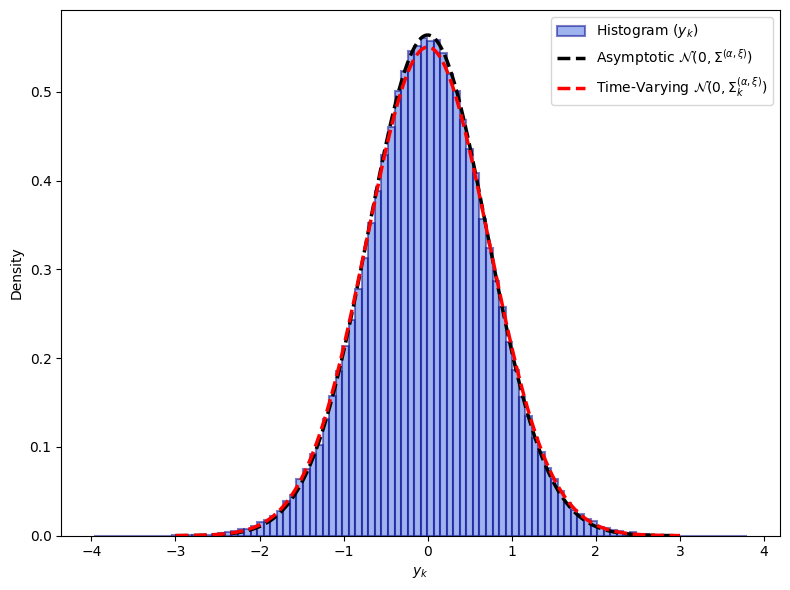}
         \caption{$\xi=0.6$, $K=100$}
         \label{fig:0.6-K100}
     \end{subfigure}
     \hspace{5mm}
     \begin{subfigure}[b]{0.25\textwidth}
         \centering
         \includegraphics[width=\linewidth]{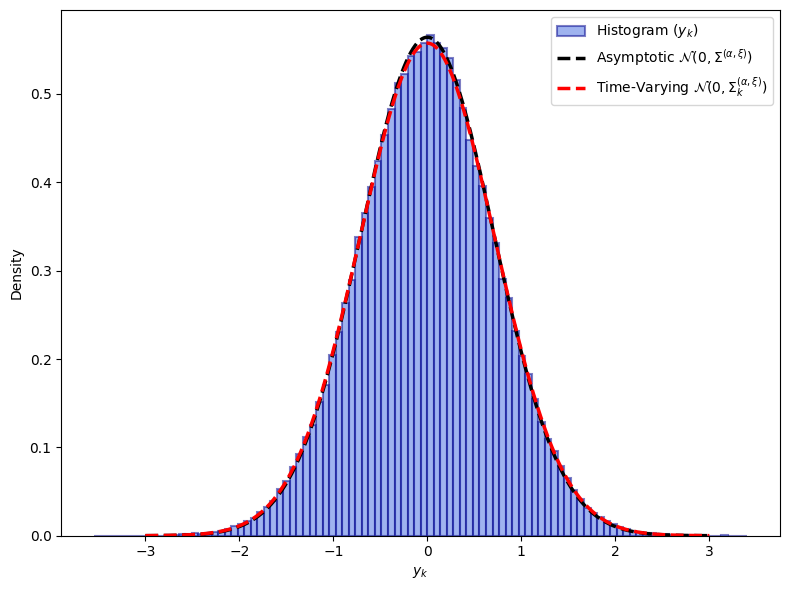}
         \caption{$\xi=0.6$, $K=1000$}
         \label{fig:0.6-K1000}
     \end{subfigure}
     \hspace{5mm}
     \begin{subfigure}[b]{0.25\textwidth}
         \centering
         \includegraphics[width=\linewidth]{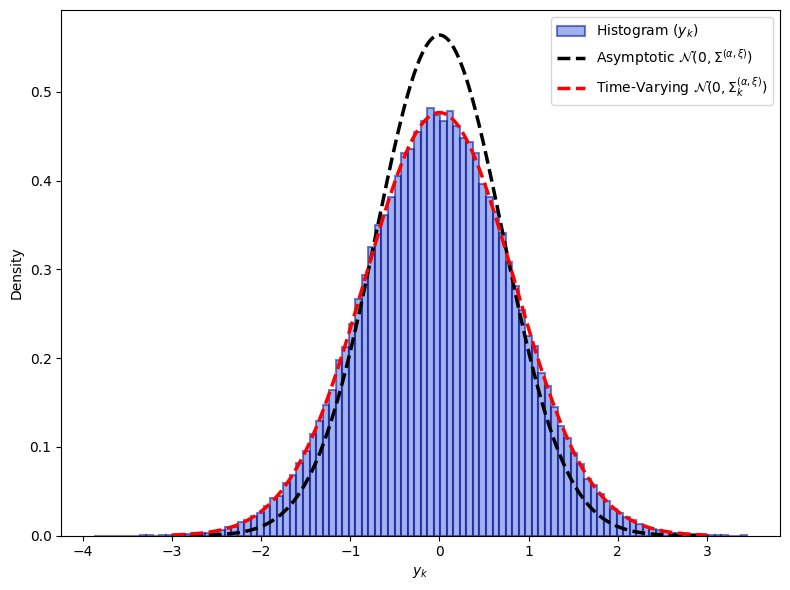}
         \caption{$\xi=0.6$, $K=3000$}
         \label{fig:0.6-K3000}
     \end{subfigure}
        \caption{Comparing the Gaussian approximations for $y_k$ at a fixed time instant $k=200$ (or fixed number of samples) for various choice of hyperparameter $K$.}
        \label{fig:comparison_good}
\end{figure*}

\subsection{Simulation details for Figure \ref{fig:slope}}
As with our previous experiments, these simulations were conducted using JAX on Tensor Processing Units (TPUs). Almost all the parameters for the SA iteration remains identical to the previous experiment. We changed the noise distribution to exponential distribution with scale parameter as 1 (we centered the noise so that it is mean zero). Furthermore, we increased the number of sample paths to $N=2\times 10^8$ independent paths with $k_{max}=100,000$ iterations for accurate results. The offset parameter $K$ in the step size was set as $1000$.

\subsection{Simulation details for Figure \ref{fig:slope_dist}}
These simulations were also conducted using JAX on Tensor Processing Units (TPUs). All the parameters for the SA iteration remains identical to the previous experiments, with the exception that the offset hyperparameter $K$ was set to $1000$. The scale parameters for Laplace distribution is set as $1$ while the support for uniform distribution is in $[-1, 1]$.

\section{Auxiliary Lemmas}\label{sec:aux_lemmas}

\begin{lemma}\label{lem:inv_bound}
    Let $\{\Upsilon_k^{(\alpha, \xi)}\}_{k\geq 0}$ be given by 
    \begin{align*}
        \Upsilon_{k+1}^{(\alpha, \xi)}=(I + \alpha_k J^{(\alpha, \xi)}_k)\Upsilon_k^{(\alpha, \xi)}(I + \alpha_k J^{(\alpha, \xi)}_k)^T+\alpha_k\Sigma_b
    \end{align*}
    where $\Upsilon_0^{(\alpha, \xi)}$ is arbitrary. Then, we have the following bounds for all $k\geq 1$:
    \begin{enumerate}
        \item For $\xi=0$, we have
        \begin{align*}
            \|\Upsilon_k^{(\alpha, 0)}\|^{-1/2}\leq \alpha^{-1/2}\|\Sigma_b\|^{-1/2}.
        \end{align*}
        \item For $\xi\in (0, 1)$, we have
        \begin{align*}
            \|\Upsilon_k^{(\alpha, \xi)}\|^{-1/2}\leq \left(\max\left(\alpha_k\|\Sigma_b\|, \lambda_{min}(\Sigma)-\|\Upsilon_k^{(\alpha, \xi)}-\Sigma\|\right)\right)^{-1/2}.
        \end{align*}
        \item For $\xi=1$, we have
        \begin{align*}
            \|\Upsilon_k^{(\alpha, 1)}\|^{-1/2}\leq \left(\max\left(\alpha_k\|\Sigma_b\|, \lambda_{min}(\Sigma^{(\alpha)})-\|\Upsilon_k^{(\alpha, 1)}-\Sigma^{(\alpha)}\|\right)\right)^{-1/2}.
        \end{align*}
    \end{enumerate}
    where $\lambda_{min}(\cdot)$ denotes the smallest eigenvalue of the corresponding matrix.
\end{lemma}
\begin{proof}
    Note that $\Sigma_b$ by definition is a positive definite matrix. Thus, for any $x\in \mathbb{R}^d/\{0\}$, we have
    \begin{align*}
        x^T\Upsilon_{k+1}^{(\alpha, \xi)}x=x^T(I + \alpha_k J^{(\alpha, \xi)}_k)\Upsilon_k^{(\alpha, \xi)}(I + \alpha_k J^{(\alpha, \xi)}_k)^Tx+\alpha_kx^T\Sigma_bx.
    \end{align*}
    \begin{enumerate}
        \item $\xi=0$: Note that $x^T(I + \alpha J_F)\Upsilon_k^{(\alpha, 0)}(I + \alpha J_F)^Tx\geq 0$. Thus, we have
        \begin{align*}
            x^T\Upsilon_{k+1}^{(\alpha, 0)}x\geq \alpha x^T\Sigma_bx>0.
        \end{align*}
        Since $\Upsilon_k^{(\alpha, 0)}$ is symmetric by construction, we get
        \begin{align*}
            \|\Upsilon_k^{(\alpha, 0)}\|^{-1/2}\leq \alpha^{-1/2}\|\Sigma_b\|^{-1/2}.
        \end{align*}
        \item $\xi\in (0, 1)$: For this regime, we can use the previous approach to obtain an upper bound on the inverse, however recall that for $\xi>0$, $\alpha_k\downarrow 0$. Thus, the upper bound blows to infinity in the limit. To circumvent this, we will use a corollary due to Weyl's inequality and the fact that $\Upsilon_k^{(\alpha, \xi)}\to \Sigma$ in the following manner:
        \begin{align*}
            |\lambda_{min}(\Upsilon_k^{(\alpha, \xi)})-\lambda_{min}(\Sigma)|\leq \|\Upsilon_k^{(\alpha, \xi)}-\Sigma\|\\
            \implies \lambda_{min}(\Upsilon_k^{(\alpha, \xi)})\geq \lambda_{min}(\Sigma)-\|\Upsilon_k^{(\alpha, \xi)}-\Sigma\|.
        \end{align*}
         We also have the following lower bound from the analysis in the previous case:
        \begin{align*}
            \lambda_{min}(\Upsilon_k^{(\alpha, \xi)})\geq \alpha_k\lambda_{min}(\Sigma).
        \end{align*}
        Combining both the lower bounds, we get
        \begin{align*}
            \lambda_{min}(\Upsilon_k^{(\alpha, \xi)})&\geq \max\left(\alpha_k\lambda_{min}(\Sigma), \lambda_{min}(\Sigma)-\|\Upsilon_k^{(\alpha, \xi)}-\Sigma\|\right)\\
            \implies \|\Upsilon_k^{(\alpha, \xi)}\|^{-1/2}&\leq \left(\max\left(\alpha_k\|\Sigma_b\|, \lambda_{min}(\Sigma)-\|\Upsilon_k^{(\alpha, \xi)}-\Sigma\|\right)\right)^{-1/2}.
        \end{align*}
        \item $\xi=1$: The bound for $\xi=1$ can be obtained by following the exact analysis as in the previous case and replacing $\Sigma$ with $\Sigma^{(\alpha)}$. 
    \end{enumerate}
\end{proof}

\begin{lemma}\label{lem:wass-2}
    Consider two centered Gaussian distributions $\mathcal{N}(0, \Sigma_1)$ and $\mathcal{N}(0, \Sigma_2)$ where $\Sigma_1, \Sigma_2\in \mathbb{R}^{d\times d}$. Then, the Wasserstein-2 distance between them is bounded by
    \begin{align*}
        d_{\mathcal{W}_2}(\Sigma_1^{1/2}Z, \Sigma_2^{1/2}Z)&\leq \frac{\sqrt{d}}{\sqrt{\lambda_{min}(\Sigma_1)}+\sqrt{\lambda_{min}(\Sigma_2)}}\|\Sigma_1-\Sigma_2\|.
    \end{align*}
    where $\lambda_{min}(\cdot)$ is the minimum eigenvalue of the corresponding matrix.
\end{lemma}
\begin{proof}
    The Wasserstein-2 distance between two mean-zero Gaussian is given by
    \begin{align*}
        d_{\mathcal{W}_2}^2(\Sigma_1^{1/2}Z, \Sigma_2^{1/2}Z)= \mathrm{Tr}\left(\Sigma_2+\Sigma_k^{(\alpha, \xi)}-2\left(\Sigma_2^{1/2}\Sigma_k^{(\alpha, \xi)}\Sigma_2^{1/2}\right)^{1/2}\right).
    \end{align*}
    Using standard linear algebra properties, we can easily show that $\mathrm{Tr}\left((A^TA)^{1/2}\right)\geq \mathrm{Tr}(A)$ for any square matrix $A$. Setting $A=(\Sigma_2)^{1/2}\Sigma_1^{1/2}$, we obtain
    \begin{align*}
       \mathrm{Tr}\left(\left(\Sigma_2^{1/2}\Sigma_1\Sigma_2^{1/2}\right)^{1/2}\right)\geq \mathrm{Tr}(\Sigma_2^{1/2}\Sigma_1^{1/2}). 
    \end{align*}
    Furthermore, for any two square matrices $A$ and $B$, we have the identity
    \begin{align*}
        \|A - B\|_F^2 = \text{Tr}((A-B)^2) = \text{Tr}(A^2 + B^2 - 2AB),
    \end{align*}
    where $\|\cdot\|_F$ is the Frobenius norm. Combining the above two relations, we get
    \begin{align*}
        \mathrm{Tr}\left(\Sigma_1+\Sigma_2-2\left(\Sigma_2^{1/2}\Sigma_1\Sigma_2^{1/2}\right)^{1/2}\right)\leq \|\Sigma_1^{1/2}-\Sigma_2^{1/2}\|_F^2.
    \end{align*}
    Finally, to relate the difference of the square roots with the difference of the covariances, we use the following identity
    \begin{align*}
        \Sigma_1-\Sigma_2=\Sigma_1^{1/2}\left(\Sigma_1^{1/2}-\Sigma_2^{1/2}\right)+\left(\Sigma_1^{1/2}-\Sigma_2^{1/2}\right)\Sigma_2^{1/2}.
    \end{align*}
    Using Kronecker product notation and vectorization, we get
    \begin{align*}
        \text{vec}(\Sigma_1-\Sigma_2)=\left(I\otimes\Sigma_1^{1/2}+\Sigma_2^{1/2}\otimes I\right)\text{vec}\left(\Sigma_1^{1/2}-\Sigma_2^{1/2}\right).
    \end{align*}
    Note that Kronecker product terms are symmetric and thus the resultant matrix is symmetric. Thus, we have
    \begin{align*}
        \|\text{vec}(\Sigma_1-\Sigma_2)\|&\geq \lambda_{min}\left(\left(I\otimes\Sigma_1^{1/2}+\Sigma_2^{1/2}\otimes I\right)\right)\|\text{vec}\left(\Sigma_1^{1/2}-\Sigma_2^{1/2}\right)\|\\
        &=(\lambda_{min}(\Sigma_1^{1/2})+\lambda_{min}(\Sigma_2^{1/2}))\|\text{vec}\left(\Sigma_1^{1/2}-\Sigma_2^{1/2}\right)\|
    \end{align*}
    where the equality is obtained by noting that the eigenvectors of $I\otimes\Sigma_1^{1/2}+\Sigma_2^{1/2}\otimes I$ is given by the Kronecker product of eigenvectors of each matrices. The above relation implies
    \begin{align*}
        \|\Sigma_1^{1/2}-\Sigma_2^{1/2}\|_F&\leq \frac{1}{\sqrt{\lambda_{min}(\Sigma_1)}+\sqrt{\lambda_{min}(\Sigma_2)}}\|\Sigma_1-\Sigma_2\|_F\\
        &\leq \frac{\sqrt{d}}{\sqrt{\lambda_{min}(\Sigma_1)}+\sqrt{\lambda_{min}(\Sigma_2)}}\|\Sigma_1-\Sigma_2\|.
    \end{align*}
    The above relation leads us to
    \begin{align*}
         d_{\mathcal{W}_2}(\Sigma_1^{1/2}Z, \Sigma_2^{1/2}Z)&\leq \frac{\sqrt{d}}{\sqrt{\lambda_{min}(\Sigma_1)}+\sqrt{\lambda_{min}(\Sigma_2)}}\|\Sigma_1-\Sigma_2\|.
    \end{align*}
\end{proof}

\begin{lemma}\label{lem:rec_sol}
    Let $\rho_1,\rho_2\geq 0$. Consider the sequence $\{u_k\}_{k\geq 0}$ satisfying the following bound:
    \begin{align*}
        u_{k+1}\leq (1-\mu_1\alpha_k)u_k+\mu_2\mathbbm{1}_{\xi<1}\alpha_k\exp{(-\frac{\rho_1\alpha}{1-\xi}(k+K)^{1-\xi})}+\mu_3\alpha_k^{1+\rho_2}.
    \end{align*}
    where $\mu_1, \mu_2, \mu_3, u_0\geq 0$ and $\alpha_k$ satisfies Assumption \ref{assump:step-size}. Then, the following relations hold:
    \begin{enumerate}
        \item For $\xi=0$ and $\mu_1\alpha\leq 1$, we have
        \begin{align*}
            u_k\leq u_0e^{-\mu_1\alpha k}+\frac{\mu_2e^{\mu_1\alpha}}{\mu_1-\rho_1}\left(e^{-\rho_1\alpha k}-e^{-\mu_1\alpha k}\right)+\frac{\mu_3\alpha^{\rho_2}}{\mu_1}.
        \end{align*}
        \item For $\xi\in(0,1)$ and $K\geq (1/(\mu_1\alpha))^{1/(1-\xi)}$, we have
        \begin{align*}
            u_k&\leq u_0e^{-\frac{\mu_1\alpha}{1-\xi}\left((k+K)^{1-\xi}-K^{1-\xi}\right)}+\frac{\mu_2 e^{\mu_1\alpha}}{\mu_1-\rho_1} \left(e^{-\frac{\rho_1\alpha}{1-\xi}(k+K)^{1-\xi}}-e^{-\frac{\mu_1\alpha}{1-\xi}(k+K)^{1-\xi}}\right)+\frac{2\mu_3\alpha_k^{\rho_2}}{\mu_1}.
        \end{align*}
        \item For $\xi=1$, we have the following three cases:
        \begin{enumerate}
            \item[(a)] When $\mu_1\alpha>\rho_2$, we have
            \begin{align*}
                u_k\leq u_0\left(\frac{K}{k+K}\right)^{\mu_1\alpha}+\frac{\alpha\mu_3\alpha_k^{\rho_2}}{\mu_1\alpha-\rho_2}.
            \end{align*}
            \item[(b)] When $\mu_1\alpha<\rho_2$, we have
            \begin{align*}
                u_k\leq u_0\left(\frac{K}{k+K}\right)^{\mu_1\alpha}+\frac{(2\alpha)^{1+\rho_2}\mu_3\alpha_k^{\mu_1\alpha}}{\alpha^{\mu_1\alpha}(\rho_2-\mu_1\alpha)}.
            \end{align*}
            \item[(c)] When $\mu_1\alpha=\rho_2$, we have
            \begin{align*}
                u_k\leq u_0\left(\frac{K}{k+K}\right)^{\mu_1\alpha}+2^{1+\rho_2}\mu_3\alpha\log(k+K)\alpha_k^{\rho_2}.
            \end{align*}
        \end{enumerate}
    \end{enumerate}
\end{lemma}
\begin{proof}
    Note that we can unroll $u_k$ to re-write it as
    \begin{align*}
        u_k&=u_0\underbrace{\prod_{j=0}^{k-1}(1-\mu_1\alpha_j)}_{T_1}+\mu_3\underbrace{\sum_{j=0}^{k-1}\alpha_j^{1+\rho_2}\prod_{l=j+1}^{k-1}(1-\mu_1\alpha_l)}_{T_2}\\
        &~+\mu_2\underbrace{\sum_{j=0}^{k-1}\alpha_j\exp{\left(-\frac{\rho_1\alpha}{1-\xi}\left((j+K)^{1-\xi}\right)\right)}\prod_{l=j+1}^{k-1}(1-\mu_1\alpha_l)}_{T_3}.
    \end{align*}
    Now we will analyze these terms for each case separately. 
    \begin{enumerate}
        \item For $T_1$, we have
        \begin{align*}
            T_1=(1-\mu_1\alpha)^k\leq e^{-\mu_1\alpha k}.\tag{$e^x\geq 1+x$}
        \end{align*}
        For $T_2$, we have
        \begin{align*}
            T_2=\alpha^{1+\rho_2}\sum_{j=0}^{k-1}(1-\mu_1\alpha)^{k-j-1}\leq \frac{\alpha^{\rho_2}}{\mu_1}.
        \end{align*}
        For $T_3$, we have
        \begin{align*}
            T_3&=\alpha\sum_{j=0}^{k-1}e^{-\rho_1\alpha j}(1-\mu_1\alpha)^{k-j-1}\\
            &\leq \alpha e^{-\mu_1\alpha(k-1)}\sum_{j=0}^{k-1}e^{(\mu_1-\rho_1)\alpha j}\tag{$1-x\leq e^{-x}$}\\
            &=\alpha e^{-\mu_1\alpha(k-1)}\left(\frac{e^{(\mu_1-\rho_1)\alpha k}-1}{e^{(\mu_1-\rho_1)\alpha}-1}\right)\\
            &\leq e^{\mu_1\alpha}\left(\frac{e^{-\rho_1\alpha k}-e^{-\mu_1\alpha k}}{\mu_1-\rho_1}\right)\tag{$\frac{1}{e^x-1}\leq \frac{1}{x}$}.
        \end{align*}
        \item Using the expression for $\alpha_k$ and the fact that $e^x\geq 1+x$, we have
        \begin{align*}
            T_1=\prod_{j=0}^{k-1}\left(1-\mu_1\frac{\alpha}{(k+K)^\xi}\right)\leq \exp\left(-\mu_1\alpha\sum_{j=0}^{k-1}\frac{1}{(k+K)^{\xi}}\right)
        \end{align*}
        Since the function $h(x)=1/(x+K)^{\xi}$ is non-increasing, we use the inequality $\int_a^{b+1}h(x)dx\leq \sum_{j=a}^bh(j)\leq \int_{a-1}^bh(x)dx$, we obtain
        \begin{align*}
            T_1\leq \exp\left(-\frac{\mu_1\alpha}{1-\xi}\left((k+K)^{1-\xi}-K^{1-\xi}\right)\right).
        \end{align*}
        For $T_2$, we define $\{v_k\}_{k\geq0}$ as
        \begin{align*}
            v_0=0;~~~v_{k+1}=(1-\mu_1\alpha_k)v_k+\alpha_k^{1+\rho_2}.
        \end{align*}
        Clearly, $v_k=\sum_{j=0}^{k-1}\alpha_j^{1+\rho_2}\prod_{l=j+1}^{k-1}(1-\mu_1\alpha_l)=T_2$. We will use induction to show that for $K\geq (2/(\mu_1\alpha))^{1/(1-\xi)}$, we have that $v_k\leq 2\alpha_k^{\rho_2}/\mu_1$. The base case, $v_0=0$ satisfies the condition by construction. Suppose that the statement is true for $k$. Then, for $v_{k+1}$, we have
        \begin{align*}
            \frac{2\alpha_{k+1}^{\rho_2}}{\mu_1}-v_{k+1}&=\frac{2\alpha_{k+1}^{\rho_2}}{\mu_1}-(1-\mu_1\alpha_k)v_k-\alpha_k^{1+\rho_2}\\
            &\geq \frac{2\alpha_{k+1}^{\rho_2}}{\mu_1}-(1-\mu_1\alpha_k)\frac{2\alpha_k^{\rho_2}}{\mu_1}-\alpha_k^{1+\rho_2}\\
            &=\frac{2\alpha_{k+1}^{\rho_2}}{\mu_1}-\frac{2\alpha_k^{\rho_2}}{\mu_1}+\alpha_k^{1+\rho_2}\\
            &=\frac{2\alpha^{\rho_2}}{\mu_1}\left(\frac{1}{(k+K+1)^{\xi\rho_2}}-\frac{1}{(k+K)^{\xi\rho_2}}+\frac{\mu_1\alpha}{2}\frac{1}{(k+K)^{\xi(1+\rho)}}\right)\\
            &\geq \frac{2\alpha^{\rho}}{\mu_1}\left(-\frac{1}{(k+K)^{1+\xi\rho_2}}+\frac{\mu_1\alpha}{2}\frac{1}{(k+K)^{\xi(1+\rho_2)}}\right)\tag{Lemma \ref{lem:step-size_prop}}\\
            &=\frac{2\alpha^{\rho_2}}{\mu_1}\frac{1}{(k+K)^{\xi\rho_2}}\left(-\frac{1}{(k+K)}+\frac{\mu_1\alpha}{2}\frac{1}{(k+K)^{\xi}}\right)\\
            &\geq 0.\tag{$K\geq (2/(\mu_1\alpha))^{1/(1-\xi)}$}
        \end{align*}
        For $T_3$, we using the integral bound to get
        \begin{align*}
            T_3&\leq\sum_{j=0}^{k-1}\alpha_je^{-\frac{\rho_1\alpha}{1-\xi}(j+K)^{1-\xi}}e^{-\frac{\mu_1\alpha}{1-\xi}\left((k+K)^{1-\xi}-(j+1+K)^{1-\xi}\right)}\\
            &=e^{-\frac{\mu_1\alpha}{1-\xi}(k+K)^{1-\xi}}\sum_{j=0}^{k-1}\alpha_je^{-\frac{\rho_1\alpha}{1-\xi}(j+K)^{1-\xi}}e^{\frac{\mu_1\alpha}{1-\xi}(j+1+K)^{1-\xi}}.
        \end{align*}
        Using the fact $\left(\frac{j+1+K}{j+K}\right)^{1-\xi}\leq 1+\frac{1-\xi}{j+K}$, we get
        \begin{align*}
            T_3&\leq \alpha e^{-\frac{\mu_1\alpha}{1-\xi}(k+K)^{1-\xi}}e^{\frac{\mu_1\alpha}{K^\xi}}\sum_{j=0}^{k-1}\frac{1}{(j+K)^\xi}e^{\left(\mu_1-\rho_1\right)\frac{\alpha}{1-\xi}(j+K)^{1-\xi}}.
        \end{align*}
        Using the integral bound for the sum, we obtain
        \begin{align*}
            T_3&\leq \alpha e^{-\frac{\mu_1\alpha}{1-\xi}(k+K)^{1-\xi}}e^{\frac{\mu_1\alpha}{K^\xi}}\int_{K-1}^{k+K}\frac{1}{x^\xi}e^{\left(\mu_1-\rho_1\right)\frac{\alpha}{1-\xi}x^{1-\xi}}dx\\
            &\leq \frac{\alpha}{1-\xi} e^{-\frac{\mu_1\alpha}{1-\xi}(k+K)^{1-\xi}}e^{\mu_1\alpha}\int_{(K-1)^{1-\xi}}^{(k+K)^{1-\xi}}e^{\left(\mu_1-\rho_1\right)\frac{\alpha}{1-\xi}t}dt\tag{$t=x^{1-\xi}$}\\
            &\leq \frac{\alpha}{1-\xi} e^{-\frac{\mu_1\alpha}{1-\xi}(k+K)^{1-\xi}}e^{\mu_1\alpha}\int_{0}^{(k+K)^{1-\xi}}e^{\left(\mu_1-\rho_1\right)\frac{\alpha}{1-\xi}t}dt\\
            &= \frac{ e^{\mu_1\alpha}}{\mu_1-\rho_1} \left(e^{-\frac{\rho_1\alpha}{1-\xi}(k+K)^{1-\xi}}-e^{-\frac{\mu_1\alpha}{1-\xi}(k+K)^{1-\xi}}\right).
        \end{align*}
        
        Combining the bounds, we have the claim.
        \item In this case, $T_3=0$. For $T_1$, we proceed in a similar fashion as in the previous part.
        \begin{align*}
            T_1=\prod_{j=0}^{k-1}\left(1-\mu_1\frac{\alpha}{(k+K)}\right)\leq \exp\left(-\mu_1\alpha\sum_{j=0}^{k-1}\frac{1}{(k+K)}\right).
        \end{align*}
        Bounding the summation with the integral, gives us
        \begin{align*}
            T_1\leq \left(\frac{K}{k+K}\right)^{\mu_1\alpha}.
        \end{align*}
        For $T_2$, we split the analysis into three cases:
        \begin{enumerate}
            \item[(a)] $\mu_1\alpha>\rho_2:$ We use exact same steps we used for induction in $\xi\in (0,1)$ case to obtain $T_2\leq \alpha\alpha_k^{\rho_2}/(\mu_1\alpha-\rho_2)$ for $k\geq 0$.
            \item[(b)] $\mu_1\alpha<\rho_2:$ For this case, we take a different approach. We use the integral bound to get
            \begin{align*}
                T_2&\leq \sum_{j=0}^{k-1}\alpha_j^{1+\rho_2}\exp\left(-\mu_1\alpha\sum_{l=j+1}^{k-1}\frac{1}{(l+K)}\right)\\
                &\leq \alpha^{1+\rho_2}\sum_{j=0}^{k-1}\frac{1}{(j+K)^{1+\rho_2}}\left(\frac{j+1+K}{k+K}\right)^{\mu_1\alpha}\\
                &\leq \frac{(2\alpha)^{1+\rho_2}}{(k+K)^{\mu_1\alpha}}\sum_{j=0}^{k-1}\frac{1}{(j+1+K)^{1+\rho_2-\mu_1\alpha}}\tag{$\left(1+\frac{1}{j+K}\right)^{1+\rho_2}\leq 2^{1+\rho_2}$}\\
                &\leq \frac{(2\alpha)^{1+\rho_2}}{\rho_2-\mu_1\alpha}\frac{1}{(k+K)^{\mu_1\alpha}}.
            \end{align*}
            \item[(c)] $\mu_1\alpha=\rho_2:$ For the final case, we again use the integral bound from previous case to get
            \begin{align*}
                T_2&\leq \frac{(2\alpha)^{1+\rho_2}}{(k+K)^{\rho_2}}\sum_{j=0}^{k-1}\frac{1}{j+1+K}\\
                &\leq \frac{(2\alpha)^{1+\rho_2}\log(k+K)}{(k+K)^{\rho_2}}.
            \end{align*}
        \end{enumerate}
        
        \end{enumerate}
\end{proof}

\begin{lemma}\label{lem:step-size_prop}
Let $\rho\geq 0$. Then, step-size sequence in Assumption \ref{assump:step-size} has following properties:
    \begin{align*}
        \alpha_k\leq \alpha_{k-1};&~~\alpha_{k-1}\leq 2\alpha_k;~~\frac{1}{(k+K)^{\rho}}-\frac{1}{(k+K+1)^{\rho}}\leq \frac{\rho}{(k+K)^{1+\rho}};\\
        &\left|\frac{\sqrt{\alpha_k}}{\sqrt{\alpha_{k+1}}}-1-\frac{\xi}{2(k+K)}\right|\leq \frac{\xi}{4}\left(1-\frac{\xi}{2}\right)\frac{1}{(k+K)^2}.
    \end{align*}
\end{lemma}
\begin{proof}
    The step-size is non-decreasing by construction. Now consider the ratio 
    \begin{align*}
        \frac{\alpha_{k-1}}{\alpha_k}&=\left(\frac{K+k}{K+k-1}\right)^{\xi}\\
        &=\left(1+\frac{1}{K+k-1}\right)^{\xi}.
    \end{align*}
    Since $k\geq 0$ we have $\frac{1}{K+k-1}\leq \frac{1}{K}\leq 1$. Combining together with the expression above, we get
    \begin{align*}
        \frac{\alpha_{k-1}}{\alpha_k}\leq 2^\xi\leq 2.\tag{$\xi\leq 1$}
    \end{align*}
    For the third part, consider the function $f(x)=\frac{1}{(k+x)^{\rho}}$ for $x\in [0,1]$ and $k\geq 0$. Using Taylor's series expansion, there exists $z\in [x, 1]$ such that we have
    \begin{align*}
        f(1)&=f(x)+(1-x)f'(z)\\
        &=f(x)-\frac{(1-x)\rho}{(k+z)^{1+\rho}}\\
        f(x)-f(1)&= \frac{(1-x)\rho}{(k+z)^{1+\rho}}\\
        &\leq \frac{(1-x)\rho}{k^{1+\rho}}\tag{$z\geq 0$}
    \end{align*}
    $x=0$, we have $\forall~k\geq 0$
    \begin{align*}
        \frac{1}{k^{\rho}}-\frac{1}{(k+1)^{\rho}}\leq \frac{\rho}{k^{1+\rho}}.
    \end{align*}
    For the final part, we use Taylor series up to second order for the function $f(x)=(1+x)^{\xi/2}$ given as
    \begin{align*}
        f(x)&=f(0)+f'(0)x+\frac{1}{2}f''(z)x^2.\tag{$z\in [0,x]$}\\
        &=1+\frac{\xi}{2}x+\frac{\xi}{4}\left(\frac{\xi}{2}-1\right)(1+z)^{\xi/2-2}x^2.
    \end{align*}
    Note $f''(z)$ is a decreasing function. Thus,
    \begin{align*}
       \left|f(x)-1-\frac{\xi}{2}x\right|&\leq  \frac{\xi}{4}\left(1-\frac{\xi}{2}\right)x^2.
    \end{align*}
    Substituting $x=\frac{1}{k+K}$ in the above expansion leads us to
    \begin{align*}
        \left|\frac{\sqrt{\alpha_k}}{\sqrt{\alpha_{k+1}}}-1-\frac{\xi}{2(k+K)}\right|&=\left|\left(1+\frac{1}{k+K}\right)^{\xi/2}-1-\frac{\xi}{2(k+K)}\right|\\
        &=\frac{\xi}{4}\left(1-\frac{\xi}{2}\right)\frac{1}{(k+K)^2}.
    \end{align*}
\end{proof}

\subsection{Proof of Lemma \ref{lem:derivative_bound}}\label{appendix:derivative_bound}
Recall that the Lemma is proven for the special case of $\Sigma_X=I$ by \cite{Gallouet2018}. Thus, we perform a change of variables given as $Y=\Sigma_X^{1/2}Z$ and define $\tilde{h}(x):=h(\Sigma_X^{1/2}x)$, where $h\in Lip_L$. Note that this implies $\tilde{h}\in Lip_{L\|\Sigma^{1/2}\|}$. Now, suppose $\tilde{f}(x)$ solves the Gaussian-Stein equation for $\tilde{h}$ with identity covariance matrix,
\begin{align*}
        \tilde{h}(x)-\mathbb{E}[\tilde{h}(Z)]&=\operatorname{Tr}(\nabla^2\tilde{f}(x))-x^T\nabla \tilde{f}(x).
\end{align*}
Now set $f(y):=\tilde{f}(\Sigma_X^{-1/2}y)$ and let $x:=\Sigma_X^{-1/2}y$, then we have 
\begin{align*}
    h(y)-\mathbb{E}[h(Y)]&=\tilde{h}(\Sigma^{-\frac{1}{2}}y)-\mathbb{E}[\tilde{h}(Z)]\\
    &= \tilde{h}(x) - \mathbb{E}[\tilde{h}(Z)]\\
    &=\operatorname{Tr}(\nabla^2\tilde{f}(x))-x^T\nabla \tilde{f}(x)\\
    &\overset{(a)}{=}\operatorname{Tr}(\Sigma_X\nabla^2f(y))-y^T\nabla f(y).
\end{align*}
The equality $(a)$ holds from the following chain rule,
\begin{align*}
    &\nabla f(y) =\Sigma_X^{-\frac{1}{2}}\nabla\tilde{f}(\Sigma_X^{-\frac{1}{2}}y)=\Sigma_X^{-\frac{1}{2}}\nabla\tilde{f}(x),\quad
    \nabla^2f(y) = (\Sigma_X^{-\frac{1}{2}})^T(\nabla^2\tilde{f}(x))\Sigma_X^{-\frac{1}{2}}
\end{align*}
The above implies that $f(y)$ is the solution to Gaussian-Stein equation for $Y\sim \mathcal{N}(0, \Sigma_X)$. 
By \cite{Gallouet2018}[Proposition 2.2], we have
    \begin{align*}
        \| \nabla^2 \tilde{f}(x)-\nabla^2 \tilde{f}(y)\|&\leq \left(\widetilde C_{1}(d)+\frac{2}{1-\beta}\right)\|x-y\|^\beta \sup_{z\in\mathbb{R}^d}\|\nabla \tilde{h}(z)\|.
    \end{align*}
    Also, note that for $\nabla^2f$ and $\nabla^2\tilde{f}$, we have
    \begin{align*}
        \| \nabla^2 f(x)-\nabla^2 f(y)\|&\leq \|\Sigma_X^{-\frac{1}{2}}\|^2 \|\nabla^2 \tilde{f}(\Sigma_X^{-\frac{1}{2}}x)-\nabla^2 \tilde{f}(\Sigma_X^{-\frac{1}{2}}y)\|.
    \end{align*}
    The above bounds lead us to the following relation:
    \begin{align*}
        \| \nabla^2 f(x)-\nabla^2 f(y)\| &\leq \|\Sigma_X^{-\frac{1}{2}}\|^2 \left(\tilde{C}_1(d)+\frac{2}{1-\beta}\right)\|\Sigma_X^{-\frac{1}{2}}(x-y)\|^\beta \|\tilde{h}\|_{Lip}\\
        &\leq \left(\tilde{C}_1(d)+\frac{2}{1-\beta}\right)L\|\Sigma_X^{\frac{1}{2}}\|\|\Sigma_X^{-\frac{1}{2}}\|^{2+\beta}\|x-y\|^\beta\\
        &\leq  \left(\tilde{C}_1(d)+\frac{2}{1-\beta}\right)L\sigma_{max}\|x-y\|^\beta.
    \end{align*}

\subsection{Lemmas for general Hurwitz matrix}
\begin{lemma}\label{lem:theta_bound}
    Let $\Theta_i = \sqrt{\alpha_i} \prod_{l = i+1}^k (I+\alpha_l J_l^{(\alpha, \xi)}) \text{ and } \Theta_k=\sqrt{\alpha_k}$. Then, we have the following bounds:
    \begin{enumerate}
        \item When $\xi=0$ and $\alpha\leq \min(1, 2\iota_V/\|J\|^2_V)$, we have
        \begin{align*}
            \sum_{i=0}^{k-1}\|\Theta_i\|^{2+\beta}\leq \left(\frac{\lambda^{max}_V}{\lambda^{min}_V}\right)^{3/2}\frac{\alpha^{\beta/2}}{\iota_V}.
        \end{align*}
        \item When $\xi\in (0,1)$ and $K$ is chosen large enough such that $K\geq (1/(\iota_V\alpha))^{1/(1-\xi)}$, $\alpha_k\leq \min(1, 2\iota_V/\|J\|^2_V)$, and $\alpha^{-1}\xi K^{\xi-1}/2\leq \min(1, \iota_V/3)$, we have
       \begin{align*}
            \sum_{i=0}^{k-1}\|\Theta_i\|^{2+\beta}\leq 2\left(\frac{\lambda^{max}_V}{\lambda^{min}_V}\right)^{3/2}\frac{\alpha_k^{\beta/2}}{\iota_V}.
        \end{align*}
        \item When $\xi=1$, $\iota_V\alpha> 1/2$, and $K$ is chosen large enough such that $\alpha_k\leq 4\iota_V/(4\|J\|_V^2+\alpha^{-2})$, we have 
        \begin{align*}
            \sum_{i=0}^{k-1}\|\Theta_i\|^{2+\beta}\leq \left(\frac{\lambda^{max}_V}{\lambda^{min}_V}\right)^{3/2}\frac{2\alpha\alpha_k^{\beta/2}}{2\iota_V\alpha-1}.
        \end{align*}
    \end{enumerate}
\end{lemma}
\begin{proof}
For $\xi<1$, note that $\|\Theta_i\|^{2+\beta}\leq \alpha_i^{(2+\beta)/2}(\lambda^{max}_V/\lambda^{min}_V)^{3/2}\prod_{l=i+1}^{k-1}\|I+\alpha_lJ_l^{(\alpha, \xi)}\|_V^{2+\beta}$. Recall 
    \begin{align*}
        J_l^{(\alpha, \xi)}=J+\frac{\alpha^{-1}\xi}{2(l+K)^{1-\xi}}I.
    \end{align*}
    Using the first statement of Lemma \ref{lem:contraction_prop}  by setting $\frac{\alpha^{-1}\xi}{2(l+K)^{1-\xi}}=\epsilon_2$, the bound simplifies to 
    \begin{align*}
        \|\Theta_i\|^{2+\beta}\leq \alpha_i^{(2+\beta)/2}(\lambda^{max}_V/\lambda^{min}_V)^{3/2}\prod_{l=i+1}^{k-1}(1-\iota_V\alpha_l).
    \end{align*} 
    Consider the sequence $\{v_k\}_{k\geq 0}$
    \begin{align*}
        v_0=0;~~v_{k+1}=(1-\iota_V\alpha_l)v_k+\left(\frac{\lambda^{max}_V}{\lambda^{min}_V}\right)^{3/2}\alpha_i^{1+\beta/2}.
    \end{align*}
    Clearly, $v_k\geq \sum_{i=0}^{k-1} \|\Theta_i\|^{2+\beta}$. Applying Lemma \ref{lem:rec_sol} to $v_k$ by setting $u_0=0$, $\mu_1=\iota_V$, $\mu_2=0$, $\mu_3=(\lambda^{max}_V/\lambda^{min}_V)^{3/2}$, and $\rho_2=\beta/2$, we obtain the claim.
    
    For $\xi=1$, we proceed with the same steps as before but use statement 2 of Lemma \ref{lem:contraction_prop}  to obtain the bounds.
    
\end{proof}

\begin{lemma}\label{lem:bound_sigma_k}
    Let $\{\tilde{\Sigma}_k^{(\alpha, \xi)}\}_{k\geq 0}$ be given by Eq. \eqref{eq:sigma_k_eq_pf}. Assume that for each specific choice of $\xi$, all the conditions on the step-sizes stated in Lemma \ref{lem:theta_bound} hold. Then, we have the following bounds for all $k\geq 0$:
    \begin{align*}
        \left\|\tilde{\Sigma}_k^{(\alpha, \xi)}\right\|\leq \sqrt{\frac{\lambda^{max}_V}{\lambda^{min}_V}}\frac{\|\Sigma_b\|_V}{\iota_V}.
    \end{align*}
\end{lemma}
\begin{proof}
    For $\xi<1$, taking matrix norm both sides and applying statement 1 of Lemma \ref{lem:contraction_prop}, we obtain
    \begin{align*}
        \|\tilde{\Sigma}_{k}^{(\alpha, \xi)}\|_V&\leq (1 - \iota_V\alpha_{k-1})\|\tilde{\Sigma}_{k-1}^{(\alpha, \xi)}\|_V+\alpha_{k-1}\|\Sigma_b\|_V\\
        \implies \|\tilde{\Sigma}_{k}^{(\alpha, \xi)}\|_V&\leq (1 - \iota_V\alpha_{k-1})\|\tilde{\Sigma}_{k-1}^{(\alpha, \xi)}\|_V+\iota_V\alpha_{k-1}\frac{\|\Sigma_b\|_V}{\iota_V}.
    \end{align*}
    It is straightforward to show that $\|\tilde{\Sigma}_k^{(\alpha, \xi)}\|_V\leq \Sigma_b/\iota_V$ using induction. For $\xi=1$, we use statement 2 of Lemma \ref{lem:contraction_prop} to obtain the bounds.
\end{proof}

\begin{lemma}\label{lem:matrix_theta_bound}
    Let $\{\Upsilon_k^{(\alpha, \xi)}\}_{k\geq 0}$ be given by 
    \begin{align*}
        \Upsilon_{k+1}^{(\alpha, \xi)}=(I + \alpha_k J^{(\alpha, \xi)}_k)\Upsilon_k^{(\alpha, \xi)}(I + \alpha_k J^{(\alpha, \xi)}_k)^T+\alpha_k\Sigma_b
    \end{align*}
    where $\Upsilon_0^{(\alpha, \xi)}$ is arbitrary. Assume that for each specific choice of $\xi$, all the conditions on the step-sizes stated in Lemma \ref{lem:theta_bound} hold. Then, we have the following bounds for all $k\geq 0$:
    \begin{enumerate}
        \item When $\xi=0$ and $\alpha\leq \min(1, 2\iota_V/\|J\|^2_V)$, we have
        \begin{align*}
            \left\|\Upsilon_k^{(\alpha, 0)}-\Sigma\right\|\leq \sqrt{\frac{\lambda^{max}_V}{\lambda^{min}_V}}\left(\|\Upsilon_0^{(\alpha, 0)}-\Sigma\|_Ve^{-\iota_V\alpha k}+\frac{\|J_F\|_V^2\|\Sigma\|_V\alpha}{\iota_V}\right).
        \end{align*}
        \item When $\xi\in (0,1)$ and $K$ is chosen large enough such that $K\geq (1/(\iota_V\alpha))^{1/(1-\xi)}$, $\alpha_k\leq \min(1, 2\iota_V/\|J\|^2_V)$, and $\alpha^{-1}\xi K^{\xi-1}/2\leq \min(1, \iota_V/3)$, we have
       \begin{align*}
            \left\|\Upsilon_k^{(\alpha, \xi)}-\Sigma\right\|&\leq \sqrt{\frac{\lambda^{max}_V}{\lambda^{min}_V}}\Bigg(\|\Upsilon_0^{(\alpha, \xi)}-\Sigma\|_Ve^{-\frac{\iota_V\alpha}{1-\xi}\left((k+K)^{1-\xi}-K^{1-\xi}\right)}+\frac{2\left(\|J_F\|_V+\frac{\alpha^{-1}}{2}\right)^2\|\Sigma\|_V\alpha_k}{\iota_V}\\
            &~+\frac{2\alpha^{-1}\|\Sigma\|_V}{(k+K)^{1-\xi}\iota_V}\Bigg).
        \end{align*}
        \item When $\xi=1$, $\iota_V\alpha> 1$, and $K$ is chosen large enough such that $\alpha_k\leq 4\iota_V/(4\|J\|_V^2+\alpha^{-2})$, we have 
        \begin{align*}
            \left\|\Upsilon_k^{(\alpha, 1)}-\Sigma^{(\alpha)}\right\|\leq \sqrt{\frac{\lambda^{max}_V}{\lambda^{min}_V}}\left(\|\Upsilon_k^{(\alpha, 1)}-\Sigma^{(\alpha)}\|_V\left(\frac{K}{k+K}\right)^{\iota_V\alpha}+\frac{\alpha\|\left(\|J_F\|_V+\frac{\alpha^{-1}}{2}\right)^2\|\Sigma^{(\alpha)}\|_V\alpha_k}{\iota_V\alpha-1}\right).
        \end{align*}
    \end{enumerate}
\end{lemma}
\begin{proof}
    Recall that $\Upsilon_k^{(\alpha, \xi)}$ is given by
    \begin{align*}
        \Upsilon_{k+1}^{(\alpha, \xi)}&=(I + \alpha_k J^{(\alpha, \xi)}_k)\Upsilon_k^{(\alpha, \xi)}(I + \alpha_{k} J^{(\alpha, \xi)}_k)^T+\alpha_{k}\Sigma_b\\
        &=\sum_{i=0}^{k}\Theta_i\Sigma_b\Theta_i^T\\
        \implies \Upsilon_{k+1}^{(\alpha, \xi)}-\Sigma^{(\alpha, \xi)}&=(I + \alpha_{k} J^{(\alpha, \xi)}_k)(\Upsilon_k^{(\alpha, \xi)}-\Sigma^{(\alpha, \xi)})(I + \alpha_{k} J^{(\alpha, \xi)}_k)^T+\alpha_{k}^2J^{(\alpha, \xi)}_k\Sigma^{(\alpha, \xi)} (J^{(\alpha, \xi)}_k)^T\\
        &~+\frac{\xi\mathbbm{1}_{\xi<1}}{(k+K)}\Sigma^{(\alpha, \xi)}.
    \end{align*}
    where for the last equation, we used the fact that $\Sigma^{(\alpha, \xi)}$ is the solution to Lyapunov equation \eqref{eq:lyap_eq}. 
    \begin{enumerate}
        \item $\xi=0$: In this case, the last term in the above equation is 0. Thus, using the statement 1 Lemma \ref{lem:contraction_prop}, we have
        \begin{align*}
            \|\Upsilon_{k+1}^{(\alpha, 0)}-\Sigma\|_V&\leq (1-\iota_V\alpha)\|\Upsilon_k^{(\alpha, 0)}-\Sigma\|_V+\alpha^2\|J_F\Sigma J_F^T\|_V.
        \end{align*}
        Now it is easy to apply Lemma \ref{lem:rec_sol} by setting $u_0=\|\Upsilon_0^{(\alpha, \xi)}-\Sigma\|_V, \mu_1=\iota_V$, $\mu_2=0$, $\mu_3=\|J_F\|_V^2\|\Sigma\|_V$, and $\rho_2=1$.
        \item $\xi\in (0, 1)$: In this case, the last term in non-zero and needs to be considered in the analysis. Again using statement 1 of Lemma \ref{lem:contraction_prop}, we have
        \begin{align*}
            \|\Upsilon_{k+1}^{(\alpha, \xi)}-\Sigma\|_V&\leq (1-\iota_V\alpha_{k})\|\Upsilon_k^{(\alpha, \xi)}-\Sigma\|_V+\alpha_{k}^2\|J^{(\alpha, \xi)}_k\Sigma (J^{(\alpha, \xi)}_k)^T\|_V+\alpha_k^{\frac{1}{\xi}}\alpha^{-1}\|\Sigma\|_V.
        \end{align*}
         Now we just need to apply Lemma \ref{lem:rec_sol} twice by setting $\rho_2=1$ and $\rho_2=1/\xi-1$ and then combine the upper bounds to get the claim.
        
        \item $\xi=1$: In this case, the last term is also 0. Thus, using the statement 2 Lemma \ref{lem:contraction_prop}, we have
        \begin{align*}
            \|\Upsilon_{k+1}^{(\alpha, 1)}-\Sigma^{(\alpha)}\|_V&\leq (1-\iota_V\alpha_{k})\|\Upsilon_k^{(\alpha, 1)}-\Sigma^{(\alpha)}\|_V+\alpha_{k}^2\|J_F^{(\alpha)}\Sigma^{(\alpha)} (J_F^{(\alpha)})^T\|_V.
        \end{align*}
        We again apply Lemma \ref{lem:rec_sol} by setting the appropriate constants to get the claim.
    \end{enumerate}
\end{proof}

\begin{lemma}\label{lem:sigmak_and_hatsigmak}
    Let $\{\Sigma_k^{(\alpha, \xi)}\}_{k\geq 0}$ be given by Eq. \eqref{eq:sigma_k_eq} and $\{\hat{\Sigma}_k^{(\alpha, \xi)}\}_{k\geq 0}$ be given by the solution the Lyapunov equation \eqref{eq:time_varying_lyap}. Assume that for each specific choice of $\xi$, all the conditions on the step-sizes stated in Lemma \ref{lem:theta_bound} hold. Then, we have the following bounds for all $k\geq 0$:
    \begin{enumerate}
        \item When $\xi\in (0,1)$ and $K$ is chosen large enough such that $K\geq (1/(\iota_V\alpha))^{1/(1-\xi)}$, $\alpha_k\leq \min(1, 2\iota_V/\|J\|^2_V)$, and $\alpha^{-1}\xi K^{\xi-1}/2\leq \min(1, \iota_V/3)$, we have
       \begin{align*}
            \left\|\Sigma_k^{(\alpha, \xi)}-\hat{\Sigma}_k^{(\alpha, \xi)}\right\|&\leq \sqrt{\frac{\lambda^{max}_V}{\lambda^{min}_V}}\Bigg(\|\tilde{\Sigma}_0^{(\alpha, \xi)}-\hat{\Sigma}_0^{(\alpha, \xi)})\|_Ve^{-\frac{\iota_V\alpha}{1-\xi}\left((k+K)^{1-\xi}-K^{1-\xi}\right)}+\sqrt{\frac{\lambda^{max}_V}{\lambda^{min}_V}}\frac{2\left(\|J_F\|_V+\frac{\alpha^{-1}}{2}\right)^2\hat{\sigma}_{max}^{(\alpha, \xi)}\alpha_k}{\iota_V}\\
            &~+\sqrt{\frac{\lambda^{max}_V}{\lambda^{min}_V}}\frac{2\alpha^{-2}\xi(1-\xi)\psi_{max}^{(\alpha, \xi)}}{(k+K)^{2-2\xi}\iota_V}\Bigg).
        \end{align*}
    \end{enumerate}
\end{lemma}
\begin{proof}
    Recall that $\tilde{\Sigma}_{k}^{(\alpha, \xi)}$ is given by
    \begin{align*}
        \tilde{\Sigma}_{k+1}^{(\alpha, \xi)}&=(I + \alpha_k J^{(\alpha, \xi)}_k)\tilde{\Sigma}_{k}^{(\alpha, \xi)}(I + \alpha_{k} J^{(\alpha, \xi)}_k)^T+\alpha_{k}\Sigma_b\\
        \implies \tilde{\Sigma}_{k+1}^{(\alpha, \xi)}-\hat{\Sigma}^{(\alpha, \xi)}_{k+1}&=(I + \alpha_{k} J^{(\alpha, \xi)}_k)(\tilde{\Sigma}_{k}^{(\alpha, \xi)}-\hat{\Sigma}_k^{(\alpha, \xi)})(I + \alpha_{k} J^{(\alpha, \xi)}_k)^T+\alpha_{k}^2J^{(\alpha, \xi)}_k\hat{\Sigma}_k^{(\alpha, \xi)} (J^{(\alpha, \xi)}_k)^T\\
        &~+\hat{\Sigma}^{(\alpha, \xi)}_k-\hat{\Sigma}^{(\alpha, \xi)}_{k+1}.
    \end{align*}
    where for the last equation, we used the fact that $\hat{\Sigma}^{(\alpha, \xi)}_k$ is the solution to Lyapunov equation \eqref{eq:time_varying_lyap}. Taking $V$-weighted norm both sides and using statement 1 of Lemma \ref{lem:contraction_prop}, we get
    \begin{align*}
        \|\tilde{\Sigma}_{k+1}^{(\alpha, \xi)}-\hat{\Sigma}^{(\alpha, \xi)}_{k+1}\|_V&\leq (1-\iota_V\alpha)\|\tilde{\Sigma}_{k}^{(\alpha, \xi)}-\hat{\Sigma}_k^{(\alpha, \xi)})\|_V+\alpha_{k}^2\|J^{(\alpha, \xi)}_k\|_V^2\|\hat{\Sigma}_k^{(\alpha, \xi)}\|_V+\|\hat{\Sigma}^{(\alpha, \xi)}_k-\hat{\Sigma}^{(\alpha, \xi)}_{k+1}\|_V\\
        &\leq (1-\iota_V\alpha)\|\tilde{\Sigma}_{k}^{(\alpha, \xi)}-\hat{\Sigma}_k^{(\alpha, \xi)})\|_V+\sqrt{\frac{\lambda^{max}_V}{\lambda^{min}_V}}\alpha_{k}^2\left(\|J_F\|_V+\frac{\alpha^{-1}}{2}\right)^2\hat{\sigma}_k^{(\alpha, \xi)}\\
        &~+\sqrt{\frac{\lambda^{max}_V}{\lambda^{min}_V}}\|\hat{\Sigma}^{(\alpha, \xi)}_k-\hat{\Sigma}^{(\alpha, \xi)}_{k+1}\|\\
        &\leq (1-\iota_V\alpha)\|\tilde{\Sigma}_{k}^{(\alpha, \xi)}-\hat{\Sigma}_k^{(\alpha, \xi)})\|_V+\sqrt{\frac{\lambda^{max}_V}{\lambda^{min}_V}}\left(\alpha_{k}^2\left(\|J_F\|_V+\frac{\alpha^{-1}}{2}\right)^2\hat{\sigma}_{max}^{(\alpha, \xi)}+\frac{\alpha^{-1}\xi(1-\xi)}{(k+K)^{2-\xi}}\psi_{max}^{(\alpha, \xi)}\right),
    \end{align*}
    where for the last inequality, we used Lemma \ref{lem:lyap_sol_diff}. Applying Lemma \ref{lem:rec_sol}, we get the claim.
\end{proof}

\begin{lemma}\label{lem:lyap_sol_diff}
    Let $\{\Sigma_k^{(\alpha, \xi)}\}_{k\geq 0}$ be given by Eq. \eqref{eq:sigma_k_eq} and $\{\hat{\Sigma}_k^{(\alpha, \xi)}\}_{k\geq 0}$ be given by the solution the Lyapunov equation \eqref{eq:time_varying_lyap}. The following relation holds:
    \begin{align*}
        \|\hat{\Sigma}^{(\alpha, \xi)}_{k+1}-\hat{\Sigma}^{(\alpha, \xi)}_k\|&\leq \frac{\alpha^{-1}\xi(1-\xi)}{(k+K)^{2-\xi}}\psi_{max}^{(\alpha, \xi)}.
    \end{align*}
\end{lemma}
\begin{proof}
    Recall that $\hat{\Sigma}^{(\alpha, \xi)}_k$ and $\hat{\Sigma}^{(\alpha, \xi)}_{k+1}$ satisfy the following Lyapunov equations:
    \begin{align*}
        J^{(\alpha, \xi)}_k\hat{\Sigma}^{(\alpha, \xi)}_k +\hat{\Sigma}^{(\alpha, \xi)}_k(J^{(\alpha, \xi)}_k)^T+\Sigma_b&=0\\
        J^{(\alpha, \xi)}_{k+1}\hat{\Sigma}^{(\alpha, \xi)}_{k+1} +\hat{\Sigma}^{(\alpha, \xi)}_{k+1}(J^{(\alpha, \xi)}_{k+1})^T+\Sigma_b&=0.
    \end{align*}
    Subtracting both the equations, we get 
    \begin{align*}
        J^{(\alpha, \xi)}_{k+1}(\hat{\Sigma}^{(\alpha, \xi)}_{k+1}-\hat{\Sigma}^{(\alpha, \xi)}_k) +(\hat{\Sigma}^{(\alpha, \xi)}_{k+1}-\hat{\Sigma}^{(\alpha, \xi)}_k)(J^{(\alpha, \xi)}_{k+1})^T+(J^{(\alpha, \xi)}_k-J^{(\alpha, \xi)}_{k+1})\hat{\Sigma}^{(\alpha, \xi)}_k+\hat{\Sigma}^{(\alpha, \xi)}_k(J^{(\alpha, \xi)}_k-J^{(\alpha, \xi)}_{k+1})^T&=0\\
        J^{(\alpha, \xi)}_{k+1}(\hat{\Sigma}^{(\alpha, \xi)}_{k+1}-\hat{\Sigma}^{(\alpha, \xi)}_k) +(\hat{\Sigma}^{(\alpha, \xi)}_{k+1}-\hat{\Sigma}^{(\alpha, \xi)}_k)(J^{(\alpha, \xi)}_{k+1})^T+\left(\frac{\alpha^{-1}\xi}{(k+K)^{1-\xi}}-\frac{\alpha^{-1}\xi}{(k+K+1)^{1-\xi}}\right)\hat{\Sigma}^{(\alpha, \xi)}_k&=0.
    \end{align*}
    From the above equation it is easy to verify that $\hat{\Sigma}^{(\alpha, \xi)}_{k+1}-\hat{\Sigma}^{(\alpha, \xi)}_k$ is the solution to the Lyapunov equation. Note that the solution to the Lyapunov equation is given by
    \begin{align*}
        \hat{\Sigma}^{(\alpha, \xi)}_{k+1}-\hat{\Sigma}^{(\alpha, \xi)}_k&=\left(\frac{\alpha^{-1}\xi}{(k+K)^{1-\xi}}-\frac{\alpha^{-1}\xi}{(k+K+1)^{1-\xi}}\right)\int_0^{\infty} e^{J^{(\alpha, \xi)}_{k+1}\tau}\hat{\Sigma}^{(\alpha, \xi)}_k e^{(J^{(\alpha, \xi)}_{k+1})^T\tau}d\tau\\
        &=\left(\frac{\alpha^{-1}\xi}{(k+K)^{1-\xi}}-\frac{\alpha^{-1}\xi}{(k+K+1)^{1-\xi}}\right)\Phi_k^{(\alpha, \xi)}.
    \end{align*}
    Taking norm both sides, we get
    \begin{align*}
        \|\hat{\Sigma}^{(\alpha, \xi)}_{k+1}-\hat{\Sigma}^{(\alpha, \xi)}_k\|&\leq \left|\frac{\alpha^{-1}\xi}{(k+K)^{1-\xi}}-\frac{\alpha^{-1}\xi}{(k+K+1)^{1-\xi}}\right|\|\Phi_k^{(\alpha, \xi)}\|\\
        &\leq \frac{\alpha^{-1}\xi(1-\xi)}{(k+K)^{2-\xi}}\|\Phi_k^{(\alpha, \xi)}\|\tag{Using Lemma \ref{lem:step-size_prop} with $\rho=1-\xi$}\\
        &\leq \frac{\alpha^{-1}\xi(1-\xi)}{(k+K)^{2-\xi}}\psi_{max}^{(\alpha, \xi)}.
    \end{align*}
\end{proof}

\begin{lemma}\label{lem:contraction_prop}
Denote $V$ as the solution to Lyapunov equation,
$$J^T V+VJ+I=0.$$
Define $\iota_V:=1/(4\|V\|)$. Then, we have
\begin{enumerate}
    \item When $\xi<1$, then for all $\epsilon_1\in \left[0, \min\left(1, 2\iota_V/\|J\|^2_V\right)\right]$ and $\epsilon_2\in \left[0, \min\left(1, \iota_V/3\right)\right]$
    \begin{align*}
        \|I+\epsilon_1 J+\epsilon_1\epsilon_2I\|_V^2\leq \left(1-\iota_V\epsilon_1\right).
    \end{align*}
    \item When $\xi=1$ and $2\iota_V\alpha\geq 1$, then for all $\epsilon\in [0, 4\iota_V/(4\|J\|_V^2+\alpha^{-2})]$
    \begin{align*}
        \|I+\epsilon J^{(\alpha, 1)}\|_V^2\leq \left(1-\iota_V\epsilon\right).
    \end{align*}
\end{enumerate}
\end{lemma}
\begin{proof}
    \begin{enumerate}
        \item Using the definition of matrix norm we have:
        \begin{align*}
            \|I+\epsilon_1 J+\epsilon_1\epsilon_2I\|_V^2&=\max_{\|x\|_{V}= 1}x^T((1+\epsilon_1\epsilon_2)I+\epsilon_1 J)^T V((1+\epsilon_1\epsilon_2)I+\epsilon_1 J)x\\
            &=\max_{\|x\|_{V}= 1} \big((1+\epsilon_1\epsilon_2)^2x^T Vx+(1+\epsilon_1\epsilon_2)\epsilon_1 x^T(J^T V+VJ)x\\
            &~~+\epsilon_1^2x^T J^TVJx\big)\\
            &\leq (1+\epsilon_1\epsilon_2)^2-(1+\epsilon_1\epsilon_2)\epsilon_1\min_{\|x\|_V=1}\|x\|^2+\epsilon_1^2 \max_{\|x\|_V=1}\|Jx\|_V^2\\
            &\leq 1+3\epsilon_1\epsilon_2-(1+\epsilon_1\epsilon_2)\epsilon_1\frac{1}{\|V\|}+\epsilon_1^2\|J\|_V^2.\tag{$\epsilon_1, \epsilon_2\leq 1$}
        \end{align*}
        For $\epsilon_1\in\left[0, \frac{1}{2\|V\|\|J\|_V^2}\right]$ and using the fact that $\epsilon_1, \epsilon_2>0$, we have:
        \begin{align*}
            \|I+\epsilon_1 J+\epsilon_1\epsilon_2I\|_V^2&\leq 1+3\epsilon_1\epsilon_2-\epsilon_1\frac{1}{2\|V\|}\\
            &\leq 1-\epsilon_1\frac{1}{4\|V\|}.
        \end{align*}
        \item For $\xi=1$, we have
        \begin{align*}
            \|I+\epsilon J^{(\alpha, 1)}\|_V^2&=\left\|\left(1+\frac{\epsilon\alpha^{-1}}{2}\right)I+\epsilon J^{(\alpha)}\right\|_V^2\\
            &=\max_{\|x\|_{V}= 1} \Bigg(\left(1+\frac{\epsilon\alpha^{-1}}{2}\right)^2x^T Vx+\left(1+\frac{\epsilon\alpha^{-1}}{2}\right)\epsilon x^T(J^T V+VJ)x\\
            &~~+\epsilon^2x^T J^TVJx\Bigg)\\
            &\leq \left(1+\frac{\epsilon\alpha^{-1}}{2}\right)^2-\left(1+\frac{\epsilon\alpha^{-1}}{2}\right)\epsilon\min_{\|x\|_V=1}\|x\|^2+\epsilon^2\|J\|_V^2\\
            &\leq \left(1+\frac{\epsilon\alpha^{-1}}{2}\right)^2-\epsilon\min_{\|x\|_V=1}\|x\|^2+\epsilon^2\|J\|_V^2\tag{$\epsilon, \alpha>0$}\\
            &\leq 1+\epsilon\alpha^{-1}-\epsilon\frac{1}{\|V\|}+\epsilon^2\left(\|J\|_V^2+\frac{\alpha^{-2}}{4}\right)\\
            &\leq 1-\epsilon\frac{1}{2\|V\|}+\epsilon^2\left(\|J\|_V^2+\frac{\alpha^{-2}}{4}\right)\tag{$2\iota_V>\alpha^{-1}$}
        \end{align*}
        For $\epsilon\in \left[0, \frac{1}{\|V\|\left(4\|J\|_V^2+\alpha^{-2}\right)}\right]$, we have:
        \begin{align*}
            \|I+\epsilon J\|_V^2\leq 1-\epsilon\frac{1}{4\|V\|}.
        \end{align*}
    \end{enumerate}
    
\end{proof}

\subsection{Lemmas for symmetric negative definite matrix}
\begin{lemma}\label{lem:theta_bound_symm}
    Let $\Theta_i = \sqrt{\alpha_i} \prod_{l = i+1}^k (I+\alpha_l J_l^{(\alpha, \xi)}) \text{ and } \Theta_k=\sqrt{\alpha_k}$. Furthermore, suppose $J$ is a symmetric negative definite matrix and $\iota_J$ is the magnitude of the largest eigenvalue of $J$. Then, we have the following bounds:
    \begin{enumerate}
        \item When $\xi=0$ and $\alpha\leq 8\iota_J/9$, we have
        \begin{align*}
            \sum_{i=0}^{k-1}\|\Theta_i\|^{2+\beta}\leq \frac{\alpha^{\beta/2}}{\iota_J}.
        \end{align*}
        \item When $\xi\in (0,1)$ and $K$ is chosen large enough such that $K\geq (1/(\iota_J\alpha))^{1/(1-\xi)}$, $\alpha_k\leq 8\iota_J/9$, and $ 2\alpha^{-1}\xi/K^{1-\xi}\leq \iota_J$, we have
       \begin{align*}
            \sum_{i=0}^{k-1}\|\Theta_i\|^{2+\beta}\leq \frac{2\alpha_k^{\beta/2}}{\iota_J}.
        \end{align*}
        \item When $\xi=1$, $3\left(\iota_J-\frac{\alpha^{-1}}{2}\right)\alpha> 1$, and $K$ is chosen large enough such that $\alpha_k\leq (\iota_J-\alpha^{-1}/2)/2$, we have 
        \begin{align*}
            \sum_{i=0}^{k-1}\|\Theta_i\|^{2+\beta}\leq \frac{4\alpha\alpha_k^{\beta/2}}{6\iota_J\alpha-5}.
        \end{align*}
    \end{enumerate}
\end{lemma}
\begin{proof}
For $\xi<1$, note that $\|\Theta_i\|^{2+\beta}\leq \alpha_i^{(2+\beta)/2}\prod_{l=i+1}^{k-1}\|I+\alpha_lJ_l^{(\alpha, \xi)}\|^{2+\beta}$. Furthermore, $\|I+\alpha_lJ_l^{(\alpha, \xi)}\|^{2+\beta}\leq \|I+\alpha_lJ_l^{(\alpha, \xi)}\|^{2}$. Using Lemma \ref{lem:contraction_prop_sym} and the above relations, the bound simplifies to $\|\Theta_i\|^{2+\beta}\leq \alpha_i^{(2+\beta)/2}\prod_{l=i+1}^{k-1}(1-\iota_J\alpha_l)$. Consider the sequence $\{v_k\}_{k\geq 0}$
    \begin{align*}
        v_0=0;~~v_{k+1}=(1-\iota_J\alpha_k)v_k+\alpha_k^{1+\beta/2}.
    \end{align*}
    Clearly, $v_k\geq \sum_{i=0}^{k-1} \|\Theta_i\|^{2+\beta}$. Applying Lemma \ref{lem:rec_sol} to $v_k$ by setting $u_0=0$, $\mu_1=\iota_V$, $\mu_2=0$, $\mu_3=1$, and $\rho_2=\beta/2$, we obtain the claim. 
    
    For $\xi=1$, we use statement 2 of Lemma \ref{lem:contraction_prop_sym} .
    Following that, the proof for this case proceeds using the exact same steps as before.
    
\end{proof}

\begin{lemma}\label{lem:bound_sigma_k_symm}
    Let $\{\tilde{\Sigma}_k^{(\alpha, \xi)}\}_{k\geq 0}$ be given by Eq. \eqref{eq:sigma_k_eq_pf}. Furthermore, suppose $J$ is a symmetric negative definite matrix and $\iota_J$ is the magnitude of the largest eigenvalue of $J$. Assume that for each specific choice of $\xi$, all the conditions on the step-sizes stated in Lemma \ref{lem:theta_bound_symm} hold. Then, we have the following bounds for all $k\geq 0$:
    \begin{align*}
        \left\|\tilde{\Sigma}_k^{(\alpha, \xi)}\right\|\leq \frac{\|\Sigma_b\|}{\iota_J}.
    \end{align*}
\end{lemma}
\begin{proof}
    For $\xi<1$, taking matrix norm both sides and applying statement 1 of Lemma \ref{lem:contraction_prop_sym}, we obtain
    \begin{align*}
        \|\tilde{\Sigma}_{k}^{(\alpha, \xi)}\|&\leq (1 - \iota_J\alpha_{k-1})\|\tilde{\Sigma}_{k-1}^{(\alpha, \xi)}\|+\alpha_{k-1}\|\Sigma_b\|\\
        \implies \|\tilde{\Sigma}_{k}^{(\alpha, \xi)}\|&\leq (1 - \iota_J\alpha_{k-1})\|\tilde{\Sigma}_{k-1}^{(\alpha, \xi)}\|+\iota_J\alpha_{k-1}\frac{\|\Sigma_b\|}{\iota_J}.
    \end{align*}
    It is straightforward to show that $\|\tilde{\Sigma}_k^{(\alpha, \xi)}\|\leq \Sigma_b/\iota_J$ using induction. For $\xi=1$, we use statement 2 of Lemma \ref{lem:contraction_prop_sym} to obtain the bounds.
\end{proof}

\begin{lemma}\label{lem:matrix_theta_bound_symm}
    Let $\{\Upsilon_k^{(\alpha, \xi)}\}_{k\geq 0}$ be given by 
    \begin{align*}
        \Upsilon_{k+1}^{(\alpha, \xi)}=(I + \alpha_k J^{(\alpha, \xi)}_k)\Upsilon_k^{(\alpha, \xi)}(I + \alpha_k J^{(\alpha, \xi)}_k)^T+\alpha_k\Sigma_b
    \end{align*}
    where $\Upsilon_0^{(\alpha, \xi)}$ is arbitrary. Furthermore, suppose $J$ is a symmetric negative definite matrix and $\iota_J$ is the magnitude of the largest eigenvalue of $J$. Assume that for each specific choice of $\xi$, all the conditions on the step-sizes stated in Lemma \ref{lem:theta_bound_symm} hold. Then, we have the following bounds for all $k\geq 0$:
    \begin{enumerate}
        \item When $\xi=0$ and $\alpha\leq 8\iota_J/9$, we have
        \begin{align*}
            \left\|\Upsilon_k^{(\alpha, 0)}-\Sigma\right\|\leq \|\Upsilon_0^{(\alpha, \xi)}-\Sigma\|e^{-\iota_J\alpha k}+\frac{\|J_F\|^2\|\Sigma\|\alpha}{\iota_J}.
        \end{align*}
        \item When $\xi\in (0,1)$ and $K$ is chosen large enough such that $K\geq (1/(\iota_V\alpha))^{1/(1-\xi)}$, $\alpha_k\leq 8\iota_J/9$, and $2\alpha^{-1}\xi K^{\xi-1}\leq \iota_J$, we have
       \begin{align*}
            \left\|\Upsilon_k^{(\alpha, \xi)}-\Sigma\right\|&\leq \|\Upsilon_0^{(\alpha, \xi)}-\Sigma\|e^{-\frac{\iota_J\alpha}{1-\xi}\left((k+K)^{1-\xi}-K^{1-\xi}\right)}+\frac{2\left(\|J_F\|+\frac{\alpha^{-1}}{2}\right)^2\|\Sigma\|\alpha_k}{\iota_J}+\frac{2\alpha^{-1}\|\Sigma\|}{(k+K)^{1-\xi}\iota_J}.
        \end{align*}
        \item When $\xi=1$, $\frac{3}{2}\left(\iota_J-\frac{\alpha^{-1}}{2}\right)\alpha> 1$, and $K$ is chosen large enough such that $\alpha_k\leq (\iota_J-\alpha^{-1}/2)/2$, we have 
        \begin{align*}
            \left\|\Upsilon_k^{(\alpha, 1)}-\Sigma^{(\alpha)}\right\|\leq \|\Upsilon_0^{(\alpha, \xi)}-\Sigma^{(\alpha)}\|\left(\frac{K}{k+K}\right)^{\iota_J\alpha}+\frac{\alpha\|\left(\|J_F\|+\frac{\alpha^{-1}}{2}\right)^2\|\Sigma^{(\alpha)}\|\alpha_k}{\iota_J\alpha-1}.
        \end{align*}
    \end{enumerate}
\end{lemma}
\begin{proof}
    Recall that $\Upsilon_k^{(\alpha, \xi)}$ is given by
    \begin{align*}
        \Upsilon_{k+1}^{(\alpha, \xi)}&=(I + \alpha_k J^{(\alpha, \xi)}_k)\Upsilon_k^{(\alpha, \xi)}(I + \alpha_{k} J^{(\alpha, \xi)}_k)^T+\alpha_{k}\Sigma_b\\
        &=\sum_{i=0}^{k}\Theta_i\Sigma_b\Theta_i^T\\
        \implies \Upsilon_{k+1}^{(\alpha, \xi)}-\Sigma^{(\alpha, \xi)}&=(I + \alpha_{k} J^{(\alpha, \xi)}_k)(\Upsilon_k^{(\alpha, \xi)}-\Sigma^{(\alpha, \xi)})(I + \alpha_{k} J^{(\alpha, \xi)}_k)^T+\alpha_{k}^2J^{(\alpha, \xi)}_k\Sigma^{(\alpha, \xi)} (J^{(\alpha, \xi)}_k)^T\\
        &~+\frac{\xi\mathbbm{1}_{\xi<1}}{(k+K)}\Sigma^{(\alpha, \xi)}.
    \end{align*}
    where for the last equation, we used the fact that $\Sigma^{(\alpha, \xi)}$ is the solution to Lyapunov equation \eqref{eq:lyap_eq}. 
    \begin{enumerate}
        \item $\xi=0$: In this case, the last term in the above equation is 0. Thus, using the statement 1 Lemma \ref{lem:contraction_prop_sym}, we have
        \begin{align*}
            \|\Upsilon_{k+1}^{(\alpha, 0)}-\Sigma\|&\leq (1-\iota_J\alpha)\|\Upsilon_k^{(\alpha, 0)}-\Sigma\|+\alpha^2\|J_F\Sigma J_F^T\|.
        \end{align*}
        Now it is easy to apply Lemma \ref{lem:rec_sol} by setting $u_0=\|\Upsilon_0^{(\alpha, \xi)}-\Sigma\|, \mu_1=\iota_J$, $\mu_2=0$, $\mu_3=\|J_F\|^2\|\Sigma\|$, and $\rho_2=1$.
        \item $\xi\in (0, 1)$: In this case, the last term in non-zero and needs to be considered in the analysis. Again using statement 1 of Lemma \ref{lem:contraction_prop_sym}, we have
        \begin{align*}
            \|\Upsilon_{k+1}^{(\alpha, \xi)}-\Sigma\|&\leq (1-\iota_J\alpha_{k})\|\Upsilon_k^{(\alpha, \xi)}-\Sigma\|_V+\alpha_{k}^2\|J^{(\alpha, \xi)}_k\Sigma (J^{(\alpha, \xi)}_k)^T\|+\alpha_k^{\frac{1}{\xi}}\alpha^{-1}\|\Sigma\|.
        \end{align*}
         Now we just need to apply Lemma \ref{lem:rec_sol} twice by setting $\rho_2=1$ and $\rho_2=1/\xi-1$ and then combine the upper bounds to get the claim.
        
        \item $\xi=1$: In this case, the last term is also 0. Thus, using the statement 2 Lemma \ref{lem:contraction_prop_sym}, we have
        \begin{align*}
            \|\Upsilon_{k+1}^{(\alpha, 1)}-\Sigma^{(\alpha)}\|&\leq (1-\iota_V\alpha_{k})\|\Upsilon_{k}^{(\alpha, 1)}-\Sigma^{(\alpha)}\|+\alpha_{k}^2\|J_F^{(\alpha)}\Sigma^{(\alpha)} (J_F^{(\alpha)})^T\|.
        \end{align*}
        We again apply Lemma \ref{lem:rec_sol} by setting the appropriate constants to get the claim.
    \end{enumerate}
\end{proof}

\begin{lemma}\label{lem:contraction_prop_sym}
Suppose that $J$ is a symmetric negative definite matrix. Then, we have the following:
\begin{enumerate}
    \item When $\xi<1$ and $K$ is chosen large enough such that $\iota_J\geq 2\alpha^{-1}\xi/K^{1-\xi}$ and $\alpha_k\leq 8\iota_J/9$, we have
    \begin{align*}
        \|I+\alpha_lJ_l^{(\alpha, \xi)}\|^{2}\leq 1-\iota_J\alpha_k.
    \end{align*}
    \item When $\xi=1$ and $K$ is chosen large enough such that $\alpha_l\leq (\iota_J-\alpha^{-1}/2)/2$, we have
    \begin{align*}
        \|I+\alpha_lJ_l^{(\alpha, 1)}\|^{2}&\leq 1-\frac{3}{2}\left(\iota_J-\frac{\alpha^{-1}}{2}\right)\alpha_k.
    \end{align*}
\end{enumerate}
\end{lemma}
\begin{proof}
    \begin{enumerate}
        \item Recall 
        \begin{align*}
            J_l^{(\alpha, \xi)}=J+\frac{\alpha^{-1}\xi}{2(l+K)^{1-\xi}}I.
        \end{align*}
        Since $K$ is chosen large enough such that $\iota_J\geq 2\alpha^{-1}\xi/K^{1-\xi}$, we have 
        \begin{align*}
            \|I+\alpha_lJ_l^{(\alpha, \xi)}\|^{2}&\leq \left(1-\frac{3\iota_J\alpha_k}{4}\right)^2\\
            &= 1-\frac{3\iota_J\alpha_k}{2}+\frac{9\iota_J^2\alpha_k^2}{16}.
        \end{align*}
        Recall that $\alpha_k$ is chosen such that $\alpha_l\leq 8\iota_J/9$. Thus,
        \begin{align*}
            \|I+\alpha_lJ_l^{(\alpha, \xi)}\|^{2}
            &\leq 1-\iota_J\alpha_k.
        \end{align*}
        \item Note that in this case
        \begin{align*}
            J_l^{(\alpha, 1)}=J+\frac{\alpha^{-1}}{2}I.
        \end{align*}
        Thus, we have
        \begin{align*}
            \|I+\alpha_lJ_l^{(\alpha, 1)}\|^{2}&= \left(1-\left(\iota_J-\frac{\alpha^{-1}}{2}\right)\alpha_k\right)^2\\
            &= 1-(2\iota_J-\alpha^{-1})\alpha_k+\left(\iota_J-\frac{\alpha^{-1}}{2}\right)^2\alpha_k^2.
        \end{align*}
        Recall that $K$ is chosen large enough such that $\alpha_l\leq (\iota_J-\alpha^{-1}/2)/2$. Thus,
        \begin{align*}
            \|I+\alpha_lJ_l^{(\alpha, 1)}\|^{2}&\leq 1-\frac{3}{2}\left(\iota_J-\frac{\alpha^{-1}}{2}\right)\alpha_k.
        \end{align*}
    \end{enumerate}
\end{proof}
\end{document}